\pdfoutput=1

\documentclass[11pt]{article}

\usepackage{EMNLP2023}
\usepackage{times}
\usepackage{latexsym}

\usepackage[T1]{fontenc}

\usepackage[utf8]{inputenc}

\usepackage{microtype}

\usepackage{inconsolata}
\usepackage{epsfig}
\usepackage{graphicx}
\usepackage{amsmath}
\usepackage{amssymb}
\usepackage{subcaption}
\usepackage{mathtools}
\usepackage{wrapfig}
\usepackage{multirow}
\usepackage{multicol}
\usepackage{microtype}
\usepackage{microtype}
\usepackage{fixltx2e}
\usepackage{comment}
\usepackage{graphicx}

\usepackage{booktabs}
\usepackage{color}
\usepackage{algorithm}
\usepackage{algpseudocode}
\usepackage{amsthm}
\usepackage{nameref}
\usepackage{xcolor,colortbl}
\usepackage{ulem}

\usepackage{float}

\usepackage[american]{babel}

\usepackage{cleveref}

\def\eqref#1{equation~\ref{#1}}

\def\onedot{.}
\def\eg{\emph{e.g}\onedot} 

\def\ie{\emph{i.e}\onedot}

\newtheorem{theorem}{Theorem}

\newtheorem{lemma}[theorem]{Lemma}

\newtheorem{corollary}[theorem]{Corollary}

\newtheorem{remark}{Remark}

\newcounter{mylabelcounter}

\makeatletter
\newcommand{\labelText}[2]{%
#1\refstepcounter{mylabelcounter}%
\immediate\write\@auxout{%
  \string\newlabel{#2}{{1}{\thepage}{{\unexpanded{#1}}}{mylabelcounter.\number\value{mylabelcounter}}{}}%
}%
}
\makeatother

\renewcommand{\comment}[1]{}

\newlength\myindent 
\algdef{SE}[SUBALG]{Indent}{EndIndent}{}{\algorithmicend\ }%
\algtext*{Indent}
\algtext*{EndIndent}

\def \problem {representation degeneration problem}

\newcommand{\ours}{\textsc{InvGC}}
\newcommand{\pool}{\textsc{AvgPool}}
\newcommand{\tail}{\textsc{LocalAdj}}
\newcommand{\fcut}{\textsc{InvGC} w/\textsc{LocalAdj}}
\useunder{\uline}{\ul}{}

\begin{document}
\title{\ours: Robust Cross-Modal Retrieval by Inverse Graph Convolution}

\author{Xiangru Jian \and Yimu Wang\thanks{\ \  Corresponding author.}\\
  University of Waterloo \\
  \texttt{\{xiangru.jian,yimu.wang\}@uwaterloo.ca} \\
}
\maketitle
\begin{abstract}
Over recent decades, significant advancements in cross-modal retrieval are mainly driven by breakthroughs in visual and linguistic modeling. 
However, a recent study shows that multi-modal data representations tend to cluster within a limited convex cone (as \problem), which hinders retrieval performance due to the inseparability of these representations. 
In our study, we first empirically validate the presence of the \problem\ across multiple cross-modal benchmarks and methods. 
Next, to address it, we introduce a novel method, called \ours, a post-processing technique inspired by graph convolution and average pooling. 
Specifically, \ours\ defines the graph topology within the datasets and then applies graph convolution in a subtractive manner. 
This method effectively separates representations by increasing the distances between data points. 
To improve the efficiency and effectiveness of \ours, we propose an advanced graph topology, \tail, which only aims to increase the distances between each data point and its nearest neighbors. 
To understand why \ours\ works, we present a detailed theoretical analysis, proving that the lower bound of recall will be improved after deploying \ours. 
Extensive empirical results show that \ours\ and \fcut\ significantly mitigate the \problem, thereby enhancing retrieval performance.

Our code is available at \href{https://github.com/yimuwangcs/Better_Cross_Modal_Retrieval}{link}.
\end{abstract}

\begin{figure}[t!]
  \centering
    \subcaptionbox{Original representations.\label{fig: datadeg a}
    }{
    \includegraphics[width=0.22\textwidth]{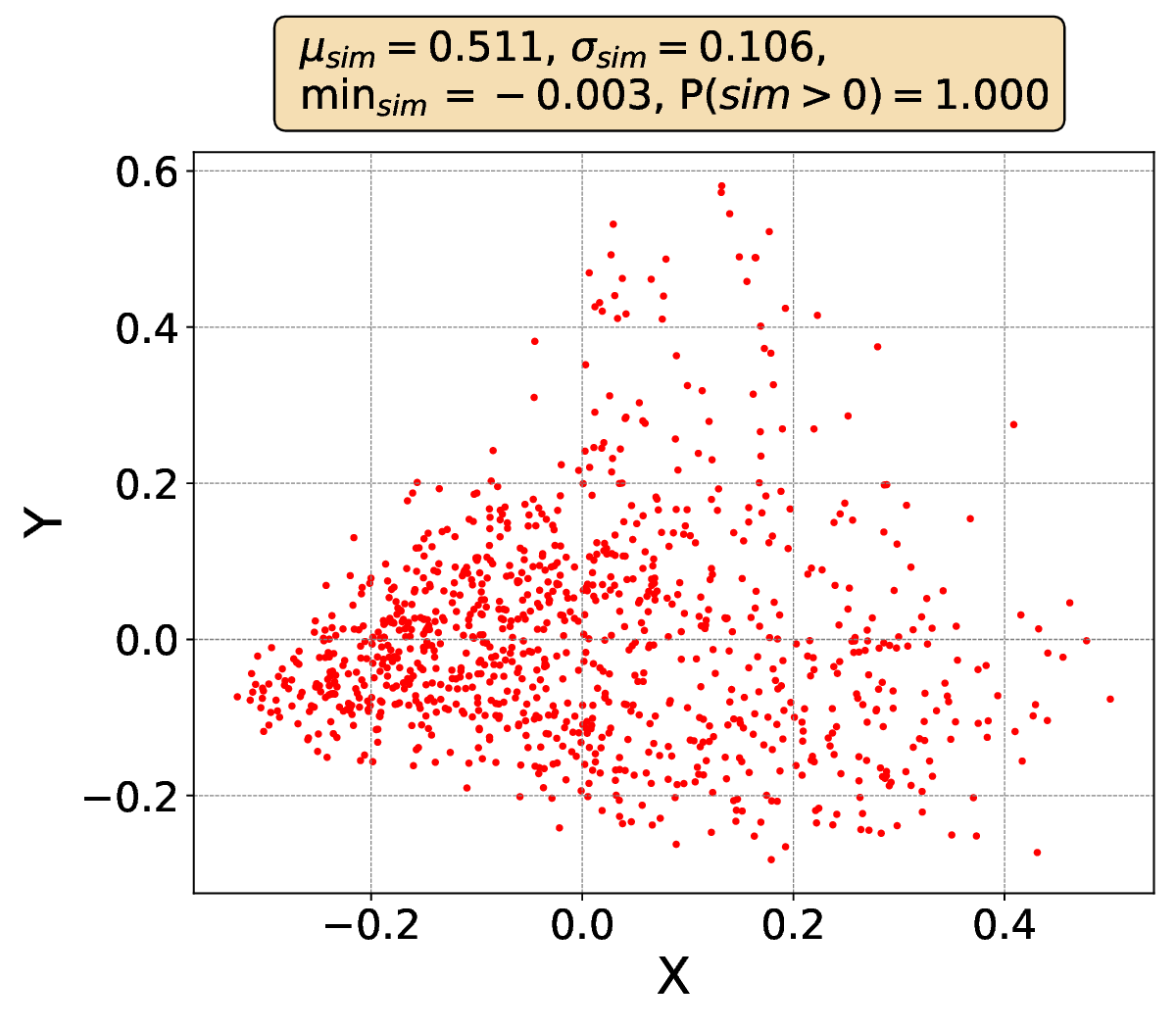}
    }
    \subcaptionbox{Updated representations by \ours. \label{fig: datadeg b}}{
    \includegraphics[width=0.22\textwidth]{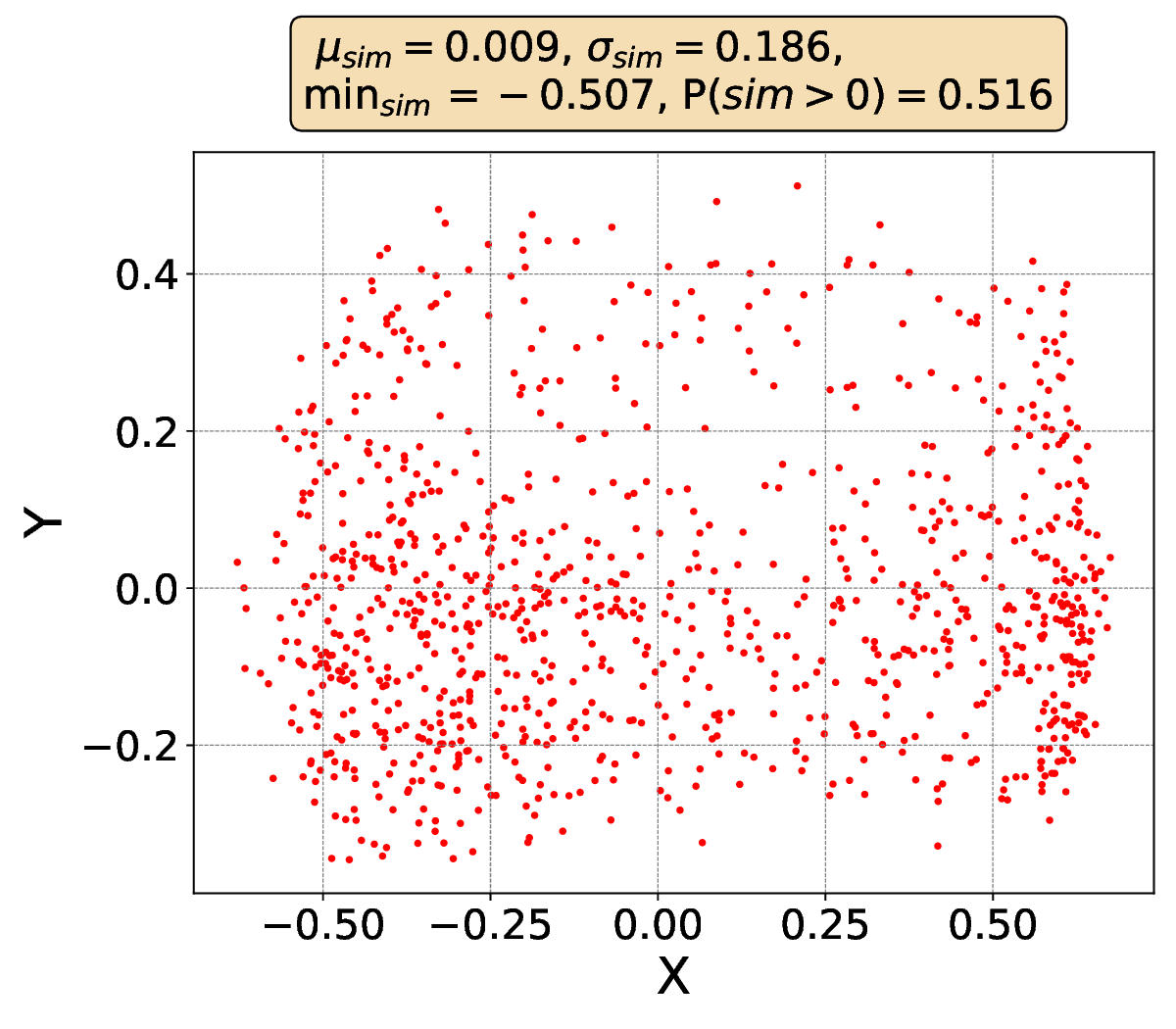}
    }

    \subcaptionbox{Histogram of the similarity between nearest neighbors. \label{fig: 3d_deg left}}{
    \includegraphics[width=0.22\textwidth]{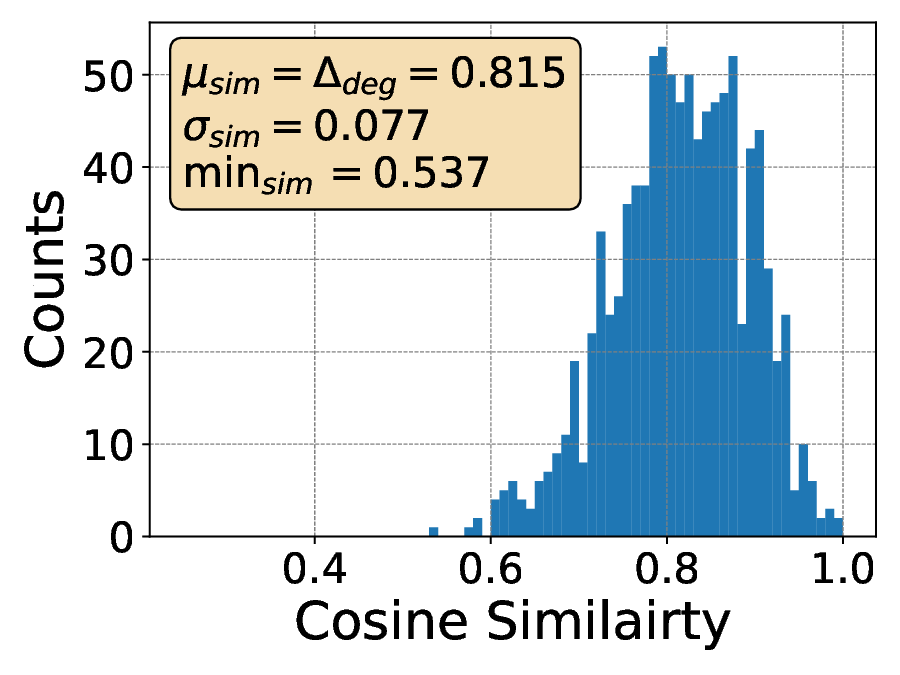}
    }
    \subcaptionbox{Histogram of the similarity between nearest neighbors after \ours. \label{fig: 3d_deg right}}{
    \includegraphics[width=0.22\textwidth]{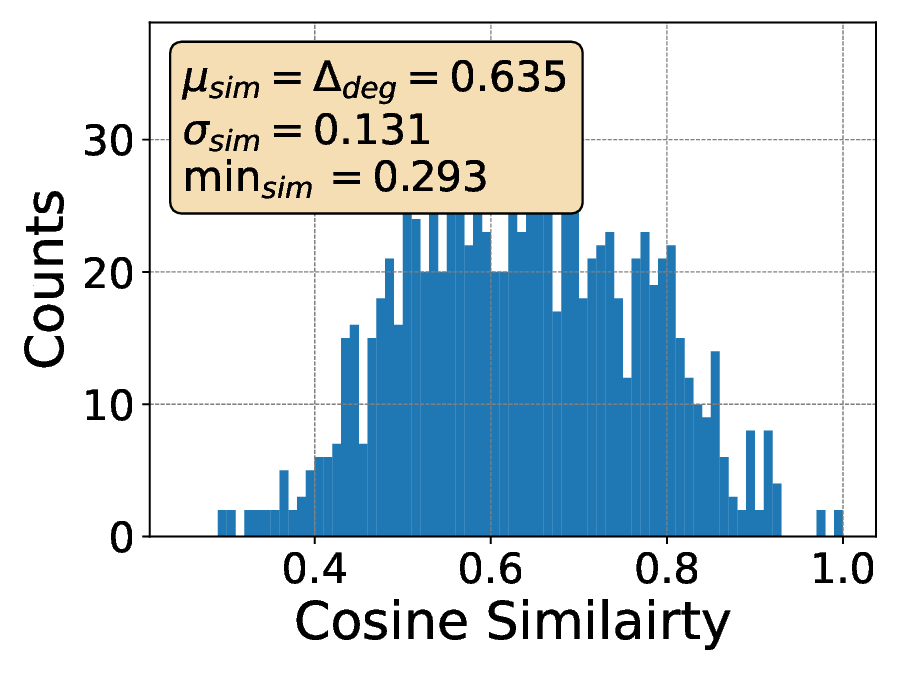}
    }
    \caption{A 2D visualization of the representation of text data uniformly sampled from the MSCOCO dataset generated by CLIP~\cite{DBLP:conf/icml/RadfordKHRGASAM21} and then reduced by PCA. 
    As shown in \Cref{fig: datadeg a}, the \problem\ can be observed as all the representations are concentrated in a small convex cone, while after employing \ours, representations are better separated. 
    $\mu_{sim}$, $\sigma_{sim}$, and $\min_{sim}$ are the mean, standard deviation, and minimum of the similarity. 
    $\Delta_{deg}$ is the score of \problem\ defined in \Cref{eq: deg} .
    }
    \label{fig: datadeg}
\end{figure}

\section{Introduction}
Cross-modal retrieval (CMR)~\cite{ijcai2021p156,yu2023multimodal,kim2023nice}, which aims to enable flexible retrieval across different modalities, \eg, images, videos, audio, and text, has attracted significant research interest in the last few decades. 
The goal of CMR is to learn a pair of encoders that map data from different modalities into a common space where they can be directly compared. 
For example, in text-to-video retrieval, the objective is to rank gallery videos based on the features of the query text.  
Recently, inspired by the success in self-supervised learning~\cite{DBLP:conf/icml/RadfordKHRGASAM21}, significant progress has been made in CMR, including image-text retrieval~\cite{DBLP:conf/icml/RadfordKHRGASAM21,DBLP:conf/eccv/Li0LZHZWH0WCG20,10.1145/3394171.3413882}, video-text retrieval~\cite{DBLP:conf/cvpr/ChenZJW20,DBLP:journals/corr/abs-2109-04290,DBLP:journals/corr/abs-2111-05610,DBLP:conf/cvpr/LeiLZGBB021,DBLP:conf/mm/MaXSYZJ22,park-etal-2022-exposing,DBLP:conf/mm/WangXHLJHD22,9878037,10.1145/3477495.3531950,wang_video-text_2023,wang2023balance}, and audio-text retrieval~\cite{DBLP:conf/interspeech/OncescuKHAA21}, with satisfactory retrieval performances.

However, \citet{liang2022mind} demonstrate that the current common learning paradigm of CMR leads to the \problem, which concentrates all the representations in a small (convex) cone~\cite{DBLP:conf/iclr/GaoHTQWL19,zhang-etal-2020-revisiting} in image-text retrieval and the cosine similarity between any two points is positive, as shown in \Cref{fig: datadeg a}. 
Consequently, retrieval performance will be significantly affected~\cite{JMLR:v11:radovanovic10a,DBLP:conf/iclr/GaoHTQWL19}. Due to the limitation of space, detailed related works are presented in \Cref{sec: relatedwork}. 

In this paper, to step forward in addressing \problem\ and further improve the retrieval performance, we first empirically test whether it is prevalent across various cross-modal retrieval models and datasets, including video, image, text, and audio. 
We found that the representations of MSCOCO generated by CLIP are gathered in a very narrow cone in the embedding space, proving the existence of \problem, as shown in \Cref{fig: datadeg a}. 
This case does not stand alone since similar observations are observed across several cross-modal retrieval models and datasets as shown in \Cref{sec: prevalence of prob}. 

Next, to model how severe the \problem\ is and the relationship between this problem and the retrieval performance, drawing inspiration from previous work~\cite{DBLP:conf/iclr/GaoHTQWL19,zhang-etal-2020-revisiting,yu-etal-2022-rare,liang2022mind,Tang2022cvpr,NEURIPS2021_Huang}, we define it in CMR as the average similarity between each point and its nearest neighbor in the gallery set, as shown in \Cref{eq: deg}. 
We observe that the scores are very high across different datasets and methods. They are able to model this problem as a high score always leads to more concentrated data distribution as shown in \Cref{sec: prevalence of prob}.

While CMR has suffered from this problem, on the other side, the graph convolution~\cite{kipf2017semisupervised} and average pooling~\cite{10.5555/3104322.3104338}, which are widely employed in graph neural networks~\cite{gilmer2017mpnn,kipf2017semisupervised,velickovic2018graph} and deep neural networks~\cite{DBLP:conf/cvpr/HeZRS16,yu2023multimodal}, respectively, are designed to move the representations closer to each other if they are semantically similar~\cite{baranwal2023effects}. It might lead to the emergence of \problem.

Drawing inspiration from the graph convolution and average pooling, we propose a novel method, \ours, which separates representations by performing the graph convolution inversely to separate representations with a bigger margin as shown in \Cref{fig: datadeg a}. 
Specifically, different from the vanilla graph convolution, considering one modality, \ours\ separates representations by subtracting the representation of the neighboring nodes from each node, instead of aggregating them as,
\begin{equation}\label{eq:invconvwithr}
    \mathbf{x_i}^\prime =  \mathbf{x_i} - r \sum_{j \neq i} S_{ij} \mathbf{x_j}, \forall i\,,
\end{equation}
where $\mathbf{x_i}^\prime$ and $\mathbf{x_i}$ are the updated and the original representations, $S_{ij}$ is the similarity between the $i$-th and $j$-th data, and $r$ is a predefined hyperparameter.
As shown in \Cref{fig: datadeg}, \ours\ better scatter representations and alleviates \problem.
Moreover, the histogram of similarity between any two points is more balanced with \ours, as shown in \Cref{fig: datadeg}. 
To boost the effectiveness and efficiency of \ours, we propose an advanced adjacency matrix \tail\ that directs \fcut\ to focus on the nearest neighbors of each data point instead of considering all the data points.

To evaluate the effectiveness of \ours\ and \fcut, we conducted experiments on eight cross-modal benchmarks~\cite{DBLP:conf/cvpr/XuMYR16,chen-dolan-2011-collecting,caba2015activitynet,hendricks_localizing_2017,DBLP:conf/eccv/LinMBHPRDZ14,DBLP:journals/ijcv/PlummerWCCHL17,kim-etal-2019-audiocaps,drossos_clotho_2020}. 
Experimental results show that \ours \ alleviates \problem\ across different datasets and methods and improves retrieval performance as a by-product. 

In summary, our contributions are as follows\footnote{The code is released at \href{https://github.com/yimuwangcs/Better_Cross_Modal_Retrieval}{link}.}:
\begin{itemize}
    \item We are the first to formalize the definition of representation degeneration in cross-modal retrieval and perform a theoretical analysis of the relationship between \problem\ and retrieval performance.
    \item Inspired by the graph convolution, we propose the first post-processing method in cross-modal retrieval, namely \ours, to alleviate \problem\ without any training process or additional data. 
    \item We design an adjacency matrix, called \tail, for the graph convolution, which leverages only the nearest neighbors of each data point instead of all the data.
    \ours\ with \tail, namely \fcut.
    It is shown to be more effective and efficient. 
    \item  Extensive experiments show that \ours\ and \fcut\ alleviate \problem\ and improve retrieval performance as a by-product.
\end{itemize}

\section{Preliminaries} 
\subsection{Task Definition}\label{sec:def}
In this paper, we focus on the representation degeneration problem in cross-modal retrieval~\cite{DBLP:conf/prcv/WangWXZ20}. 
Two modalities are denoted as $\mathcal{X}$ and $\mathcal{Y}$. $\mathcal{X}$ is the query modality, while $\mathcal{Y}$ is the gallery modality. 
The (test) gallery, denoted $G=\{\mathbf{g}_1, \ldots, \mathbf{g}_{N_g}\}$, contains all the representations of the (test) gallery data, where $N_g$ is the size. 
The query set is $Q = \{\mathbf{q}_1, \ldots, \mathbf{q}_{N_q}\}$, where $N_q$ is the number of queries.
Usually, in cross-modal retrieval, the gallery data does not overlap with the training data. 
Additionally, as \ours\ requires training (or validation) data to address the representation degeneration problem, we define the set of representations of training (or validation) query and gallery data as $\hat{Q} = \{\hat{\mathbf{q}}_1, \ldots, \hat{\mathbf{q}}_{N_{\hat{Q}}}\}$ and $\hat{G} = \{\hat{\mathbf{g}}_1, \ldots, \hat{\mathbf{g}}_{N_{\hat{G}}}\}$, respectively, where $N_{\hat{Q}}$ and $N_{\hat{G}}$ are the size of the training (or validation) query and gallery set, respectively. 
The similarity between two embeddings $\mathbf{a}$ and $\mathbf{b}$ is defined as $s_{\mathbf{a}, \mathbf{b}} = \operatorname{sim}(\mathbf{a}, \mathbf{b})$, where $\operatorname{sim}(\cdot, \cdot)$ could be some measure of distance.

\subsection{Representation Degeneration in Cross-modal Retrieval}

Taking inspiration from \citet{liang2022mind}, we define the \problem\ as the average similarity of pairs that are closest to each other,
\begin{equation}\label{eq: deg}
    \Delta_{deg} = \frac{1}{m} \sum_{\mathbf{x}} \operatorname{sim}(\mathbf{x}, \mathbf{y})\,,
\end{equation}
where $\mathbf{y}$ is the closest data point to $\mathbf{x}$ while $m$ is the number of data points.
A higher score, as shown in \Cref{sec: prevalence of prob}, indicates a more severe representation degeneration issue, with each data point being closer to its nearest neighbors in the set.

To understand the relationship between the degeneration score (\Cref{eq: deg}) and the retrieval performance, we present the following theorems.

\begin{theorem}\label{the: 1}
Let $\mathbf{x}_1$ be any point in $G$, $\mathbf{x}_2$ be the nearest neighbor of $\mathbf{x}_1$ and $n$ be the dimension of the representation. Given a query point $\mathbf{q}$ that is semantically similar to $\mathbf{x}_1$ and sampled from $Q$, which follows an independent and identical uniform distribution in $\mathbb{R}^n$, the probability of a query point $\mathbf{q}$ to successfully retrieve $\mathbf{x_1}$, denoted $\operatorname{P}(\mathbf{x_1}, b)$, is bounded by,

$$
\frac{n}{2 }\cdot b^{n} > \operatorname{P}(\mathbf{x_1},b) >\frac{1}{4}\cdot b^{n+1}\,,
$$
where $b = ||\mathbf{x_1} -\mathbf{x_2}||_2/2$.
\end{theorem}

\begin{corollary}\label{the: 4} 
The ratio between the probability of successful retrieval of any two different neighborhood radii, namely $b_1$ and $b_2$, is 
$$
\frac{\operatorname{P}(\mathbf{x_1},b_1)}{\operatorname{P}(\mathbf{x_1},b_2)} = \mathcal{O}\left(n \cdot \left(\frac{b_1}{b_2}\right)^{n} \right)\,.
$$

\end{corollary}
Due to space limitation, the proofs are deferred to \Cref{A:1}. 

\begin{remark}
    These theorems show that a high similarity of the nearest neighbor, \ie, smaller $b$, leads to an exponentially lower probability for successful retrieval. 
    Therefore, a higher $\Delta_{deg}$ score leads to bad retrieval performance. 
\end{remark}

\section{\ours}

To alleviate the \problem\ and further boost the retrieval performance, we design a post-processing method, called \ours, which does not need any additional training.

Our idea is generally based on the mathematical principles laid out in \Cref{eq:invconvwithr} and \Cref{fig: datadeg}, where the inverse version of graph convolution is able to decrease the similarity between data points and their neighbors. 
For the sake of clear representation, we use the cosine similarity as the similarity metric in this section following \citet{DBLP:journals/ijon/LuoJZCLDL22}, as it is the most common practice in cross-modal retrieval\footnote{\ours\ can be easily migrated to other similarity metrics adopted by different retrieval methods~\cite{DBLP:conf/iccv/CroitoruBLJZAL21,DBLP:conf/bmvc/LiuANZ19}.}.
\textbf{The formal definition of \ours\ is presented in \Cref{sec: formal define}}.

\subsection{Mathematical Intuition}
\ours\ originates from graph convolution, which is widely employed in graph neural networks~\cite{kipf2017semisupervised,gilmer2017mpnn,velickovic2018graph}. 

Specifically, graph convolution will concentrate all the embeddings of similar nodes which might lead to the concentration of similarity~\cite{10.1214/aop/1176988291} and the data degeneration problem ~\cite{baranwal2023effects}. 
On the other side, a similar operation, average pooling, has been employed in computer vision~\cite{DBLP:conf/cvpr/HeZRS16,wang_video-text_2023}. 
Average pooling will aggregate the features that are location-based similar\footnote{The details of graph convolution and average pooling discussed in this study are deferred to the \Cref{sec:graphconv,sec:averagepool}, respectively.}.

As a result, graph convolution and average pooling concentrate the representation of all the similar nodes and force them to become very similar to each other, potentially leading to \problem. 
This observation inspires us to pose the following research question:
\begin{center}
{\textit{Can the issue of \problem\ in cross-modal retrieval be alleviated by conducting inverse graph convolution (\ours)?}}
\end{center}
To answer this question, we first give an inverse variant of graph convolution as follows, 
\begin{equation}\label{eq:invconv}
\mathbf{x_i}^\prime = \mathbf{x_i} - \sum_{j} A_{ij} \mathbf{x_j}\,,
\end{equation}
where $A$ is the adjacency matrix for the data in the gallery set. 
Since the adjacency matrix is not available, to encourage separating the nearest neighbor in terms of the similarity score of each node, we choose $A_{ij} = \operatorname{sim}(i,j) \coloneqq S_{ij}$ with detailed discussion in \Cref{Sec adj}. 
Therefore, based on the inverse graph convolution, we propose the basic setting of \ours\ as shown in \Cref{eq:invconvwithr}, which only considers one modality. 
We notice that it is able to alleviate the \problem\ as shown in \Cref{fig: datadeg}.

Note that the ultimate goal of \ours\ is to reduce $\Delta_{deg}$ score of the distribution of the representation of the gallery instead of merely the gallery set $G$, which can be regarded only as a sampled subset of the distribution with very limited size in practice. 
Therefore, the best approximation of the distribution is the training (or validation) gallery set $\hat{G}$ since it is the largest one we can obtain\footnote{In practice, the test queries are invisible to each other as the queries do not come at the same time. So the size of the query set $N_g$ is equal to 1.}. 

Similarly, the distribution of the query set $\hat{Q}$ is theoretically expected to be similar to that of $\hat{G}$ as a basic assumption in machine learning~\cite{DBLP:conf/cvpr/BogolinCJLA22}. 
A detailed explanation of the claims is included in \Cref{sec: distribution}. 
Moreover, as CMR needs to contrast the data points from both modalities, we utilize the (train or validation) gallery set $\hat{G}$ and the (train or validation) query set $\hat{Q}$ to better estimate the hidden distribution as shown in \Cref{eq:invconvpttp}, 
\begin{equation}\label{eq:invconvpttp}
\mathbf{x_i}^\prime =  \mathbf{x_i} - r_g \sum_{\mathbf{x_j} \in \hat{G}} S^g_{ij} \mathbf{x_j} - r_q \sum_{\mathbf{x_j} \in \hat{Q}} S^q_{ij} \mathbf{x_j}\,.
\end{equation}
where $r_g$ and $r_q$ are two hyperparameters, $\mathcal{S}^{g} \in \mathbb{R}^{N_g \times N_{\hat{G}}}$ and $\mathcal{S}^{q} \in \mathbb{R}^{N_q \times N_{\hat{Q}}}$ is the adjacency matrices between every pair of embedding from $G$ and $\hat{G}$ and that from $Q$ and $\hat{Q}$, respectively.

{To the best of our knowledge, \ours\ is the first to utilize the (inverse) graph convolution for separating the representation of data. Instead of the commonly used capability of aggregation and message passing, we introduce an inverse variant of convolution that separates data representations compared to the vanilla graph convolution.}

\subsection{Constructing Adjacency matrix}\label{Sec adj}

Next, we need to establish the graph structure, \ie, build the adjacency matrix $\mathcal{S}^{g}$ and $\mathcal{S}^{q}$, since there is no natural graph structure in the dataset. 
The simplest idea will be that the edge weight between $i$ and $j$ equals 1, \ie, $S_{ij} = 1$, if the cosine similarity between $\mathbf{x_i}$ and $\mathbf{x_j}$, \ie, $\operatorname{sim}(\mathbf{x_i}, \mathbf{x_j})$, is larger than 0 (or some thresholds). 
\ours\ with this form of the adjacency matrix is an inverse variant of average pooling. It serves as a baseline in this study, denoted as \pool. 

However, this scheme is not capable to reflect the degree of the similarity between $\mathbf{x_i}$ and $\mathbf{x_j}$ since it cannot precisely depict the relation between different data points. 

As the magnitude of $S_{ij}$ directly controls how far will $\mathbf{x_i}$ go in the opposite direction of $\mathbf{x_j}$, a greedy design will be using the similarity score between them as the edge weight, \ie, $S_{ij}= \operatorname{sim}(\mathbf{x_i}, \mathbf{x_j})$.

Therefore, we can calculate the adjacency matrix $\mathcal{S}^{g}$, which contains the similarity score between every pair of embedding from $G$ and $\hat{G}$, respectively. 
Specifically, the $(i,j)$-entry of $\mathcal{S}^{g}$ follows,
\begin{equation} \label{eq:sim1}
\mathcal{S}^{g}_{i,j} = \operatorname{sim}(\mathbf{g_i}, \mathbf{\hat{g}_j})\,.
\end{equation}
Similarly, the matrix $\mathcal{S}^{q}$ containing the similarity score between every pair of embedding from $G$ and $\hat{Q}$ is calculated as follows,
\begin{equation} \label{eq:sim2}
    \mathcal{S}^{q}_{i,j} = \operatorname{sim}(\mathbf{g_i}, \mathbf{\hat{q}_j})\,.
\end{equation}
Now, with the well-defined $\mathcal{S}^{q}$ and $\mathcal{S}^{g}$, we can finally perform \ours\ to alleviate \problem.

As shown in~\Cref{the: 1}, given any data point, the similarity of the nearest neighbor in the representation space is critical to retrieval performance. 
Inspired by this, when performing the inverse convolution on a node, we force \ours\ to pay attention to those most similar nodes to it. 
This can be achieved by assigning edge weight only to the nodes having the top $k$ percent of the largest similarity scores relative to the given node.
Specifically, each entry of $\mathcal{S}^{g}$ and $\mathcal{S}^{q}$ in this case is calculated as follows,
\begin{align} \label{eq:sim3}
\small
\mathcal{S}^{g}(i,j)= \begin{cases}
\operatorname{sim}(\mathbf{g_i}, \mathbf{\hat{g}_j})&\text{, if}\ \operatorname{sim}(\mathbf{g_i}, \mathbf{\hat{g}_j}) \geq P_i(\hat{G},k)\\
0\,&\text{,\,else}
\end{cases}\,
\end{align}
\begin{align} \label{eq:sim4}
\small
\mathcal{S}^{q}(i,j)= \begin{cases}
\operatorname{sim}(\mathbf{q_i}, \mathbf{\hat{q}_j})&\text{, if} \operatorname{sim}(\mathbf{q_i}, \mathbf{\hat{q}_j}) \geq P_i(\hat{Q},k)\\
0\,&\text{,\,else}
\end{cases}\,
\end{align}

where $P_i(\hat{G},k)$ is the value of $k$-percentage largest similarity between node $i$ and all the nodes in $\hat{G}$. The same approach applies to $P_i(\hat{Q},k)$ as well.
We denote this refined adjacency matrix as \tail\ since it focuses on the local neighborhoods of each node.

\subsection{Formal Definition of \ours}\label{sec: formal define}

With the well-formed adjacency matrices , we formally define \ours\ to obtain the updated embedding of $G$, denoted as $G^\prime$, in a matrix form as,
\begin{equation} \label{eq:formaldef}
\small
     G^\prime =   \frac{1}{2}\left[ \operatorname{norm}(G -r_g \mathcal{S}^{g} \hat{G}) + \operatorname{norm}(G -r_q \mathcal{S}^{q} \hat{Q})\right]\,,
\end{equation}
where $\operatorname{norm}(\cdot)$ normalizes a matrix with respect to rows, which is employed for uniformly distributing the intensity of convolution when cosine similarity is used and should be removed when adopting other similarity metrics and the adjacency matrices $\mathcal{S}^{g}$ and $\mathcal{S}^{q}$ can be calculated as \Cref{eq:sim1,eq:sim2,eq:sim3}. 
Note that, compared to \Cref{eq:invconvpttp}, we separate the convolution on $\hat{G}$ and $\hat{Q}$ to pursue more robust results, for avoiding the distribution shifts between $\hat{G}$ and $\hat{Q}$ due to the imperfection of the representation method. 

In summary, \ours, as shown in \Cref{eq:formaldef}, is a brand new type of graph operation performed on the specifically designed graph structures (adjacency matrix) of data points in the cross-modal dataset. It helps alleviate the \problem\ by separating data points located close to each other in the representation space.

After obtaining $G^{\prime}$, the similarity between the $i$-th gallery points $\mathbf{g}^{\prime}$ and a query point $\mathbf{q}$ is calculated as $sim(\mathbf{q}, \mathbf{g}^{\prime})$.

\begin{table*}[t!]
\centering
\caption{\textbf{Similarity measures within gallery set (MeanSim@1 equivalent to $\Delta_{deg}(G)$) on MSCOCO.} 
``MeanSim'' and ``MeanSim@k'' refer to the mean of the similarity between all data points and that between all data points and their top-k neighbors.
}
\label{tab: sim measure intra}
\resizebox{\textwidth}{!}{%
\begin{tabular}{l|ccc|ccc}
\toprule
& \multicolumn{3}{c|}{Text-to-Video Retrieval}                                                 & \multicolumn{3}{c}{Video-to-Text Retrieval}                                                 \\
& MeanSim $\downarrow$& MeanSim@1$\downarrow$ & MeanSim@10$\downarrow$ & MeanSim $\downarrow$& MeanSim@1$\downarrow$ & MeanSim@10$\downarrow$\\
\midrule
CLIP & 0.4717 & 0.8211 & 0.7803 & 0.5162 & 0.9029 & 0.8596  \\
\midrule
\rowcolor{red!10} CLIP w. \ours\  & 0.4693 & 0.8142 & 0.7738 & 0.5137 & 0.8972 & 0.8542        \\
Difference to the baseline (\%) & 0.51 $\downarrow$  & 0.84 $\downarrow$  & 0.83 $\downarrow$  & 0.48 $\downarrow$  & 0.63 $\downarrow$  & 0.63  $\downarrow$                 \\
\midrule 
\rowcolor{red!10} CLIP w. \fcut\  & 0.4646 & 0.8059 & 0.7647 & 0.5105 & 0.8924 & 0.8477        \\
Difference to the baseline (\%) & 1.51 $\downarrow$  & 1.85  $\downarrow$ & 2.00  $\downarrow$ & 1.09 $\downarrow$  & 1.16 $\downarrow$  & 1.38   $\downarrow$                \\
\bottomrule
\end{tabular}%
}
\end{table*}

\begin{table*}[ht!]
\centering
\caption{\textbf{Similarity measures between test gallery and query set on MSCOCO}. The nearest neighbor is considered with respect to the gallery set.}
\label{tab: sim measure inter}
\resizebox{\textwidth}{!}{%
\begin{tabular}{l|ccc|ccc}
\toprule
& \multicolumn{3}{c|}{Text-to-Video Retrieval}                                                 & \multicolumn{3}{c}{Video-to-Text Retrieval}                                                 \\
& MeanSim$\downarrow$ & MeanSim@1 $\downarrow$& MeanSim@10$\downarrow$ & MeanSim$\downarrow$ & MeanSim@1$\downarrow$ & MeanSim@10$\downarrow$\\
\midrule
CLIP & 0.1516 & 0.3282 & 0.3138 & 0.1516 & 0.3213 & 0.3009 \\
\midrule
\rowcolor{red!10} CLIP w. \ours\  & 0.1500 & 0.3245 & 0.3101  & 0.1511 & 0.3198 & 0.2994        \\

Difference to the baseline(\%) & 1.06$\downarrow$ & 1.13$\downarrow$ & 1.18$\downarrow$ & 0.33 $\downarrow$& 0.47$\downarrow$ & 0.50 $\downarrow$                  \\
\midrule
\rowcolor{red!10} CLIP w. \fcut\  & 0.0635 & 0.2214 & 0.2035  & 0.0921 & 0.2414 & 0.2208       \\
Difference to the baseline(\%) & 58.11 $\downarrow$ & 32.54 $\downarrow$& 35.15 $\downarrow$& 39.25 $\downarrow$& 24.87 $\downarrow$& 26.62 $\downarrow$ \\
\bottomrule
\end{tabular}%
}
\end{table*}

\section{Experiments}

We conduct a series of experiments to demonstrate that \ours\ can efficiently alleviate the \problem\ in a post-hoc manner.
We compare \ours\ and \fcut\ with the baseline performance produced by the representation model adopted. 
We also introduce the inverse version of average pooling, namely \pool, as another baseline. 
A series of ablation studies indicate that \ours\ addresses the \problem\ by reducing the similarity between points in the gallery set and is not sensitive to hyperparameters and the amount of training data.

\subsection{Experimental and implementation settings} \label{sec:setting}

The implementation of \ours\ exactly follows \Cref{eq:formaldef} with the adjacency matrix defined in \Cref{eq:sim1,eq:sim2} and a separate pair of tuned hyperparameters $r_g$ and $r_q$ for each retrieval model of each dataset. To balance the edge weight in $\mathcal{S}^{q}$ and $\mathcal{S}^{g}$ and make sure the scale of weights is stable, they subtract their average edge weights, respectively.

The only difference between \fcut\ and \ours\ is the adjacency matrix applied. Instead of the ones from \Cref{eq:sim1,eq:sim2}, we apply $\mathcal{S}^{q}$ and $\mathcal{S}^{g}$ defined in \Cref{eq:sim3,eq:sim4}. We choose the value of $k$ to be 1 (\ie, top 1\% nearest neighbors) throughout the study while the proposed method is very robust to the value of the $k$. There is a large range of reasonable values that can boost the performance, proved by the ablation study in \Cref{sec: ablation}.

\pool\ is the simplest variant of inverse graph convolution with binary adjacency matrix where the edge weight between a pair of nodes is 1 if they are neighbors, or 0 otherwise. 
Then, following the approach in \tail, we update the embeddings using only the top $p$ percent of nearest neighbors.
To provide a more comprehensive benchmark, we pick four values of $p$ from $10\%$ to $100\%$. Note that we also independently tune $r_g$ and $r_q$ for \pool\ to make sure the evaluation is fair. The comparison with \pool\ validates not only the effectiveness of \ours\ but also the importance of the similarity-based adjacency matrix. The detailed setting is included in \Cref{sec:pool}.

\subsection{Datasets and evaluation metrics}\label{sec:data}
While we mainly focus our experiments on standard benchmarks for text-image retrieval, \ie, MSCOCO~\cite{DBLP:conf/eccv/LinMBHPRDZ14} and Flickr30k~\cite{DBLP:journals/ijcv/PlummerWCCHL17}, we also explore the generalization of \ours\ by selecting four text-video retrieval benchmarks (MSR-VTT~\cite{DBLP:conf/cvpr/XuMYR16}, MSVD~\cite{chen-dolan-2011-collecting}, Didemo~\cite{hendricks_localizing_2017}, and ActivityNet~\cite{caba2015activitynet}) and two text-audio retrieval benchmark (AudioCaps~\cite{kim-etal-2019-audiocaps} and CLOTHO~\cite{drossos_clotho_2020}). 
The details of the benchmarks are deferred to \Cref{sec:dateapp}.

To evaluate the retrieval performance of \ours, we use recall at Rank K (R@K, higher is better), median rank (MdR, lower is better), and mean rank (MnR, lower is better) as retrieval metrics, which are widely used in previous retrieval works~\cite{DBLP:journals/ijon/LuoJZCLDL22,DBLP:conf/mm/MaXSYZJ22,DBLP:conf/icml/RadfordKHRGASAM21}.

\subsection{\ours\ and \fcut} \label{sec: ablation}
In this section, to answer a series of questions relating to \ours\ and \fcut, we investigate its performance on MSCOCO with CLIP given different settings. Due to space limitations, discussions of the sensitivity of \fcut\ to $k$ and the complexity of both methods are included in the Appendix.

\textbf{RQ1: Is the data degeneration problem alleviated?} 
This problem can be firstly explained by the change of the similarity measure within the test gallery set $G$. 
As presented in \Cref{tab: sim measure intra}, we collect the mean similarity of the gallery set of both tasks for three scenarios, the overall mean (MeanSim), the mean between the nearest neighbor(MeanSim@1), and the mean between nearest 10 neighbors (MeanSim@10). Note that MeanSim@1 is strictly equivalent to $\Delta_{deg}(G)$.
It is very clear that \ours\ and \fcut\ do reduce the similarity on all accounts, especially MeanSim@1, indicating a targeted ability to alleviate the data degeneration issue.

Besides, given the assumption that the distribution of the test query set $Q$ is close to that of $G$, the similarity score of the gallery set to the query set for both tasks is also worth exploring. Since any point in $G$ should theoretically share an identical representation with its corresponding query in $Q$ and we try to maximize the similarity between them, we exclude this similarity score between them. Consequently, we want the similarity to be as small as possible since it reflects the margin of the retrieval task, as a gallery point is supposed to be as far from the irrelevant queries as possible to have a more robust result. We adopt the same metrics as \Cref{tab: sim measure intra}, and the results are presented in \Cref{tab: sim measure inter}. 
Again, we observe a comprehensive decrease in the similarity score, especially between the nearest neighbors. Note that, compared to \ours, \fcut\ can better address the \problem, validating that the design of \tail\ does help alleviate the problem by focusing more on the local information. 
Not for MSCOCO alone, we witness exactly similar results across all the datasets and methods, whose detail is included in \textbf{Continuation on RQ1} \Cref{sec:more_ablt}.

\begin{figure}[t!]
    \centering
    \subcaptionbox{Image-to-Text R@1 w.r.t $r_g$}{\includegraphics[width=0.235\textwidth]{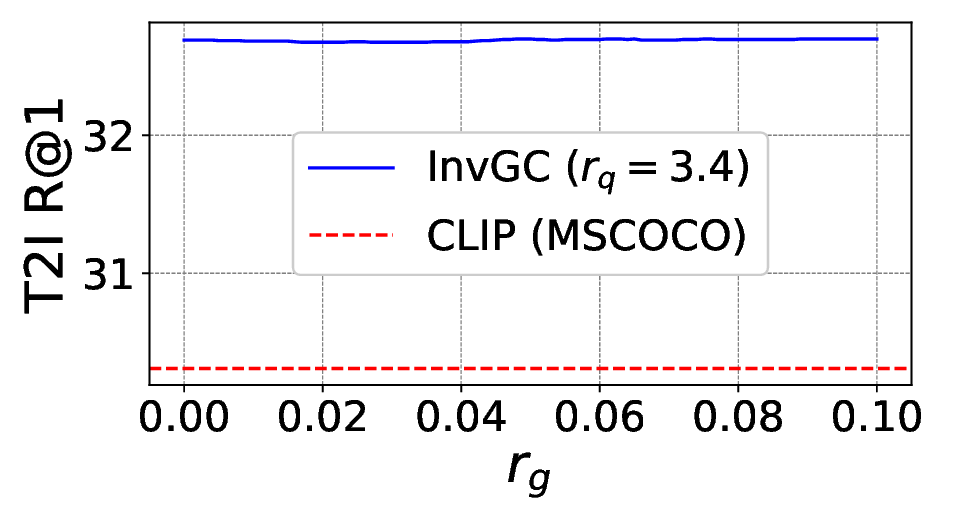}}
    \subcaptionbox{Image-to-Text R@1 w.r.t $r_q$}{\includegraphics[width=0.235\textwidth]{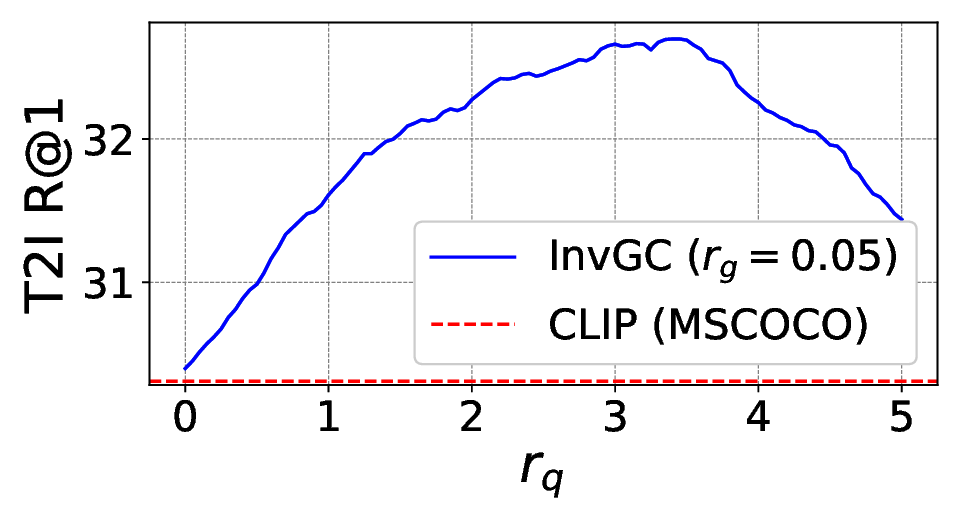}}
    
    \subcaptionbox{Text-to-Image R@1 w.r.t $r_g$}{\includegraphics[width=0.235\textwidth]{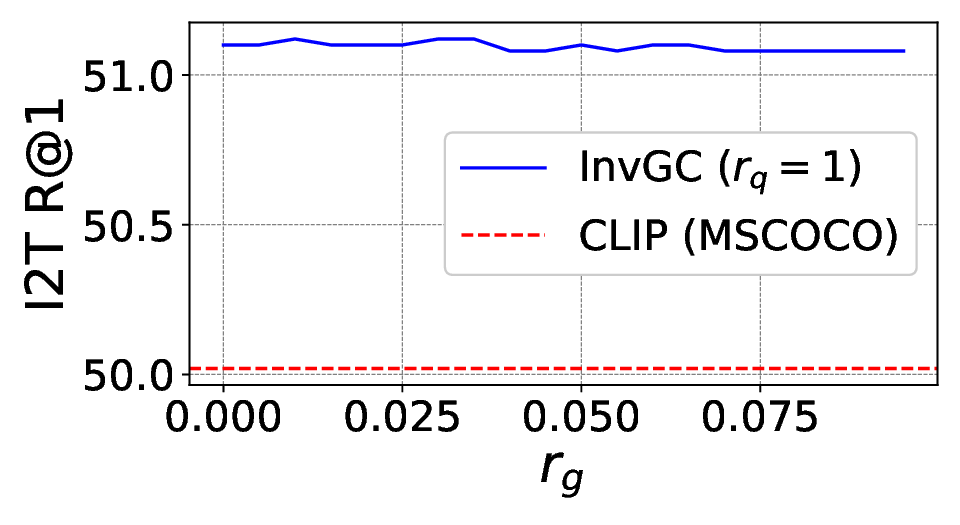}}
    \subcaptionbox{Text-to-Image R@1 w.r.t $r_q$}{\includegraphics[width=0.235\textwidth]{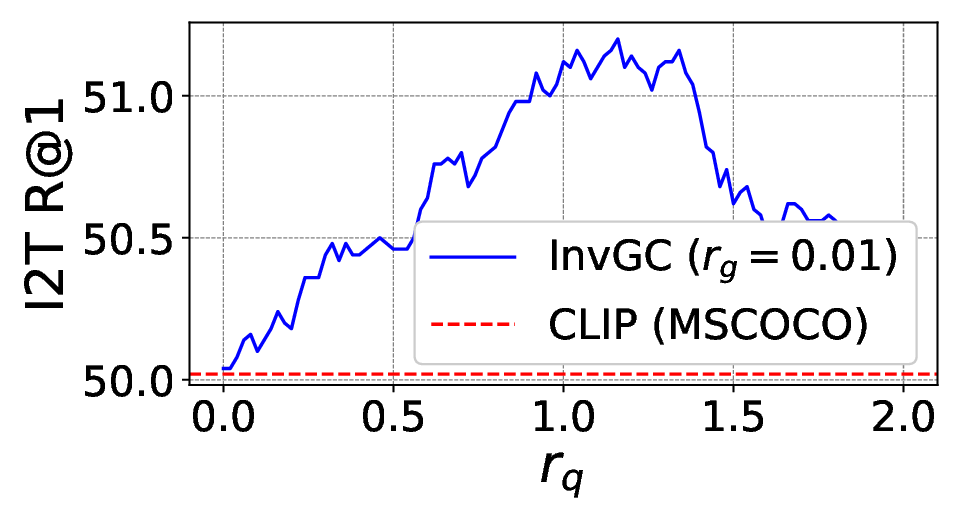}}
    \caption{Hyperpapameters sensitivity of \ours\ on MSCOCO.
    }
    \label{fig: hyper sen}
\end{figure}

\begin{figure}[t!]
    \centering
    \subcaptionbox{Image-to-Text R@1 w.r.t the size of data we use.}{\includegraphics[width=0.235\textwidth]{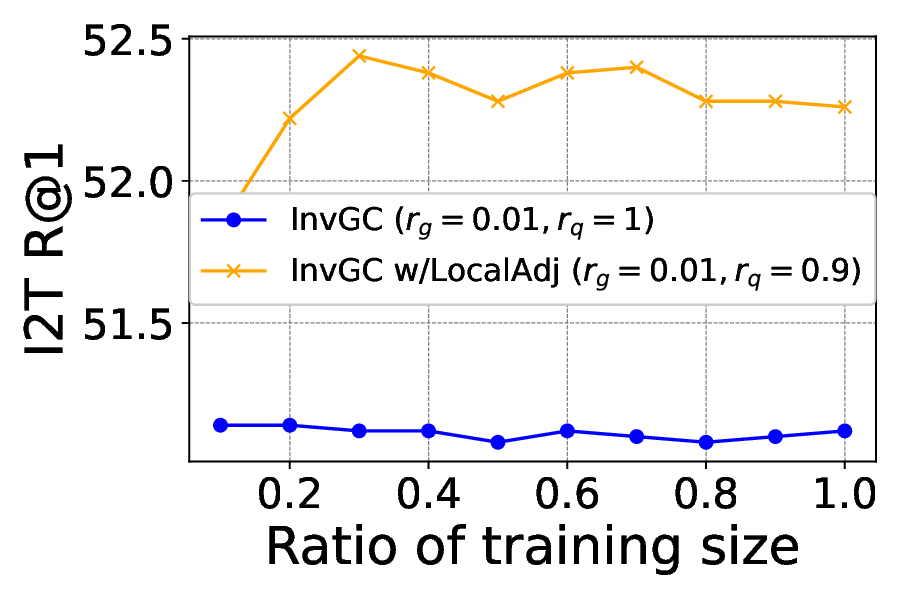}}
    \subcaptionbox{Text-to-Image R@1 w.r.t the size of data we use.}{\includegraphics[width=0.235\textwidth]{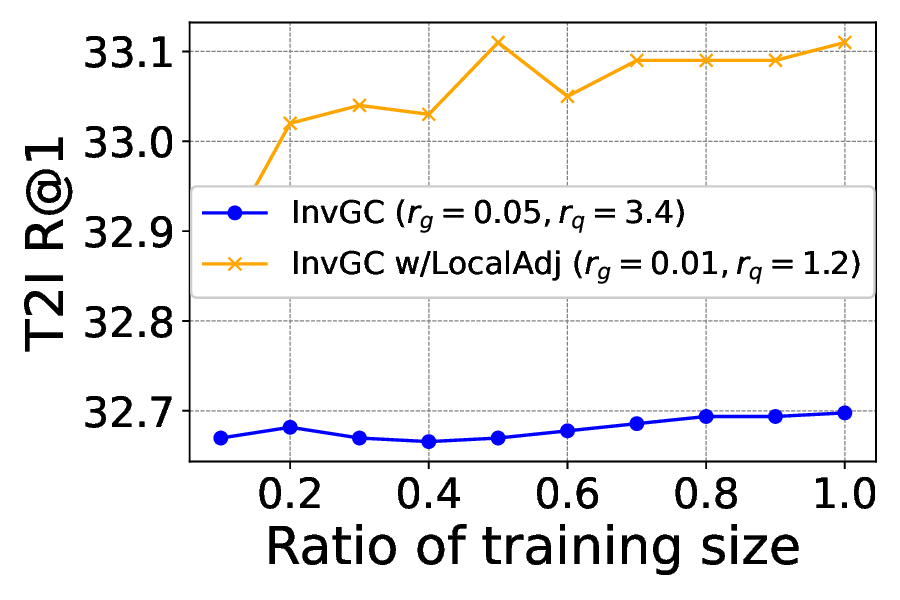}}
    \caption{Image-text retrieval performance w.r.t the size on
MSCOCO.}
    \label{fig: train rate}
\end{figure}

\textbf{RQ2: Is \ours\ (or \fcut) sensitive to the hyperparameter $r_g$ and $r_q$?}
To answer the question, we evaluate the R@1 metrics of \ours\ with a very large range of hyperparameters respective to the optimal choice adopted. We defer the analysis of \fcut\ to \Cref{sec:more_ablt}. For each task, we fix one of $r_g$ or $r_q$ and tune the other to show the change on the R@1 metrics. 
The results of the MSCOCO dataset are shown in \Cref{fig: hyper sen}. Although subject to some variation, \ours\ constantly outperforms the baseline, which is presented as the red dashed line. This indicates that the proposed method can consistently improve performance with a very large range of parameters.

\textbf{RQ3: How much data is needed for both proposed methods?}
Since we use the training query and gallery set as a sampled subset from the hidden distribution of representation, it is quite intuitive to ask if the performance of \ours\ or \fcut\ is sensitive to the size of this sampled subset. Therefore, we uniformly sample different ratios of data from both the training query and gallery set at the same time and evaluate the performance of both methods with the identical hyperparameters, as presented in \Cref{fig: train rate}.
From the results, we conclude that both proposed methods perform stably and robustly regardless of the size of the training data.

\textbf{RQ4: Is \fcut\ sensitive to the hyperparameter $k$?}
To address the question, we evaluate the R@1 metrics of \fcut\ with a very large range of possible $k$ values (even in logarithmic scale) compared to the optimal choice adopted(\ie,$k=1$). For each task, we fix everything except the $k$ value.The results of MSCOCO dataset are shown in \Cref{fig: k_value}. Compared with the baseline, it is safe to claim that the proposed \fcut\ is very robust to the choice of $k$. 

\begin{figure}[ht!]
    \centering
    \subcaptionbox{Image-to-Text R@1 w.r.t the value of $k$ we use.}{\includegraphics[width=0.235\textwidth]{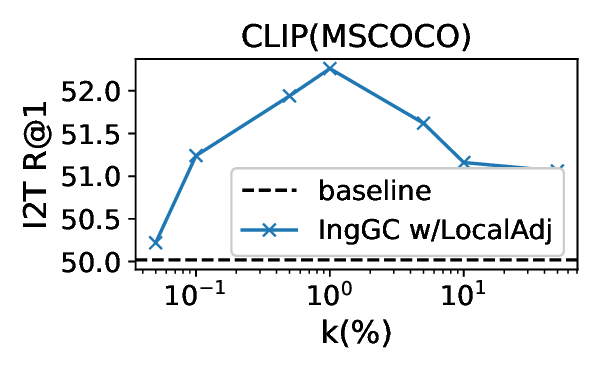}}
    \subcaptionbox{Text-to-Image R@1 w.r.t the value of $k$ we use.}{\includegraphics[width=0.235\textwidth]{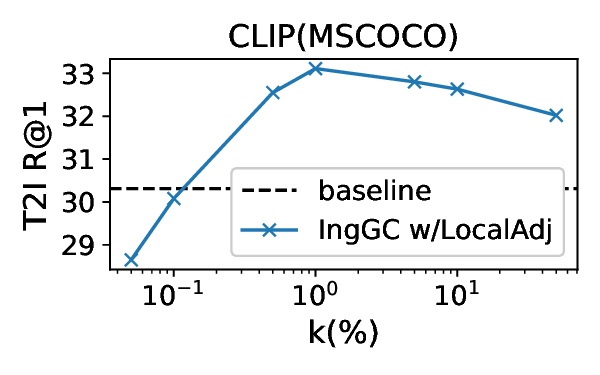}}
    \caption{The value of $k$ of \fcut\ on MSCOCO.}
    \label{fig: k_value}
\end{figure}

\subsection{Quantative results}

In this section, we present the quantitative results of four cross-modal retrieval benchmarks. Across eight different benchmarks, \ours\ and \fcut\ significantly improve upon the baselines, demonstrating the superiority of our proposed methods.

\begin{table}[t!]
\centering
\caption{Retrieval performance on MSCOCO (5k split). Best in \textbf{Bold} and the second best is \underline{underlined}. }
\label{tab: quan: coco_t2v}
\resizebox{\columnwidth}{!}{%
\begin{tabular}{ll|ccccc}
\toprule
                       & \multirow{2}{*}{Normalization}      & \multicolumn{5}{c}{Text-to-Image Retrieval}                                                   \\
                        &                                & R@1 $\uparrow$ & R@5 $\uparrow$ & R@10 $\uparrow$ & MdR $\downarrow$ & MnR  $\downarrow$ \\ \midrule
\multirow{2}{*}{CLIP}  &                     & 30.34          & 54.74          & 66.08          & 4.0 & 25.39                \\
 & +\pool (ratio=0.1)  & 30.37          & 54.77          & 66.14          & 4.0 & {\ul 25.36}           \\
 & +\pool (ratio=0.25) & 30.37          & 54.77          & 66.14          & 4.0 & {\ul 25.36}            \\
 & +\pool (ratio=0.5)  & 30.38          & 54.77          & 66.10          & 4.0 & 25.38                 \\
 & +\pool (ratio=1)    & 30.39          & 54.77          & 66.11          & 4.0 & 25.38                 \\
 \rowcolor{red!10}& +\ours              & {\ul 32.70}    & {\ul 57.53}    & \textbf{68.24} & 4.0 & \textbf{24.35}     \\
 \rowcolor{red!10}& +\fcut        & \textbf{33.11} & \textbf{57.49} & {\ul 68.19}    & 4.0 & 28.95       \\ \midrule
\multirow{2}{*}{Oscar}  &                     & 52.50          & 80.03          & \textbf{87.96} & 1.0 & {\ul 10.68}           \\
 & +\pool (ratio=0.1)  & 52.52          & {\ul 80.04}    & {\ul 87.95}    & 1.0    & 10.70    \\
 & +\pool (ratio=0.25) & 52.52          & 80.03          & \textbf{87.96} & 1.0 & {\ul 10.68}        \\
 & +\pool (ratio=0.5)  & 52.51          & 80.00          & \textbf{87.96} & 1.0 & \textbf{10.67}        \\
 & +\pool (ratio=1)    & 52.50          & 80.02          & \textbf{87.96} & 1.0 & {\ul 10.68}       \\
\rowcolor{red!10} & +\ours              & {\ul 52.63}    & \textbf{80.05} & \textbf{87.96} & 1.0 & 10.72      \\
 \rowcolor{red!10}& +\fcut        & \textbf{52.93} & \textbf{80.05} & 87.78          & 1.0 & 11.09     \\ 
\bottomrule
\end{tabular}%
}
\end{table}

\begin{table}[t!]
\centering
\caption{Retrieval performance on Flickr30k. Best in \textbf{Bold} and the second best is \underline{underlined}.}
\label{tab: quan: f30k_t2v}
\resizebox{\columnwidth}{!}{%
\begin{tabular}{ll|ccccc}
\toprule
                       & \multirow{2}{*}{Normalization}      & \multicolumn{5}{c}{Text-to-Image Retrieval}                                                     \\
                        &                                & R@1 $\uparrow$ & R@5 $\uparrow$ & R@10 $\uparrow$ & MdR $\downarrow$ & MnR  $\downarrow$ \\ \midrule
\multirow{2}{*}{CLIP}   &                     & 58.98          & 83.48          & 90.14          & 1.0 & 6.04         \\
 & +\pool (ratio=0.1)  & 59.10          & \textbf{83.56} & {\ul 90.18}    & 1.0 & 6.04      \\
 & +\pool (ratio=0.25) & 59.10          & \textbf{83.56} & {\ul 90.18}    & 1.0 & 6.04       \\
 & +\pool (ratio=0.5)  & 59.10          & {\ul 83.54}    & {\ul 90.18}    & 1.0 & 6.05       \\
 & +\pool (ratio=1)    & 59.10          & {\ul 83.54}    & {\ul 90.18}    & 1.0 & 6.04      \\
 \rowcolor{red!10}& +\ours              & {\ul 60.18}    & 85.30          & \textbf{91.20} & 1.0 & \textbf{5.52}    \\
 \rowcolor{red!10}& +\fcut       & \textbf{60.48} & 85.30          & 91.10          & 1.0 & {\ul 5.59}   \\ \midrule
\multirow{2}{*}{Oscar} &                                & 71.60          & 91.50          & 94.96          & 1.0 & \textbf{4.24}               \\

  & +\pool (ratio=0.1)  & 71.62          & 91.44          & 94.92          & 1.0 & {\ul 4.25}           \\
 & +\pool (ratio=0.25) & 71.66          & 91.50          & 94.94          & 1.0 & \textbf{4.24}         \\
 & +\pool (ratio=0.5)  & 71.66          & 91.50          & 94.92          & 1.0 & \textbf{4.24}     \\
 & +\pool (ratio=1)    & 71.62          & {\ul 91.52}    & 94.96          & 1.0 & \textbf{4.24}  \\
 \rowcolor{red!10}& +\ours              & {\ul 71.68}    & 91.46          & \textbf{95.06} & 1.0 & {\ul 4.25}            \\
 \rowcolor{red!10}& +\fcut        & \textbf{71.74} & \textbf{91.56} & {\ul 94.98}    & 1.0 & 4.29    \\
\bottomrule
\end{tabular}%
}
\end{table}

\begin{table}[t!]
\centering
\caption{Retrieval performance on MSR-VTT (full split and 1k split). Best in \textbf{Bold} and the second best is \underline{underlined}. 
}{}
\label{tab: quan: msrvtt_t2v}
\resizebox{\columnwidth}{!}{%
\begin{tabular}{ll|ccccc}
\toprule
          &   \multirow{2}{*}{Normalization}          & \multicolumn{5}{c}{Text-to-Video Retrieval}         \\
            &          & R@1   & R@5   & R@10  & MdR & MnR     \\ \midrule

\multirow{2}{*}{CLIP4Clip} &                     & 44.10          & {\ul 71.70}    & 81.40          & 2.0 & {\ul 15.51}      \\
 & +\pool (ratio=0.1)  & {\ul 44.20}    & 71.60          & {\ul 81.50}    & 2.0 & 15.55                \\
 & +\pool (ratio=0.25) & {\ul 44.20}    & 71.60          & {\ul 81.50}    & 2.0 & 15.55                 \\
 & +\pool (ratio=0.5)  & {\ul 44.20}    & 71.50          & {\ul 81.50}    & 2.0 & 15.54                   \\
 & +\pool (ratio=1)    & 44.10          & 71.60          & {\ul 81.50}    & 2.0 & 15.52                \\
 \rowcolor{red!10}& +\ours              & \textbf{44.40} & \textbf{71.90} & \textbf{81.60} & 2.0 & \textbf{15.36}  \\
 \rowcolor{red!10}& +\fcut        & \textbf{44.40} & {\ul 71.70}    & 81.20          & 2.0 & 15.65        \\\midrule
\multirow{2}{*}{CLIP2Video} 
 &                     & 46.00          & 71.60          & {\ul 81.60}    & 2.0 & \multicolumn{1}{c}{14.51}             \\
 & +\pool (ratio=0.1)  & 45.90          & 71.70          & 81.50          & 2.0 & 14.52                  \\
 & +\pool (ratio=0.25) & 46.10          & {\ul 71.80}    & 81.50          & 2.0 & 14.53                  \\
 & +\pool (ratio=0.5)  & 46.00          & 71.70          & 81.50          & 2.0 & 14.53                 \\
 & +\pool (ratio=1)    & 45.90          & 71.70          & 81.50          & 2.0 & 14.52                   \\
 \rowcolor{red!10}& +\ours              & {\ul 46.20}    & 71.70          & 81.30          & 2.0 & \textbf{14.44}  \\
 \rowcolor{red!10}& +\fcut        & \textbf{46.60} & \textbf{72.10} & \textbf{81.70} & 2.0 & {\ul 14.50}          \\ \midrule
\multirow{2}{*}{X-CLIP}   &                     & 46.30          & {\ul 74.00}    & {\ul 83.40}    & 2.0 & \multicolumn{1}{c}{\textbf{12.80}}         \\
 & +\pool (ratio=0.1)  & 46.50          & {\ul 74.00}    & {\ul 83.40}    & 2.0 & 12.88                  \\
 & +\pool (ratio=0.25) & 46.40          & {\ul 74.00}    & \textbf{83.50} & 2.0 & 12.85                  \\
 & +\pool (ratio=0.5)  & 46.30          & {\ul 74.00}    & 83.30          & 2.0 & {\ul 12.83}       \\
 & +\pool (ratio=1)    & 46.20          & {\ul 74.00}    & {\ul 83.40}    & 2.0 & {\ul 12.83}     \\
 \rowcolor{red!10}& +\ours              & \textbf{47.30} & {\ul 74.00}    & 83.30          & 2.0 & 13.42                  \\
 \rowcolor{red!10}& +\fcut        & {\ul 47.10}    & \textbf{74.20} & \textbf{83.50} & 2.0 & 13.09          \\
\bottomrule
\end{tabular}%
 }
\end{table}

\begin{table}[t!]
\centering
\caption{Retrieval performance on ActivityNet. Best in \textbf{Bold} and the second best is \underline{underlined}.}
\label{tab: quan: actnet_t2v}
\resizebox{\columnwidth}{!}{%
\begin{tabular}{ll|ccccc}
\toprule
                           & \multirow{2}{*}{Normalization} & \multicolumn{5}{c}{Text-to-Video Retrieval}                                                        \\
                           &                                & R@1            & R@5            & R@10           & MdR & MnR                      \\ \midrule
\multirow{2}{*}{CLIP4Clip}  &                     & 41.85          & 74.44          & 84.84          & 2.0 & {\ul 6.84}            \\
 & +\pool (ratio=0.1)  & 41.83          & {\ul 74.47}    & 84.84          & 2.0 & {\ul 6.84}             \\
 & +\pool (ratio=0.25) & 41.80          & 74.44          & 84.84          & 2.0 & {\ul 6.84}            \\
 & +\pool (ratio=0.5)  & 41.85          & 74.44          & 84.84          & 2.0 & {\ul 6.84}             \\
 & +\pool (ratio=1)    & 41.88          & 74.44          & 84.84          & 2.0 & {\ul 6.84}             \\
 \rowcolor{red!10}& +\ours              & {\ul 41.90}    & 74.40          & {\ul 84.86}    & 2.0 & {\ul 6.84}       \\
 \rowcolor{red!10}& +\fcut        & \textbf{43.23} & \textbf{75.58} & \textbf{85.74} & 2.0 & \textbf{6.82}\\\midrule
\multirow{2}{*}{X-CLIP}          &                     & 46.25          & {\ul 76.02}    & 86.05          & 2.0 & \multicolumn{1}{c}{{\ul 6.37}}        \\
 & +\pool (ratio=0.1)  & {\ul 46.47}    & 75.94          & 86.05          & 2.0 & 6.38                \\
 & +\pool (ratio=0.25) & 46.38          & 75.98          & 86.01          & 2.0 & 6.38                  \\
 & +\pool (ratio=0.5)  & {\ul 46.47}    & 75.94          & 86.01          & 2.0 & 6.38                  \\
 & +\pool (ratio=1)    & 46.43          & 75.89          & 86.01          & 2.0 & 6.38                 \\
\rowcolor{red!10} & +\ours              & 46.43          & 75.68          & {\ul 86.22}    & 2.0 & \textbf{6.35}    \\
 \rowcolor{red!10}& +\fcut        & \textbf{47.82} & \textbf{76.46} & \textbf{86.36} & 2.0 & 6.91 \\\bottomrule
\end{tabular}%
}
\end{table}

\begin{table}[h!]
\centering
\caption{Retrieval performance on MSVD. Best in \textbf{Bold} and the second best is \underline{underlined}.}
\label{tab: quan: mvd_t2v}
\resizebox{\columnwidth}{!}{%
\begin{tabular}{ll|ccccc}
\toprule
                            & \multirow{2}{*}{Normalization} & \multicolumn{5}{c}{Text-to-Video Retrieval}                                                       \\
                            &                                & R@1            & R@5            & R@10           & MdR & MnR            \\\midrule
\multirow{2}{*}{CLIP4Clip}   &                     & 44.64          & 74.66          & 83.99          & 2.0 & {\ul 10.32}           \\
 & +\pool (ratio=0.1)  & 44.87          & 73.89          & 83.07          & 2.0 & 11.93               \\
 & +\pool (ratio=0.25) & 45.06          & 74.04          & 83.49          & 2.0 & \textbf{11.29}       \\
 & +\pool (ratio=0.5)  & 45.12          & 74.32          & 83.66          & 2.0 & 10.76              \\
 & +\pool (ratio=1)    & 45.09          & 74.54          & 83.91          & 2.0 & 10.45             \\
 \rowcolor{red!10}& +\ours              & {\ul 45.43}    & {\ul 74.82}    & {\ul 84.00}    & 2.0 & 10.42        \\
 \rowcolor{red!10}& +\fcut        & \textbf{45.73} & \textbf{75.53} & \textbf{84.37} & 2.0 & 10.42       \\\midrule
\multirow{2}{*}{CLIP2Video}  &                     & 47.05          & 76.97          & 85.59          & 2.0 & {\ul 9.53}           \\
 & +\pool (ratio=0.1)  & 47.04          & 76.98          & 85.61          & 2.0 & 9.54             \\
 & +\pool (ratio=0.25) & 47.07          & 76.98          & 85.61          & 2.0 & 9.54              \\
 & +\pool (ratio=0.5)  & 47.06          & 76.98          & 85.62          & 2.0 & {\ul 9.53}        \\
 & +\pool (ratio=1)    & 47.06          & 76.97          & 85.61          & 2.0 & {\ul 9.53}          \\
 \rowcolor{red!10}& +\ours              & {\ul 47.09}    & {\ul 77.00}    & {\ul 85.64}    & 2.0 & \textbf{9.48}  \\
 \rowcolor{red!10}& +\fcut        & \textbf{47.47} & \textbf{77.46} & \textbf{85.84} & 2.0 & {\ul 9.53}    \\ \midrule
\multirow{2}{*}{X-CLIP}     &                     & 46.31          & \textbf{76.84} & {\ul 85.31}    & 2.0 & \textbf{9.59}         \\
 & +\pool (ratio=0.1)  & 46.36          & 76.70          & 85.28          & 2.0 & 9.66               \\
 & +\pool (ratio=0.25) & 46.33          & 76.71          & 85.28          & 2.0 & 9.64                \\
 & +\pool (ratio=0.5)  & 46.35          & 76.71          & 85.25          & 2.0 & 9.62               \\
 & +\pool (ratio=1)    & 46.41          & 76.77          & 85.27          & 2.0 & {\ul 9.60}          \\
 \rowcolor{red!10}& +\ours              & \textbf{46.82} & 76.69          & \textbf{85.38} & 2.0 & 9.63           \\
 \rowcolor{red!10}& +\fcut        & {\ul 46.49}    & {\ul 76.82}    & 85.29          & 2.0 & 9.63       \\ \bottomrule
\end{tabular}%
}
\end{table}

\begin{table}[h!]
\centering
\caption{Text-audio retrieval performance on CLOTHO. Best in \textbf{Bold} and the second best is \underline{underlined}. 
'*' refers to the results direct copied from \citep{koepke_audio_2022}. 
MoEE~\cite{miech_learning_2020}, MMT~\cite{DBLP:conf/eccv/Gabeur0AS20} are two methods used for audio-text retrieval.
}
\label{tab: quan: clotho_t2v}
\resizebox{\columnwidth}{!}{%
\begin{tabular}{ll|ccccc}
\toprule
\multirow{2}{*}{Method} & \multirow{2}{*}{Normalization}     & \multicolumn{5}{c}{Text-to-Audio Retrieval}  \\
                        &                                    & R@1 $\uparrow$    & R@5 $\uparrow$    & R@10 $\uparrow$   & MdR $\downarrow$   & MnR  $\downarrow$     \\\midrule
                        \textcolor{gray}{MoEE}*                                                &                                                     & \textcolor{gray}{6.00  }         & \textcolor{gray}{20.80  }        & \textcolor{gray}{32.30  }         & \textcolor{gray}{23.0   }          & \textcolor{gray}{60.20   }                            \\
\textcolor{gray}{MMT}*                                                 &                                                     & \textcolor{gray}{6.50     }      & \textcolor{gray}{21.60   }       & \textcolor{gray}{66.90 }          & \textcolor{gray}{23.0}             & \textcolor{gray}{67.70}                                      \\ \midrule
\multirow{2}{*}{AR-CE}  &                     & 6.27          & 22.32          & 33.30          & 23.0          & 58.96            \\
 & +\pool (ratio=0.1)  & 6.37          & {\ul 22.33}    & 33.59          & 23.0          & 65.57                   \\
 & +\pool (ratio=0.25) & 6.47          & 22.12          & 33.58          & {\ul 22.0}    & 59.84                   \\
 & +\pool (ratio=0.5)  & 6.39          & 22.24          & 33.30          & {\ul 22.0}    & 59.14             \\
 & +\pool (ratio=1)    & 6.28          & 22.32          & 33.30          & 23.0          & {\ul 58.95}     \\
 \rowcolor{red!10}& +\ours              & \textbf{6.81} & 22.14          & \textbf{34.72} & \textbf{21.0} & 61.98             \\
 \rowcolor{red!10}& +\fcut        & {\ul 6.58}    & \textbf{22.35} & {\ul 34.64}    & {\ul 22.0}    & \textbf{58.51}  \\ 
                        \bottomrule
\end{tabular}%
}
\end{table}

\begin{table}[h!]
\centering
\caption{
Text-audio retrieval performance on AudioCaps. Best in \textbf{Bold} and the second best is \underline{underlined}. 
'*' refers to the results direct copied from \citet{koepke_audio_2022}. 
MoEE~\cite{miech_learning_2020}, MMT~\cite{DBLP:conf/eccv/Gabeur0AS20} are two methods used for audio-text retrieval.
}
\label{tab: quan: audiocaps_t2v}
\resizebox{\columnwidth}{!}{%
\begin{tabular}{ll|ccccc}
\toprule
\multirow{2}{*}{Method} & \multirow{2}{*}{Normalization} & \multicolumn{5}{c}{Text-to-Audio Retrieval}                                                                                \\
                        &                                & R@1 $\uparrow$ & R@5 $\uparrow$ & R@10 $\uparrow$ & MdR $\downarrow$ & MnR  $\downarrow$  \\ \midrule
\textcolor{gray}{MoEE}*                                                &                                                     & \textcolor{gray}{23.00}          & \textcolor{gray}{55.70}          & \textcolor{gray}{71.00}           & \textcolor{gray}{4.0}              & \textcolor{gray}{16.30}                                              \\
\textcolor{gray}{MMT}*                                                 &                                                     & \textcolor{gray}{36.10}          & \textcolor{gray}{72.00}          & \textcolor{gray}{84.50}           & \textcolor{gray}{2.3}              & \textcolor{gray}{7.50}                                            \\ \midrule
\multirow{2}{*}{AR-CE}    &                     & {\ul 22.33}    & {\ul 54.49}    & {\ul 70.54}    & {\ul 5.0}    & \textbf{15.89}  \\
 & +\pool (ratio=0.1)  & {\ul 22.33}    & {\ul 54.49}    & {\ul 70.54}    & {\ul 5.0}    & \textbf{15.89}           \\
 & +\pool (ratio=0.25) & {\ul 22.33}    & {\ul 54.49}    & {\ul 70.54}    & {\ul 5.0}    & \textbf{15.89}           \\
 & +\pool (ratio=0.5)  & {\ul 22.33}    & {\ul 54.49}    & {\ul 70.54}    & {\ul 5.0}    & \textbf{15.89}           \\
 & +\pool (ratio=1)    & {\ul 22.33}    & {\ul 54.49}    & {\ul 70.54}    & {\ul 5.0}    & \textbf{15.89}   \\
 \rowcolor{red!10}& +\ours              & {\ul 22.33}    & 54.46          & \textbf{70.56} & {\ul 5.0}    & \textbf{15.89}          \\
 \rowcolor{red!10}& +\fcut        & \textbf{24.07} & \textbf{55.69} & 70.20          & \textbf{4.0} & {\ul 16.54}    \\ 
\bottomrule
\end{tabular}%
}
\end{table}

\textbf{Text-Image Retrieval.} 
Results are presented in \Cref{tab: quan: coco_t2v,tab: quan: f30k_t2v}. We observe that one of our methods achieves the best performance on R@1, R@5, and R@10 by a large margin.
When evaluated on the CLIP method, \fcut\ outperforms the baselines on both the MSCOCO and Flickr30k datasets, improving R@1 and R@5 by at least 2\% compared to all baselines.

\textbf{Text-Video Retrieval.} 
Results are presented in {\Cref{tab: quan: msrvtt_t2v,tab: quan: actnet_t2v,tab: quan: mvd_t2v}}. We can also conclude that one of our methods achieves the best performance on R@1, R@5, and R@10. Specifically, on the ActivityNet dataset,  \fcut\ shows excellent performance with both CLIP4CLIP and X-CLIP methods, significantly outperforming all the baselines on R@1 and R@5 roughly by 2\%.

\textbf{Text-Audio Retrieval.} 
Results are presented in \Cref{tab: quan: clotho_t2v,tab: quan: audiocaps_t2v}. On the CLOTHO dataset, \ours\ exhibits significantly better performance compared to all the baselines while \fcut\ achieves the best results on the AudioCaps dataset.

In summary, our experiments demonstrate that employing \ours\ consistently improves retrieval performance across different datasets and retrieval tasks. The models with \ours\ demonstrate better accuracy and ranking in retrieving relevant videos, images, and textual descriptions based on given queries. The complete results of retrieval performance can be found in \Cref{sec:more exp}.

\section{Conclusion}
This paper addressed \problem\ in cross-modal retrieval, which led to a decrease in retrieval performance. 
The
\problem\ was validated across multiple benchmarks and methods. 
To alleviate this issue, we proposed a novel method called \ours, inspired by graph convolution and average pooling. 
The method established a graph topology structure within the datasets and applied graph convolution in an inverse form with subtraction over the neighborhood. Additionally, we designed the adjacency matrix, \tail, that only leveraged the nearest neighbors of each data point rather than the entire dataset, resulting in a more effective and efficient method, \fcut.
Both \ours\ and \fcut\ were validated through theoretical analysis and demonstrated their ability to separate representations. Finally, extensive experiments on various cross-modal benchmarks showed that both of our methods successfully alleviated the problem of representation degeneration and, as a result, improved retrieval performance.

\section*{Limitations}
First, although \ours\ has been validated through theoretical analysis and demonstrated its efficacy in separating representations, its performance may vary across different datasets and modalities. 
The effectiveness of our method might be influenced by variations in dataset characteristics, such as data distribution, scale, and complexity. 
Further investigation and experimentation on a wider range of datasets are needed to fully understand the generalizability of \ours\ and its performance under diverse conditions.

Second, while our method shows promising results in alleviating the \problem, it is worth noting that cross-modal retrieval tasks can still pose challenges due to inherent differences in modalities. 
Variations in feature spaces, data modalities, and semantic gaps between modalities may limit the overall retrieval performance. 
Future research efforts should focus on exploring complementary techniques, such as multimodal fusion, attention mechanisms, or domain adaptation, to further enhance the retrieval accuracy and alleviate \problem.

In the future, it would be interesting to explore the performance of \ours\ on diverse datasets, the challenges associated with cross-modal differences, and better definitions or metrics for measuring \problem.

{
\bibliography{custom,references}

\begin{thebibliography}{64}
\expandafter\ifx\csname natexlab\endcsname\relax\def\natexlab#1{#1}\fi

\bibitem[{Baranwal et~al.(2023)Baranwal, Fountoulakis, and
  Jagannath}]{baranwal2023effects}
Aseem Baranwal, Kimon Fountoulakis, and Aukosh Jagannath. 2023.
\newblock \href {https://openreview.net/forum?id=P-73JPgRs0R} {Effects of graph
  convolutions in multi-layer networks}.
\newblock In \emph{The Eleventh International Conference on Learning
  Representations}.

\bibitem[{Bellman(2015)}]{DBLP:books/degruyter/Bellman15}
Richard Bellman. 2015.
\newblock \href {https://doi.org/10.1515/9781400874668} {\emph{Adaptive Control
  Processes - {A} Guided Tour (Reprint from 1961)}}, volume 2045 of
  \emph{Princeton Legacy Library}.
\newblock Princeton University Press.

\bibitem[{Bogolin et~al.(2022)Bogolin, Croitoru, Jin, Liu, and
  Albanie}]{DBLP:conf/cvpr/BogolinCJLA22}
Simion{-}Vlad Bogolin, Ioana Croitoru, Hailin Jin, Yang Liu, and Samuel
  Albanie. 2022.
\newblock \href {https://doi.org/10.1109/CVPR52688.2022.00513} {Cross modal
  retrieval with querybank normalisation}.
\newblock In \emph{{IEEE/CVF} Conference on Computer Vision and Pattern
  Recognition, {CVPR} 2022, New Orleans, LA, USA, June 18-24, 2022}, pages
  5184--5195. {IEEE}.

\bibitem[{Boureau et~al.(2010)Boureau, Ponce, and
  LeCun}]{10.5555/3104322.3104338}
Y-Lan Boureau, Jean Ponce, and Yann LeCun. 2010.
\newblock A theoretical analysis of feature pooling in visual recognition.
\newblock In \emph{Proceedings of the 27th International Conference on
  International Conference on Machine Learning}, ICML'10, page 111–118,
  Madison, WI, USA. Omnipress.

\bibitem[{Bruna et~al.(2014)Bruna, Zaremba, Szlam, and
  LeCun}]{bruna2014spectral}
Joan Bruna, Wojciech Zaremba, Arthur Szlam, and Yann LeCun. 2014.
\newblock \href {http://arxiv.org/abs/1312.6203} {Spectral networks and locally
  connected networks on graphs}.

\bibitem[{Chen and Dolan(2011)}]{chen-dolan-2011-collecting}
David Chen and William Dolan. 2011.
\newblock \href {https://aclanthology.org/P11-1020} {Collecting highly parallel
  data for paraphrase evaluation}.
\newblock In \emph{Proceedings of the 49th Annual Meeting of the Association
  for Computational Linguistics: Human Language Technologies}, pages 190--200,
  Portland, Oregon, USA. Association for Computational Linguistics.

\bibitem[{Chen et~al.(2020)Chen, Zhao, Jin, and Wu}]{DBLP:conf/cvpr/ChenZJW20}
Shizhe Chen, Yida Zhao, Qin Jin, and Qi~Wu. 2020.
\newblock \href {https://doi.org/10.1109/CVPR42600.2020.01065} {Fine-grained
  video-text retrieval with hierarchical graph reasoning}.
\newblock In \emph{2020 {IEEE/CVF} Conference on Computer Vision and Pattern
  Recognition, {CVPR} 2020, Seattle, WA, USA, June 13-19, 2020}, pages
  10635--10644. Computer Vision Foundation / {IEEE}.

\bibitem[{Cheng et~al.(2021)Cheng, Lin, Wu, Yang, and
  Shen}]{DBLP:journals/corr/abs-2109-04290}
Xing Cheng, Hezheng Lin, Xiangyu Wu, Fan Yang, and Dong Shen. 2021.
\newblock \href {http://arxiv.org/abs/2109.04290} {Improving video-text
  retrieval by multi-stream corpus alignment and dual softmax loss}.
\newblock \emph{CoRR}, abs/2109.04290.

\bibitem[{Croitoru et~al.(2021)Croitoru, Bogolin, Leordeanu, Jin, Zisserman,
  Albanie, and Liu}]{DBLP:conf/iccv/CroitoruBLJZAL21}
Ioana Croitoru, Simion{-}Vlad Bogolin, Marius Leordeanu, Hailin Jin, Andrew
  Zisserman, Samuel Albanie, and Yang Liu. 2021.
\newblock \href {https://doi.org/10.1109/ICCV48922.2021.01138} {Teachtext:
  Crossmodal generalized distillation for text-video retrieval}.
\newblock In \emph{2021 {IEEE/CVF} International Conference on Computer Vision,
  {ICCV} 2021, Montreal, QC, Canada, October 10-17, 2021}, pages 11563--11573.
  {IEEE}.

\bibitem[{de~la Pena and Montgomery-Smith(1995)}]{10.1214/aop/1176988291}
Victor~H. de~la Pena and S.~J. Montgomery-Smith. 1995.
\newblock \href {https://doi.org/10.1214/aop/1176988291} {{Decoupling
  Inequalities for the Tail Probabilities of Multivariate $U$-Statistics}}.
\newblock \emph{The Annals of Probability}, 23(2):806 -- 816.

\bibitem[{Devlin et~al.(2019)Devlin, Chang, Lee, and
  Toutanova}]{DBLP:conf/naacl/DevlinCLT19}
Jacob Devlin, Ming{-}Wei Chang, Kenton Lee, and Kristina Toutanova. 2019.
\newblock \href {https://doi.org/10.18653/v1/n19-1423} {{BERT:} pre-training of
  deep bidirectional transformers for language understanding}.
\newblock In \emph{Proceedings of the 2019 Conference of the North American
  Chapter of the Association for Computational Linguistics: Human Language
  Technologies, {NAACL-HLT} 2019, Minneapolis, MN, USA, June 2-7, 2019, Volume
  1 (Long and Short Papers)}, pages 4171--4186. Association for Computational
  Linguistics.

\bibitem[{Drossos et~al.(2020)Drossos, Lipping, and
  Virtanen}]{drossos_clotho_2020}
Konstantinos Drossos, Samuel Lipping, and Tuomas Virtanen. 2020.
\newblock \href {https://doi.org/10.1109/ICASSP40776.2020.9052990} {Clotho: an
  {Audio} {Captioning} {Dataset}}.
\newblock In \emph{{ICASSP} 2020 - 2020 {IEEE} {International} {Conference} on
  {Acoustics}, {Speech} and {Signal} {Processing} ({ICASSP})}, pages 736--740.
\newblock ISSN: 2379-190X.

\bibitem[{Fabian Caba~Heilbron and Niebles(2015)}]{caba2015activitynet}
Bernard~Ghanem Fabian Caba~Heilbron, Victor~Escorcia and Juan~Carlos Niebles.
  2015.
\newblock Activitynet: A large-scale video benchmark for human activity
  understanding.
\newblock In \emph{Proceedings of the IEEE Conference on Computer Vision and
  Pattern Recognition}, pages 961--970.

\bibitem[{Gabeur et~al.(2020)Gabeur, Sun, Alahari, and
  Schmid}]{DBLP:conf/eccv/Gabeur0AS20}
Valentin Gabeur, Chen Sun, Karteek Alahari, and Cordelia Schmid. 2020.
\newblock \href {https://doi.org/10.1007/978-3-030-58548-8\_13} {Multi-modal
  transformer for video retrieval}.
\newblock In \emph{Computer Vision - {ECCV} 2020 - 16th European Conference,
  Glasgow, UK, August 23-28, 2020, Proceedings, Part {IV}}, volume 12349 of
  \emph{Lecture Notes in Computer Science}, pages 214--229. Springer.

\bibitem[{Gan et~al.(2020)Gan, Chen, Li, Zhu, Cheng, and
  Liu}]{DBLP:conf/nips/Gan0LZ0020}
Zhe Gan, Yen{-}Chun Chen, Linjie Li, Chen Zhu, Yu~Cheng, and Jingjing Liu.
  2020.
\newblock \href
  {https://proceedings.neurips.cc/paper/2020/hash/49562478de4c54fafd4ec46fdb297de5-Abstract.html}
  {Large-scale adversarial training for vision-and-language representation
  learning}.
\newblock In \emph{Advances in Neural Information Processing Systems 33: Annual
  Conference on Neural Information Processing Systems 2020, NeurIPS 2020,
  December 6-12, 2020, virtual}.

\bibitem[{Gao et~al.(2019)Gao, He, Tan, Qin, Wang, and
  Liu}]{DBLP:conf/iclr/GaoHTQWL19}
Jun Gao, Di~He, Xu~Tan, Tao Qin, Liwei Wang, and Tie{-}Yan Liu. 2019.
\newblock \href {https://openreview.net/forum?id=SkEYojRqtm} {Representation
  degeneration problem in training natural language generation models}.
\newblock In \emph{7th International Conference on Learning Representations,
  {ICLR} 2019, New Orleans, LA, USA, May 6-9, 2019}. OpenReview.net.

\bibitem[{Gao et~al.(2021)Gao, Liu, Chen, Chang, Zhang, and
  Yuan}]{DBLP:journals/corr/abs-2111-05610}
Zijian Gao, Jingyu Liu, Sheng Chen, Dedan Chang, Hao Zhang, and Jinwei Yuan.
  2021.
\newblock \href {http://arxiv.org/abs/2111.05610} {{CLIP2TV:} an empirical
  study on transformer-based methods for video-text retrieval}.
\newblock \emph{CoRR}, abs/2111.05610.

\bibitem[{Gilmer et~al.(2017)Gilmer, Schoenholz, Riley, Vinyals, and
  Dahl}]{gilmer2017mpnn}
Justin Gilmer, Samuel~S. Schoenholz, Patrick~F. Riley, Oriol Vinyals, and
  George~E. Dahl. 2017.
\newblock Neural message passing for quantum chemistry.
\newblock In \emph{Proceedings of the 34th International Conference on Machine
  Learning - Volume 70}, ICML'17, page 1263–1272. JMLR.org.

\bibitem[{Gong et~al.(2013)Gong, Lazebnik, Gordo, and
  Perronnin}]{DBLP:journals/pami/GongLGP13}
Yunchao Gong, Svetlana Lazebnik, Albert Gordo, and Florent Perronnin. 2013.
\newblock \href {https://doi.org/10.1109/TPAMI.2012.193} {Iterative
  quantization: {A} procrustean approach to learning binary codes for
  large-scale image retrieval}.
\newblock \emph{{IEEE} Trans. Pattern Anal. Mach. Intell.}, 35(12):2916--2929.

\bibitem[{Goodfellow et~al.(2016)Goodfellow, Bengio, and
  Courville}]{Goodfellow-et-al-2016}
Ian Goodfellow, Yoshua Bengio, and Aaron Courville. 2016.
\newblock \emph{Deep Learning}.
\newblock MIT Press.
\newblock \url{http://www.deeplearningbook.org}.

\bibitem[{Gorti et~al.(2022)Gorti, Vouitsis, Ma, Golestan, Volkovs, Garg, and
  Yu}]{DBLP:conf/cvpr/GortiVMGVGY22}
Satya~Krishna Gorti, No{\"{e}}l Vouitsis, Junwei Ma, Keyvan Golestan, Maksims
  Volkovs, Animesh Garg, and Guangwei Yu. 2022.
\newblock \href {https://doi.org/10.1109/CVPR52688.2022.00495} {X-pool:
  Cross-modal language-video attention for text-video retrieval}.
\newblock In \emph{{IEEE/CVF} Conference on Computer Vision and Pattern
  Recognition, {CVPR} 2022, New Orleans, LA, USA, June 18-24, 2022}, pages
  4996--5005. {IEEE}.

\bibitem[{GraphConv()}]{simple-conv}
Simple GraphConv.
\newblock Pyg description of simple graph convolution.
\newblock
  \url{https://pytorch-geometric.readthedocs.io/en/latest/generated/torch_geometric.nn.conv.SimpleConv.html#torch_geometric.nn.conv.SimpleConv}.
\newblock Accessed: 2023-06-17.

\bibitem[{Hamilton(2020)}]{willgnnbook}
William~L. Hamilton. 2020.
\newblock Graph representation learning.
\newblock \emph{Synthesis Lectures on Artificial Intelligence and Machine
  Learning}, 14(3):1--159.

\bibitem[{He et~al.(2016)He, Zhang, Ren, and Sun}]{DBLP:conf/cvpr/HeZRS16}
Kaiming He, Xiangyu Zhang, Shaoqing Ren, and Jian Sun. 2016.
\newblock \href {https://doi.org/10.1109/CVPR.2016.90} {Deep residual learning
  for image recognition}.
\newblock In \emph{2016 {IEEE} Conference on Computer Vision and Pattern
  Recognition, {CVPR} 2016, Las Vegas, NV, USA, June 27-30, 2016}, pages
  770--778. {IEEE} Computer Society.

\bibitem[{Hendricks et~al.(2017)Hendricks, Wang, Shechtman, Sivic, Darrell, and
  Russell}]{hendricks_localizing_2017}
Lisa~Anne Hendricks, Oliver Wang, Eli Shechtman, Josef Sivic, Trevor Darrell,
  and Bryan Russell. 2017.
\newblock \href {https://doi.org/10.1109/ICCV.2017.618} {Localizing {Moments}
  in {Video} with {Natural} {Language}}.
\newblock In \emph{2017 {IEEE} {International} {Conference} on {Computer}
  {Vision} ({ICCV})}, pages 5804--5813, Venice. IEEE.

\bibitem[{Huang et~al.(2021)Huang, Niu, Liu, Ding, Xiao, Wu, and
  Peng}]{NEURIPS2021_Huang}
Zhenyu Huang, Guocheng Niu, Xiao Liu, Wenbiao Ding, Xinyan Xiao, Hua Wu, and
  Xi~Peng. 2021.
\newblock \href
  {https://proceedings.neurips.cc/paper_files/paper/2021/file/f5e62af885293cf4d511ceef31e61c80-Paper.pdf}
  {Learning with noisy correspondence for cross-modal matching}.
\newblock In \emph{Advances in Neural Information Processing Systems},
  volume~34, pages 29406--29419. Curran Associates, Inc.

\bibitem[{Keogh and Mueen(2017)}]{Keogh2017}
Eamonn Keogh and Abdullah Mueen. 2017.
\newblock \href {https://doi.org/10.1007/978-1-4899-7687-1_192} {\emph{Curse of
  Dimensionality}}, pages 314--315. Springer US, Boston, MA.

\bibitem[{Kim et~al.(2019)Kim, Kim, Lee, and Kim}]{kim-etal-2019-audiocaps}
Chris~Dongjoo Kim, Byeongchang Kim, Hyunmin Lee, and Gunhee Kim. 2019.
\newblock \href {https://doi.org/10.18653/v1/N19-1011} {{A}udio{C}aps:
  Generating captions for audios in the wild}.
\newblock In \emph{Proceedings of the 2019 Conference of the North {A}merican
  Chapter of the Association for Computational Linguistics: Human Language
  Technologies, Volume 1 (Long and Short Papers)}, pages 119--132, Minneapolis,
  Minnesota. Association for Computational Linguistics.

\bibitem[{Kim et~al.(2023)Kim, Ahn, Kim, Lee, Marsden, Sala, Kim, Han, Lee,
  Lee, Bae, Wu, Gao, Zhang, Yang, Guo, Lu, Oh, Cho, jin Kim, Kweon, Kim, Kang,
  Jhoo, Roh, Mun, Oh, Ak, Lee, Xu, Shen, Hwang, Shin, Lee, Park, Lee, Kwak,
  Wang, Wang, Gu, Lv, and Sun}]{kim2023nice}
Taehoon Kim, Pyunghwan Ahn, Sangyun Kim, Sihaeng Lee, Mark Marsden, Alessandra
  Sala, Seung~Hwan Kim, Bohyung Han, Kyoung~Mu Lee, Honglak Lee, Kyounghoon
  Bae, Xiangyu Wu, Yi~Gao, Hailiang Zhang, Yang Yang, Weili Guo, Jianfeng Lu,
  Youngtaek Oh, Jae~Won Cho, Dong jin Kim, In~So Kweon, Junmo Kim, Wooyoung
  Kang, Won~Young Jhoo, Byungseok Roh, Jonghwan Mun, Solgil Oh, Kenan~Emir Ak,
  Gwang-Gook Lee, Yan Xu, Mingwei Shen, Kyomin Hwang, Wonsik Shin, Kamin Lee,
  Wonhark Park, Dongkwan Lee, Nojun Kwak, Yujin Wang, Yimu Wang, Tiancheng Gu,
  Xingchang Lv, and Mingmao Sun. 2023.
\newblock \href {http://arxiv.org/abs/2309.01961} {Nice: Cvpr 2023 challenge on
  zero-shot image captioning}.

\bibitem[{Kipf and Welling(2017)}]{kipf2017semisupervised}
Thomas~N. Kipf and Max Welling. 2017.
\newblock \href {https://openreview.net/forum?id=SJU4ayYgl} {Semi-supervised
  classification with graph convolutional networks}.
\newblock In \emph{International Conference on Learning Representations}.

\bibitem[{Koepke et~al.(2022)Koepke, Oncescu, Henriques, Akata, and
  Albanie}]{koepke_audio_2022}
A.~Sophia Koepke, Andreea-Maria Oncescu, Joao Henriques, Zeynep Akata, and
  Samuel Albanie. 2022.
\newblock \href {https://doi.org/10.1109/TMM.2022.3149712} {Audio {Retrieval}
  with {Natural} {Language} {Queries}: {A} {Benchmark} {Study}}.
\newblock \emph{IEEE Transactions on Multimedia}, pages 1--1.
\newblock Conference Name: IEEE Transactions on Multimedia.

\bibitem[{Lei et~al.(2021)Lei, Li, Zhou, Gan, Berg, Bansal, and
  Liu}]{DBLP:conf/cvpr/LeiLZGBB021}
Jie Lei, Linjie Li, Luowei Zhou, Zhe Gan, Tamara~L. Berg, Mohit Bansal, and
  Jingjing Liu. 2021.
\newblock \href {https://doi.org/10.1109/CVPR46437.2021.00725} {Less is more:
  Clipbert for video-and-language learning via sparse sampling}.
\newblock In \emph{{IEEE} Conference on Computer Vision and Pattern
  Recognition, {CVPR} 2021, virtual, June 19-25, 2021}, pages 7331--7341.
  Computer Vision Foundation / {IEEE}.

\bibitem[{Li et~al.(2020)Li, Yin, Li, Zhang, Hu, Zhang, Wang, Hu, Dong, Wei,
  Choi, and Gao}]{DBLP:conf/eccv/Li0LZHZWH0WCG20}
Xiujun Li, Xi~Yin, Chunyuan Li, Pengchuan Zhang, Xiaowei Hu, Lei Zhang, Lijuan
  Wang, Houdong Hu, Li~Dong, Furu Wei, Yejin Choi, and Jianfeng Gao. 2020.
\newblock \href {https://doi.org/10.1007/978-3-030-58577-8\_8} {Oscar:
  Object-semantics aligned pre-training for vision-language tasks}.
\newblock In \emph{Computer Vision - {ECCV} 2020 - 16th European Conference,
  Glasgow, UK, August 23-28, 2020, Proceedings, Part {XXX}}, volume 12375 of
  \emph{Lecture Notes in Computer Science}, pages 121--137. Springer.

\bibitem[{Liang et~al.(2022)Liang, Zhang, Kwon, Yeung, and Zou}]{liang2022mind}
Weixin Liang, Yuhui Zhang, Yongchan Kwon, Serena Yeung, and James Zou. 2022.
\newblock \href {https://openreview.net/forum?id=S7Evzt9uit3} {Mind the gap:
  {Understanding} the modality gap in multi-modal contrastive representation
  learning}.
\newblock In \emph{Advances in neural information processing systems}.

\bibitem[{Lin et~al.(2014)Lin, Maire, Belongie, Hays, Perona, Ramanan,
  Doll{\'{a}}r, and Zitnick}]{DBLP:conf/eccv/LinMBHPRDZ14}
Tsung{-}Yi Lin, Michael Maire, Serge~J. Belongie, James Hays, Pietro Perona,
  Deva Ramanan, Piotr Doll{\'{a}}r, and C.~Lawrence Zitnick. 2014.
\newblock \href {https://doi.org/10.1007/978-3-319-10602-1\_48} {Microsoft
  {COCO:} common objects in context}.
\newblock In \emph{Computer Vision - {ECCV} 2014 - 13th European Conference,
  Zurich, Switzerland, September 6-12, 2014, Proceedings, Part {V}}, volume
  8693 of \emph{Lecture Notes in Computer Science}, pages 740--755. Springer.

\bibitem[{Liong et~al.(2017)Liong, Lu, Tan, and
  Zhou}]{DBLP:conf/iccv/LiongLT017}
Venice~Erin Liong, Jiwen Lu, Yap{-}Peng Tan, and Jie Zhou. 2017.
\newblock \href {https://doi.org/10.1109/ICCV.2017.439} {Cross-modal deep
  variational hashing}.
\newblock In \emph{{IEEE} International Conference on Computer Vision, {ICCV}
  2017, Venice, Italy, October 22-29, 2017}, pages 4097--4105. {IEEE} Computer
  Society.

\bibitem[{Liu et~al.(2019)Liu, Albanie, Nagrani, and
  Zisserman}]{DBLP:conf/bmvc/LiuANZ19}
Yang Liu, Samuel Albanie, Arsha Nagrani, and Andrew Zisserman. 2019.
\newblock \href {https://bmvc2019.org/wp-content/uploads/papers/0363-paper.pdf}
  {Use what you have: Video retrieval using representations from collaborative
  experts}.
\newblock In \emph{30th British Machine Vision Conference 2019, {BMVC} 2019,
  Cardiff, UK, September 9-12, 2019}, page 279. {BMVA} Press.

\bibitem[{Luo et~al.(2022)Luo, Ji, Zhong, Chen, Lei, Duan, and
  Li}]{DBLP:journals/ijon/LuoJZCLDL22}
Huaishao Luo, Lei Ji, Ming Zhong, Yang Chen, Wen Lei, Nan Duan, and Tianrui Li.
  2022.
\newblock \href {https://doi.org/10.1016/j.neucom.2022.07.028} {Clip4clip: An
  empirical study of {CLIP} for end to end video clip retrieval and
  captioning}.
\newblock \emph{Neurocomputing}, 508:293--304.

\bibitem[{Ma et~al.(2022)Ma, Xu, Sun, Yan, Zhang, and
  Ji}]{DBLP:conf/mm/MaXSYZJ22}
Yiwei Ma, Guohai Xu, Xiaoshuai Sun, Ming Yan, Ji~Zhang, and Rongrong Ji. 2022.
\newblock \href {https://doi.org/10.1145/3503161.3547910} {{X-CLIP:} end-to-end
  multi-grained contrastive learning for video-text retrieval}.
\newblock In \emph{{MM} '22: The 30th {ACM} International Conference on
  Multimedia, Lisboa, Portugal, October 10 - 14, 2022}, pages 638--647. {ACM}.

\bibitem[{Micciancio and Voulgaris(2010)}]{cap_vol}
Daniele Micciancio and Panagiotis Voulgaris. 2010.
\newblock Faster exponential time algorithms for the shortest vector problem.
\newblock In \emph{Proceedings of the Twenty-First Annual ACM-SIAM Symposium on
  Discrete Algorithms}, SODA '10, page 1468–1480, USA. Society for Industrial
  and Applied Mathematics.

\bibitem[{Miech et~al.(2021)Miech, Alayrac, Laptev, Sivic, and
  Zisserman}]{DBLP:conf/cvpr/MiechALSZ21}
Antoine Miech, Jean{-}Baptiste Alayrac, Ivan Laptev, Josef Sivic, and Andrew
  Zisserman. 2021.
\newblock \href {https://doi.org/10.1109/CVPR46437.2021.00970} {Thinking fast
  and slow: Efficient text-to-visual retrieval with transformers}.
\newblock In \emph{{IEEE} Conference on Computer Vision and Pattern
  Recognition, {CVPR} 2021, virtual, June 19-25, 2021}, pages 9826--9836.
  Computer Vision Foundation / {IEEE}.

\bibitem[{Miech et~al.(2020)Miech, Laptev, and Sivic}]{miech_learning_2020}
Antoine Miech, Ivan Laptev, and Josef Sivic. 2020.
\newblock \href {https://doi.org/10.48550/arXiv.1804.02516} {Learning a
  {Text}-{Video} {Embedding} from {Incomplete} and {Heterogeneous} {Data}}.
\newblock ArXiv:1804.02516 [cs].

\bibitem[{Oncescu et~al.(2021)Oncescu, Koepke, Henriques, Akata, and
  Albanie}]{DBLP:conf/interspeech/OncescuKHAA21}
Andreea{-}Maria Oncescu, A.~Sophia Koepke, Jo{\~{a}}o~F. Henriques, Zeynep
  Akata, and Samuel Albanie. 2021.
\newblock \href {https://doi.org/10.21437/Interspeech.2021-2227} {Audio
  retrieval with natural language queries}.
\newblock In \emph{Interspeech 2021, 22nd Annual Conference of the
  International Speech Communication Association, Brno, Czechia, 30 August - 3
  September 2021}, pages 2411--2415. {ISCA}.

\bibitem[{Park et~al.(2022)Park, Shen, Farhadi, Darrell, Choi, and
  Rohrbach}]{park-etal-2022-exposing}
Jae~Sung Park, Sheng Shen, Ali Farhadi, Trevor Darrell, Yejin Choi, and Anna
  Rohrbach. 2022.
\newblock \href {https://doi.org/10.18653/v1/2022.naacl-main.261} {Exposing the
  limits of video-text models through contrast sets}.
\newblock In \emph{Proceedings of the 2022 Conference of the North American
  Chapter of the Association for Computational Linguistics: Human Language
  Technologies}, pages 3574--3586, Seattle, United States. Association for
  Computational Linguistics.

\bibitem[{Plummer et~al.(2017)Plummer, Wang, Cervantes, Caicedo, Hockenmaier,
  and Lazebnik}]{DBLP:journals/ijcv/PlummerWCCHL17}
Bryan~A. Plummer, Liwei Wang, Chris~M. Cervantes, Juan~C. Caicedo, Julia
  Hockenmaier, and Svetlana Lazebnik. 2017.
\newblock \href {https://doi.org/10.1007/s11263-016-0965-7} {Flickr30k
  entities: Collecting region-to-phrase correspondences for richer
  image-to-sentence models}.
\newblock \emph{Int. J. Comput. Vis.}, 123(1):74--93.

\bibitem[{Radford et~al.(2021)Radford, Kim, Hallacy, Ramesh, Goh, Agarwal,
  Sastry, Askell, Mishkin, Clark, Krueger, and
  Sutskever}]{DBLP:conf/icml/RadfordKHRGASAM21}
Alec Radford, Jong~Wook Kim, Chris Hallacy, Aditya Ramesh, Gabriel Goh,
  Sandhini Agarwal, Girish Sastry, Amanda Askell, Pamela Mishkin, Jack Clark,
  Gretchen Krueger, and Ilya Sutskever. 2021.
\newblock \href {http://proceedings.mlr.press/v139/radford21a.html} {Learning
  transferable visual models from natural language supervision}.
\newblock In \emph{Proceedings of the 38th International Conference on Machine
  Learning, {ICML} 2021, 18-24 July 2021, Virtual Event}, volume 139 of
  \emph{Proceedings of Machine Learning Research}, pages 8748--8763. {PMLR}.

\bibitem[{Radovanovic et~al.(2010)Radovanovic, Nanopoulos, and
  Ivanovi}]{JMLR:v11:radovanovic10a}
Milos Radovanovic, Alexandros Nanopoulos, and Mirjana Ivanovi. 2010.
\newblock \href {http://jmlr.org/papers/v11/radovanovic10a.html} {Hubs in
  space: Popular nearest neighbors in high-dimensional data}.
\newblock \emph{Journal of Machine Learning Research}, 11(86):2487--2531.

\bibitem[{Singh et~al.(2022)Singh, Hu, Goswami, Couairon, Galuba, Rohrbach, and
  Kiela}]{DBLP:conf/cvpr/SinghHGCGRK22}
Amanpreet Singh, Ronghang Hu, Vedanuj Goswami, Guillaume Couairon, Wojciech
  Galuba, Marcus Rohrbach, and Douwe Kiela. 2022.
\newblock \href {https://doi.org/10.1109/CVPR52688.2022.01519} {{FLAVA:} {A}
  foundational language and vision alignment model}.
\newblock In \emph{{IEEE/CVF} Conference on Computer Vision and Pattern
  Recognition, {CVPR} 2022, New Orleans, LA, USA, June 18-24, 2022}, pages
  15617--15629. {IEEE}.

\bibitem[{Su et~al.(2019)Su, Zhong, and Zhang}]{DBLP:conf/iccv/SuZZ19}
Shupeng Su, Zhisheng Zhong, and Chao Zhang. 2019.
\newblock \href {https://doi.org/10.1109/ICCV.2019.00312} {Deep joint-semantics
  reconstructing hashing for large-scale unsupervised cross-modal retrieval}.
\newblock In \emph{2019 {IEEE/CVF} International Conference on Computer Vision,
  {ICCV} 2019, Seoul, Korea (South), October 27 - November 2, 2019}, pages
  3027--3035. {IEEE}.

\bibitem[{Velickovic et~al.(2018)Velickovic, Cucurull, Casanova, Romero, Lio,
  and Bengio}]{velickovic2018graph}
Petar Velickovic, Guillem Cucurull, Arantxa Casanova, Adriana Romero, Pietro
  Lio, and Yoshua Bengio. 2018.
\newblock \href {https://openreview.net/forum?id=rJXMpikCZ} {Graph attention
  networks}.
\newblock In \emph{International Conference on Learning Representations}.

\bibitem[{Wang et~al.(2022{\natexlab{a}})Wang, Xu, He, Li, Ji, Han, and
  Ding}]{DBLP:conf/mm/WangXHLJHD22}
Haoran Wang, Di~Xu, Dongliang He, Fu~Li, Zhong Ji, Jungong Han, and Errui Ding.
  2022{\natexlab{a}}.
\newblock \href {https://doi.org/10.1145/3503161.3548010} {Boosting video-text
  retrieval with explicit high-level semantics}.
\newblock In \emph{{MM} '22: The 30th {ACM} International Conference on
  Multimedia, Lisboa, Portugal, October 10 - 14, 2022}, pages 4887--4898.
  {ACM}.

\bibitem[{Wang et~al.(2022{\natexlab{b}})Wang, Zhu, Zheng, Xu, and
  Yang}]{9878037}
Xiaohan Wang, Linchao Zhu, Zhedong Zheng, Mingliang Xu, and Yi~Yang.
  2022{\natexlab{b}}.
\newblock \href {https://doi.org/10.1109/TMM.2022.3204444} {Align and tell:
  Boosting text-video retrieval with local alignment and fine-grained
  supervision}.
\newblock \emph{IEEE Transactions on Multimedia}, pages 1--11.

\bibitem[{Wang et~al.(2023)Wang, Jian, and Xue}]{wang2023balance}
Yimu Wang, Xiangru Jian, and Bo~Xue. 2023.
\newblock \href {http://arxiv.org/abs/2310.11612} {Balance act: Mitigating
  hubness in cross-modal retrieval with query and gallery banks}.

\bibitem[{Wang et~al.(2020{\natexlab{a}})Wang, Lu, and
  Zhang}]{10.1145/3394171.3413882}
Yimu Wang, Shiyin Lu, and Lijun Zhang. 2020{\natexlab{a}}.
\newblock Searching privately by imperceptible lying: A novel private hashing
  method with differential privacy.
\newblock In \emph{Proceedings of the 28th ACM International Conference on
  Multimedia}, page 2700–2709.

\bibitem[{Wang and Shi(2023)}]{wang_video-text_2023}
Yimu Wang and Peng Shi. 2023.
\newblock \href {https://doi.org/10.48550/arXiv.2302.09473} {Video-{Text}
  {Retrieval} by {Supervised} {Multi}-{Space} {Multi}-{Grained} {Alignment}}.
\newblock ArXiv:2302.09473 [cs].

\bibitem[{Wang et~al.(2020{\natexlab{b}})Wang, Wei, Xue, and
  Zhang}]{DBLP:conf/prcv/WangWXZ20}
Yimu Wang, Xiu{-}Shen Wei, Bo~Xue, and Lijun Zhang. 2020{\natexlab{b}}.
\newblock Piecewise hashing: {A} deep hashing method for large-scale
  fine-grained search.
\newblock In \emph{Pattern Recognition and Computer Vision - Third Chinese
  Conference, {PRCV} 2020, Nanjing, China, October 16-18, 2020, Proceedings,
  Part {II}}, pages 432--444.

\bibitem[{Wang et~al.(2021)Wang, Xue, Cheng, Chen, and Zhang}]{ijcai2021p156}
Yimu Wang, Bo~Xue, Quan Cheng, Yuhui Chen, and Lijun Zhang. 2021.
\newblock Deep unified cross-modality hashing by pairwise data alignment.
\newblock In \emph{Proceedings of the Thirtieth International Joint Conference
  on Artificial Intelligence, {IJCAI-21}}, pages 1129--1135.

\bibitem[{Xu et~al.(2016)Xu, Mei, Yao, and Rui}]{DBLP:conf/cvpr/XuMYR16}
Jun Xu, Tao Mei, Ting Yao, and Yong Rui. 2016.
\newblock \href {https://doi.org/10.1109/CVPR.2016.571} {{MSR-VTT:} {A} large
  video description dataset for bridging video and language}.
\newblock In \emph{2016 {IEEE} Conference on Computer Vision and Pattern
  Recognition, {CVPR} 2016, Las Vegas, NV, USA, June 27-30, 2016}, pages
  5288--5296. {IEEE} Computer Society.

\bibitem[{Yang et~al.(2022)Yang, Huang, Hu, Li, Lv, and Peng}]{Tang2022cvpr}
Mouxing Yang, Zhenyu Huang, Peng Hu, Taihao Li, Jiancheng Lv, and Xi~Peng.
  2022.
\newblock \href {https://doi.org/10.1109/CVPR52688.2022.01391} {Learning with
  twin noisy labels for visible-infrared person re-identification}.
\newblock In \emph{2022 IEEE/CVF Conference on Computer Vision and Pattern
  Recognition (CVPR)}, pages 14288--14297.

\bibitem[{Yu et~al.(2023)Yu, Liu, Wang, Xu, and Liu}]{yu2023multimodal}
Qiying Yu, Yang Liu, Yimu Wang, Ke~Xu, and Jingjing Liu. 2023.
\newblock \href {https://openreview.net/forum?id=Hnk1WRMAYqg} {Multimodal
  federated learning via contrastive representation ensemble}.
\newblock In \emph{The Eleventh International Conference on Learning
  Representations}.

\bibitem[{Yu et~al.(2022)Yu, Song, Kim, Lee, Ryu, and Yoon}]{yu-etal-2022-rare}
Sangwon Yu, Jongyoon Song, Heeseung Kim, Seongmin Lee, Woo-Jong Ryu, and
  Sungroh Yoon. 2022.
\newblock \href {https://doi.org/10.18653/v1/2022.acl-long.3} {Rare tokens
  degenerate all tokens: Improving neural text generation via adaptive gradient
  gating for rare token embeddings}.
\newblock In \emph{Proceedings of the 60th Annual Meeting of the Association
  for Computational Linguistics (Volume 1: Long Papers)}, pages 29--45, Dublin,
  Ireland. Association for Computational Linguistics.

\bibitem[{Zhang et~al.(2020)Zhang, Gao, Xu, Miao, Yang, and
  Shao}]{zhang-etal-2020-revisiting}
Zhong Zhang, Chongming Gao, Cong Xu, Rui Miao, Qinli Yang, and Junming Shao.
  2020.
\newblock \href {https://doi.org/10.18653/v1/2020.findings-emnlp.46}
  {Revisiting representation degeneration problem in language modeling}.
\newblock In \emph{Findings of the Association for Computational Linguistics:
  EMNLP 2020}, pages 518--527, Online. Association for Computational
  Linguistics.

\bibitem[{Zhao et~al.(2022)Zhao, Zhu, Wang, and Yang}]{10.1145/3477495.3531950}
Shuai Zhao, Linchao Zhu, Xiaohan Wang, and Yi~Yang. 2022.
\newblock \href {https://doi.org/10.1145/3477495.3531950} {Centerclip: Token
  clustering for efficient text-video retrieval}.
\newblock In \emph{Proceedings of the 45th International ACM SIGIR Conference
  on Research and Development in Information Retrieval}, SIGIR '22, page
  970–981, New York, NY, USA. Association for Computing Machinery.

\bibitem[{Zhong et~al.(2017)Zhong, Zheng, Cao, and
  Li}]{DBLP:conf/cvpr/ZhongZCL17}
Zhun Zhong, Liang Zheng, Donglin Cao, and Shaozi Li. 2017.
\newblock \href {https://doi.org/10.1109/CVPR.2017.389} {Re-ranking person
  re-identification with k-reciprocal encoding}.
\newblock In \emph{2017 {IEEE} Conference on Computer Vision and Pattern
  Recognition, {CVPR} 2017, Honolulu, HI, USA, July 21-26, 2017}, pages
  3652--3661. {IEEE} Computer Society.

\end{thebibliography}
\bibliographystyle{acl_natbib}
}
\newpage
\clearpage
\appendix
\useunder{\uline}{\ul}{}

\section{Related Work} \label{sec: relatedwork}

We review prior work in cross-modal retrieval and representation degeneration, which are the two most related areas to our work.

\textbf{Cross-modal Retrival.} 
The goal of cross-modal retrieval is to learn a common representation space, where the similarity between samples from different modalities can be directly measured. 
Recently, inspired by the success of deep learning~\cite{DBLP:conf/naacl/DevlinCLT19, DBLP:conf/cvpr/HeZRS16}, numerous methods based on deep neural networks have been proposed for image-text retrieval~\cite{DBLP:conf/icml/RadfordKHRGASAM21}, video-text retrieval~\cite{DBLP:journals/ijon/LuoJZCLDL22}, and audio-text retrieval~\cite{DBLP:conf/interspeech/OncescuKHAA21}. 
Further, to learn a better representation space, vision-language pretraining~\cite{DBLP:conf/nips/Gan0LZ0020, DBLP:conf/eccv/Li0LZHZWH0WCG20, DBLP:conf/cvpr/SinghHGCGRK22} on large-scale unlabeled cross-modal data has been widely employed and have shown promising performance. 
Motivated by this, recent works have attempted to pretrain or fine-tune cross-modal retrieval models, \eg, image-text retrieval~\cite{DBLP:conf/icml/RadfordKHRGASAM21,DBLP:conf/eccv/Li0LZHZWH0WCG20}, video-text retrieval~\cite{DBLP:conf/cvpr/ChenZJW20,DBLP:journals/corr/abs-2109-04290,DBLP:journals/corr/abs-2111-05610,DBLP:conf/cvpr/GortiVMGVGY22,DBLP:conf/cvpr/LeiLZGBB021,DBLP:conf/mm/MaXSYZJ22,park-etal-2022-exposing,DBLP:conf/mm/WangXHLJHD22,9878037,10.1145/3477495.3531950}, and audio-text retrieval~\cite{DBLP:conf/interspeech/OncescuKHAA21} in an end-to-end manner. 

In contrast to the methods that focus on improving the representation learning ability, another line of cross-modal retrieval research has focused on improving the effectiveness of retrieval, including k-d trees~\cite{DBLP:books/degruyter/Bellman15}, re-ranking~\cite{DBLP:conf/cvpr/ZhongZCL17,DBLP:conf/cvpr/MiechALSZ21}, query expansion~\cite{DBLP:conf/cvpr/ChenZJW20}, vector compression schemes based on binary codes~\cite{DBLP:conf/iccv/SuZZ19,DBLP:conf/iccv/LiongLT017} and quantization~\cite{DBLP:journals/pami/GongLGP13} that help address the curse of dimensionality~\cite{Keogh2017}. 
However, a recent study~\cite{liang2022mind} shows that the \problem\ has significantly affected the performance of multi-modal learning. 
To investigate the influence of \problem\ in the cross-modal retrieval, 

we show that \problem\ widely exists in different datasets and models.

\textbf{Representation degeneration}. 
The representation degeneration problem was first introduced in the natural language processing (NLP) area~\cite{DBLP:conf/iclr/GaoHTQWL19}. 
It was found that when training a model for natural language generation tasks through likelihood maximization with the weight-tying trick, especially with large training datasets, many of the learned word embeddings tend to degenerate and be distributed into a narrow cone. 
This limitation largely reduces the representation power of word embeddings. 
The representation degeneration problem leads to an increase in the overall similarity between token embeddings, which has a negative effect on the performance of the models.
It was noted that Laplacian regularization can address this problem better than cosine regularization through theoretical proof and empirical experiments \cite{zhang-etal-2020-revisiting}. Subsequent work highlighted \cite{yu-etal-2022-rare} that the training dynamics of the token embeddings focus on rare token embedding which leads to the degeneration problem for all tokens. 
To this end, they use adaptive gradient gating which gates the specific part of the gradient for rare token embeddings and thus better alleviates the data degeneration problem.

Though representation degeneration has been explored in NLP, it remains unexplored in multi-modal learning for a long time. 
A recent work~\cite{liang2022mind} shows that the representation generated by a common deep neural network is restricted to a narrow cone and consequently, with two modality-dependent encoders, the representations from the two modalities are clearly apart during the whole training procedure. 
Further, they also show that varying the modality gap distance has a significant impact on improving the model’s downstream zero-shot classification performance and fairness.

To step forward towards better representation learning in cross-modal retrieval, different from the previous methods in NLP which focus on addressing this problem in the training procedure, we propose a novel method, namely \ours, which proposes to avoid representation degeneration in a post-processing manner. 

Inspired by the representation aggregation induced by graph convolution~\cite{Keogh2017,baranwal2023effects}, we utilize the graph convolution in an opposite way to separate the data points that share similar representation.

As the first method in solving the representation degeneration problem in cross-modal retrieval, \ours\ does not require retraining the model or any other time-consuming operation. \ours\ achieves better retrieval performance with a larger margin between different representations compared to the baselines and does this faster.

\section{Elaborations on Methodologies}

\subsection{Graph Convolution}\label{sec:graphconv}

Graph convolution is a mathematical operation that transforms the features of nodes in a graph based on their local neighborhoods. The objective is to learn a function of signals/features on a graph, which takes into account the graph structure and the node features. It can be regarded as a generalization of convolutions to non-Euclidean data~\cite{bruna2014spectral}. The operation is first introduced and popularized in the work of Graph Convolution Networks~\cite{kipf2017semisupervised}, which is considered one of the most seminal papers in the area of graph learning.

The main idea behind a graph convolution operation is to generate a new representation for each node that captures the local neighborhood information around it. This is usually achieved by aggregating feature information from a node's direct neighbors, sometimes including the node itself. Formally, it can be defined as a special case of a Message Passing Network (MPNN), in which vector messages are exchanged between nodes and updated using neural networks ~\cite{gilmer2017mpnn}. The basic operation of MPNN can be expressed as ~\cite{willgnnbook}

$$
\begin{aligned}
\mathbf{x}_i^{(k+1)}  =&\operatorname{U}^{(k)}\left(\mathbf{x}_i^{(k)},\operatorname{AG}^{(k)}\left(\left\{\mathbf{x}_j^{(k)}, \forall j \in \mathcal{N}(i)\right\}\right)\right) \\
 =&\operatorname{U}^{(k)}\left(\mathbf{x}_i^{(k)}, \mathbf{m}_{\mathcal{N}(i)}^{(k)}\right)\,,
\end{aligned}
$$

where $\operatorname{U}$ (short for update) and $\operatorname{AG}$ (short for aggregation) are arbitrary differentiable functions (i.e., neural networks) and $\mathbf{x}_i^{(k)}$ is the embedding(representation) of node $i$ at $k$-th iteration. $\mathbf{m}_{\mathcal{N}(i)}$ is the "message" that is aggregated from $i$ 's graph neighborhood $\mathcal{N}(i)$. 

In this study, we adopt a simple message passing operator that performs non-trainable propagation since we want to propose a post-processing method with any training. The adopt operator is actually the backbone of multiple GNN studies, which can be expressed as~\cite{simple-conv},
$$
\mathbf{x}_i^{\prime}=\bigoplus_{j \in \mathcal{N}(i)} e_{j i} \cdot \mathbf{x}_j\,,
$$
where $\bigoplus$ defines a custom aggregation scheme. $\mathbf{x}_i^{\prime}$ is updated representation of node $i$ and $e_{ji}$ is edge weight between node $i$ and $j$.

For the sake of simplicity and explainability, we concretize the above operation only with simple addition and self-loop, as follows,
\begin{equation} \label{eq: conv}
    \mathbf{x_i}^\prime = \mathbf{x_i} + \sum_{j} A_{ij} \mathbf{x_j}\,.
\end{equation}
Note that the only difference with \Cref{eq: conv} and inverse convolution \Cref{eq:invconv} we apply in the study is that addiction is replaced with subtraction, leading to the name 'inverse'.

\subsection{Average Pooling in CNN} \label{sec:averagepool}

Average pooling is one type of pooling layers that conducts dimensionality reduction, reducing the number of parameters in the input~\cite{Goodfellow-et-al-2016}. Similar to the convolutional layer, the pooling operation sweeps a filter across the entire input, but the difference is that this filter does not have any weights. Instead, the kernel applies an aggregation function to the values within the receptive field. Specifically, the adopted average pooling calculates the average value within the receptive field of the filter as it moves across the input. Here in this study, the receptive field is subject to the size of the neighborhood of each data point. 

\subsection{Distribution of Data Representation} \label{sec: distribution}

While we keep using the discrete gallery and query set like in \Cref{eq:invconv}, they can be regarded as sampled results from the hidden continuous distribution dependent on the intrinsic properties of the corresponding dataset and the representation learning method applied. 
Therefore, ideally, two gallery sets sampled from the same dataset will probably have different $\Delta_{deg}$ scores even with the same representation learning method due to the variance introduced by the sampling process. 
That is to say, performing the same inverse convolution operation on the two sampled gallery sets might have quite different effects, especially when the sample size is small. 
Also, since the proposed method in this study is a post-processing approach without any training, we need to control the magnitude of the convolution with the help of some hyperparameters. The reason for doing this is to control the change in the representation of the gallery data that has already been aligned with the embedding space of query data by the representation learning model. 
Given the sampled gallery set is small, this means a very large variance in the value of the best hyperparameters as well. 

Unfortunately, it is quite common in practice that we only have a small sampled gallery set. 
Usually, when cross-modal retrieval is carried out, we constantly cut down the size of the gallery set with the help of some pre-ranking or indexing techniques. 
This process can be somehow regarded as sampling a set from the distribution that is empirically represented by the whole gallery set. 
Also, during the evaluation of any method on various datasets, the size of the test or evaluation gallery set is typically much smaller compared to the training set.
 
Both cases make the result of the proposed methods subjected to potentially large variance.

However, it would be more promising that our method is generally stable and robust to any size of the gallery set. 
The ideal case would be that we can perform the inverse convolution similar to \Cref{eq:invconvwithr} but based on the continuous distribution, where $\operatorname{P}(\mathbf{x_j})$ is data point $\mathbf{x_j}$ to be sampled from the hidden distribution, as follows,
\begin{equation}\label{eq:invconvideal}
\mathbf{x_i}^\prime =  \mathbf{x_i} - r\int_{\mathbf{x_j}} S_{ij} \operatorname{P}(\mathbf{x_j})\mathbf{x_j}\,.
\end{equation}

However, it is impossible to have exact access to this hidden distribution in practice. 
The best approximation is the training (or validation) gallery set $\hat{G}$ since it is the largest one we can obtain. 
Therefore, we can perform the inverse convolution on $\hat{G}$. 
Note that the distribution of the query set $\hat{Q}$ should theoretically be similar to that of $\hat{G}$ as this a basic assumption in machine learning~\cite{DBLP:conf/cvpr/BogolinCJLA22}. 
Therefore, it is possible to combine the (train or validation) gallery set $\hat{G}$ and the (train or validation)  query set $\hat{Q}$ to be the even better estimation of the hidden distribution. Thus, we go on to refine \ours\ as in \Cref{eq:invconvpttp},

In general, with reasonable and general assumptions, we strike the importance of utilizing the data of both the modality from the training set when we want to capture a more accurate and stable distribution of data representation. The idea is not bound to the proposed methods and can be adopted by any future work on the post-processing of cross-modal retrieval tasks. 

\section{Experiments}

\subsection{Datasets Details} \label{sec:dateapp}
The experiments are conducted on eight cross-modal benchmarks, which include four video-text retrieval benchmarks (MSR-VTT~\cite{DBLP:conf/cvpr/XuMYR16}, MSVD~\cite{chen-dolan-2011-collecting}, ActivityNet~\cite{caba2015activitynet}, and DiDemo~\cite{hendricks_localizing_2017}), two image-text retrieval benchmarks (MSCOCO~\cite{DBLP:conf/eccv/LinMBHPRDZ14} and Flickr30k~\cite{DBLP:journals/ijcv/PlummerWCCHL17}), as well as two audio-text retrieval benchmarks (AudioCaps~\cite{kim-etal-2019-audiocaps} and CLOTHO~\cite{drossos_clotho_2020}). The details of the datasets are presented below:
\begin{itemize}
    \item \textbf{MSR-VTT}~\cite{DBLP:conf/cvpr/XuMYR16}: Comprises approximately 10k videos, each accompanied by 20 captions. For text-video retrieval, we follow the protocol set by previous works~\cite{DBLP:conf/bmvc/LiuANZ19,DBLP:conf/iccv/CroitoruBLJZAL21,DBLP:journals/ijon/LuoJZCLDL22,DBLP:conf/mm/MaXSYZJ22,park-etal-2022-exposing}, using both the official (full) split and the 1k-A split. The full split includes 2,990 videos for testing and 497 for validation, whereas the 1k-A split has 1,000 videos for testing and around 9,000 for training.

    \item \textbf{MSVD}~\cite{chen-dolan-2011-collecting}: Contains 1,970 videos and about 80k captions. The standard split used in prior works~\cite{DBLP:conf/bmvc/LiuANZ19,DBLP:conf/iccv/CroitoruBLJZAL21,DBLP:journals/ijon/LuoJZCLDL22,park-etal-2022-exposing} is adopted for reporting results, which includes 1,200 videos for training, 100 for validation, and 670 for testing.

    \item \textbf{ActivityNet}~\cite{caba2015activitynet}: Contains 20k videos and approximately 100K descriptive sentences. These videos are extracted from YouTube. We employ a paragraph video retrieval setup as defined in prior works~\cite{DBLP:conf/bmvc/LiuANZ19,DBLP:conf/iccv/CroitoruBLJZAL21,DBLP:journals/ijon/LuoJZCLDL22,park-etal-2022-exposing}. We report results on the val1 split. The training split includes 10,009 videos, with 4,917 videos allocated for testing.

    \item \textbf{DiDemo}~\cite{hendricks_localizing_2017}: Includes over 10,000 personal videos, each lasting between 25-30 seconds, along with over 40,000 localized text descriptions. The videos are divided into training (8,395), validation (1,065), and testing (1,004) sets.

    \item \textbf{MSCOCO}~\cite{DBLP:conf/eccv/LinMBHPRDZ14}: Consists of 123k images, each accompanied by 5 captions. The 5k split is used for evaluation.

    \item \textbf{Flickr30k}~\cite{DBLP:journals/ijcv/PlummerWCCHL17}: This dataset contains 31,000 images collected from Flickr, each accompanied by 5 reference sentences provided by human annotators.

    \item \textbf{AudioCaps}~\cite{kim-etal-2019-audiocaps}: Comprises sound samples with event descriptions. We adopt the same setup as prior work~\cite{koepke_audio_2022} where 49,291 samples are used for training, 428 for validation, and 816 for testing.

    \item \textbf{CLOTHO}~\cite{drossos_clotho_2020}: Comprises of 4,981 audio samples of 15 to 30 seconds in duration and 24,905 captions of eight to 20 words in length (five captions for each audio sample).
\end{itemize}

\subsection{Experiment Details} \label{sec:exp}
The public codes and weights of all the tasks in this study are summarized in \Cref{tab: methods links}. Note that there is no available trained model for \href{https://github.com/openai/CLIP}{CLIP} and \href{https://github.com/xuguohai/X-CLIP}{X-CLIP}. Therefore, we train both models on a single A100 GPU with the hyperparameters recommended by the original studies.

\begin{table}[]
\centering
\caption{Implmentation details.}
\label{tab: methods links}
\resizebox{\columnwidth}{!}{%
\begin{tabular}{l|l|l}
\toprule
Method     & Public Code                                                    & Public Weights                                                                                         \\\midrule
CE+        & \href{https://github.com/albanie/collaborative-experts}{Link}  & \href{http://www.robots.ox.ac.uk/~vgg/research/teachtext/data-hq/models/msrvtt-train-gpt2-xl-finetuned-adam/244af891/seed-0/2020-10-01_12-22-00/trained_model.pth}{MSR-VTT}, \href{http://www.robots.ox.ac.uk/~vgg/research/teachtext/data-hq/models/msvd-train-gpt2-xl-finetuned-adam/db396303/seed-0/2020-10-01_13-17-33/trained_model.pth}{MSVD}, \href{http://www.robots.ox.ac.uk/~vgg/research/teachtext/data-hq/models/didemo-train-gpt2-xl-finetuned-adam/616cf11b/seed-0/2020-10-01_13-31-57/trained_model.pth}{DiDeMo}, and \href{http://www.robots.ox.ac.uk/~vgg/research/teachtext/data-hq/models/activity-net-train-gpt2-xl-finetuned-adam/a791f27d/seed-0/2020-10-01_13-42-29/trained_model.pth}{ActivityNet}   \\
TT-CE+     & \href{https://github.com/albanie/collaborative-experts}{Link}  & \href{http://www.robots.ox.ac.uk/~vgg/research/teachtext/data-hq/models/msrvtt-train-gpt2-xl-finetuned-mte-adam/6427fd41/seed-0/2020-09-30_20-34-12/trained_model.pth}{MSR-VTT}, \href{http://www.robots.ox.ac.uk/~vgg/research/teachtext/data-hq/models/msvd-train-gpt2-xl-finetuned-mte-adam/0af2a1ed/seed-0/2020-09-30_21-30-15/trained_model.pth}{MSVD}, \href{http://www.robots.ox.ac.uk/~vgg/research/teachtext/data-hq/models/didemo-train-ce-intra-mte/1a5a249f/seed-0/2020-11-06_19-12-39/trained_model.pth}{DiDeMo}, and \href{http://www.robots.ox.ac.uk/~vgg/research/teachtext/data-hq/models/activity-net-train-gpt2-xl-finetuned-mte-adam/87d04a50/seed-0/2020-10-01_08-48-36/trained_model.pth}{ActivityNet} \\
CLIP4Clip  & \href{https://github.com/ArrowLuo/CLIP4Clip}{Link}             &N/A\\
CLIP2Video & \href{https://github.com/CryhanFang/CLIP2Video}{Link}          & \href{https://drive.google.com/drive/folders/1a5Dcg8wNh88Z-bxb0ZMV3IJFjtSe7X2A?usp=sharing}{MSR-VTT} and \href{https://drive.google.com/drive/folders/1LKMUZFf9EAxFbGShlA22eUCeGKC8DWx4?usp=sharing}{MSVD}  \\
X-CLIP     & \href{https://github.com/xuguohai/X-CLIP}{Link}                & N/A \\
CLIP       & \href{https://github.com/openai/CLIP}{Link}                    & \href{https://github.com/openai/CLIP}{CLIP}  \\
Oscar      & \href{https://github.com/microsoft/Oscar}{Link}                & \href{https://github.com/microsoft/Oscar}{MSCOCO and Flickr30k}\\
AR-CE      & \href{https://github.com/oncescuandreea/audio-retrieval}{Link} & \href{http://www.robots.ox.ac.uk/~vgg/research/collaborative-experts/data/models/audiocaps-train-vggish-vggsound/7e2eda12/seed-0/2021-06-09_17-06-26/trained_model.pth}{AudioCaps} and \href{http://www.robots.ox.ac.uk/~vgg/research/collaborative-experts/data/models/clotho-train-vggish-vggsound/dec0c820/seed-0/2021-06-10_14-45-51/trained_model.pth}{CLOTHO}\\\bottomrule  
\end{tabular}%
}
\end{table}

\subsection{Prevalence of Representation Degeneration Problem across Datasets and Methods} \label{sec: prevalence of prob}
To show that the \problem\ prevails in all datasets, methods, and tasks, we perform the same analysis in \Cref{fig: datadeg a}. We uniformly sample a subset of the gallery set of both retrieval tasks (\ie, text to other modality or other modality to text, other modalities can be video, image, or audio depending on the dataset), and perform PCA upon it to reduce the dimension of the representations down to 2, which is the first two principal dimensions. The results are presented in \Cref{fig: prevalence of problem}. Note that $\Delta_{deg}$ (\ie, the degree of \problem) included in each figure is the one for the complete gallery set instead of the sampled set used in the figure.

For this qualitative but very intuitive study,  We firmly validate again that almost all the data representations gathered in a very narrow cone in the embedding space for basically all the datasets, methods, and tasks. Also, though subject to the difference between datasets and methods, we can witness that a more convex-shaped distribution usually generally leads to a larger $\Delta_{deg}$.

The results imply the universality of the \problem. More quantitative results can be found in \textbf{RQ1} in \Cref{sec: ablation} and in the \textbf{Continuation on RQ1} section (\Cref{sec:more_ablt}).

\begin{table*}[ht!]
\centering
\caption{Retrieval performance on MSCOCO (5k split). Best in \textbf{Bold} and the second best is \underline{underlined}.  }
\label{tab: quan: coco}
\resizebox{\textwidth}{!}{%
\begin{tabular}{ll|ccccc|ccccc}
\toprule
                       & \multirow{2}{*}{Normalization}      & \multicolumn{5}{c|}{Text-to-Image Retrieval}                            & \multicolumn{5}{c}{Image-to-Text Retrieval}                           \\
                        &                                & R@1 $\uparrow$ & R@5 $\uparrow$ & R@10 $\uparrow$ & MdR $\downarrow$ & MnR  $\downarrow$ & R@1 $\uparrow$ & R@5  $\uparrow$ & R@10 $\uparrow$ & MdR  $\downarrow$ & MnR $\downarrow$ \\ \midrule
CLIP  &                     & 30.34          & 54.74          & 66.08          & 4.0 & 25.39          & 50.04          & 74.80          & {\ul 83.38}    & \textbf{1.0} & 9.22          \\
 & +\pool (ratio=0.1)  & 30.37          & 54.77          & 66.14          & 4.0 & {\ul 25.36}    & 49.98          & 75.08          & 83.34          & {\ul 2.0}    & 9.20          \\
 & +\pool (ratio=0.25) & 30.37          & 54.77          & 66.14          & 4.0 & {\ul 25.36}    & 49.98          & 75.06          & 83.30          & {\ul 2.0}    & 9.20          \\
 & +\pool (ratio=0.5)  & 30.38          & 54.77          & 66.10          & 4.0 & 25.38          & 49.98          & 75.10          & 83.34          & {\ul 2.0}    & 9.20          \\
 & +\pool (ratio=1)    & 30.39          & 54.77          & 66.11          & 4.0 & 25.38          & 49.98          & 75.06          & 83.36          & {\ul 2.0}    & 9.20          \\
 \rowcolor{red!10}& +\ours              & {\ul 32.70}    & {\ul 57.53}    & \textbf{68.24} & 4.0 & \textbf{24.35} & {\ul 51.04}    & {\ul 75.18}    & 83.24          & \textbf{1.0} & {\ul 8.93}    \\
 \rowcolor{red!10}& +\fcut          & \textbf{33.11} & \textbf{57.49} & {\ul 68.19}    & 4.0 & 28.95          & \textbf{52.26} & \textbf{76.42} & \textbf{84.32} & \textbf{1.0} & \textbf{8.83}\\ \midrule
Oscar  &                     & 52.50          & 80.03          & \textbf{87.96} & 1.0 & {\ul 10.68}    & 66.74          & {\ul 89.98}    & 94.98          & 1.0 & 2.95          \\
 & +\pool (ratio=0.1)  & 52.52          & {\ul 80.04}    & {\ul 87.95}    & 1.0 & 10.70          & 66.98          & {\ul 89.98}    & 95.00          & 1.0 & 2.96          \\
 & +\pool (ratio=0.25) & 52.52          & 80.03          & \textbf{87.96} & 1.0 & {\ul 10.68}    & 66.98          & 89.96          & 94.96          & 1.0 & 2.95          \\
 & +\pool (ratio=0.5)  & 52.51          & 80.00          & \textbf{87.96} & 1.0 & \textbf{10.67} & 66.94          & 89.92          & 94.94          & 1.0 & 2.95          \\
 & +\pool (ratio=1)    & 52.50          & 80.02          & \textbf{87.96} & 1.0 & {\ul 10.68}    & 66.70          & 89.90          & 95.00          & 1.0 & 2.96          \\
 \rowcolor{red!10}& +\ours              & {\ul 52.63}    & \textbf{80.05} & \textbf{87.96} & 1.0 & 10.72          & \textbf{67.90} & 89.96          & \textbf{95.22} & 1.0 & \textbf{2.92} \\
 \rowcolor{red!10}& +\fcut          & \textbf{52.93} & \textbf{80.05} & 87.78          & 1.0 & 11.09          & {\ul 67.68}    & \textbf{90.24} & {\ul 95.20}    & 1.0 & {\ul 2.94}   \\ 
\bottomrule
\end{tabular}%
}
\end{table*}

\begin{figure*}[h!]
    \begin{center}
    \subfloat[CLOTHO (AR-CE, T2A).]{
        \centering
        \includegraphics[width=0.24\textwidth]{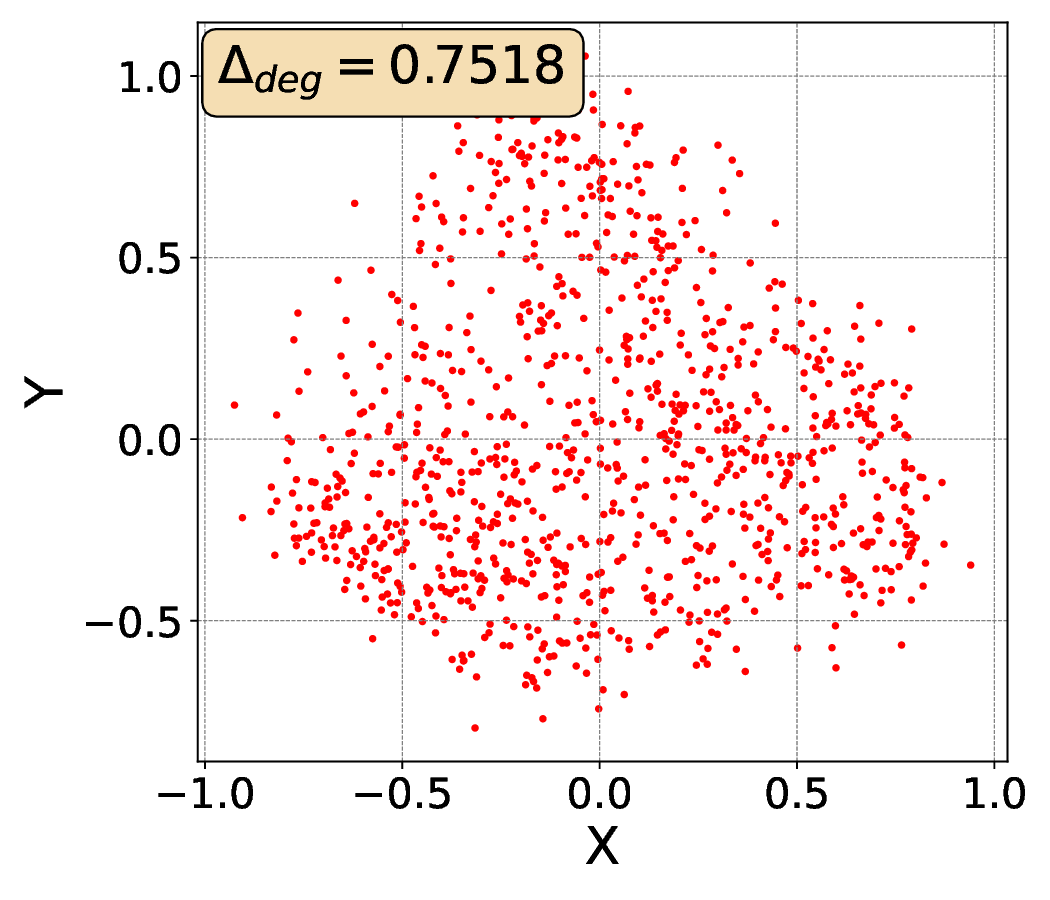}
    }
    \subfloat[CLOTHO (AR-CE, A2T).]{
        \centering
        \includegraphics[width=0.24\textwidth]{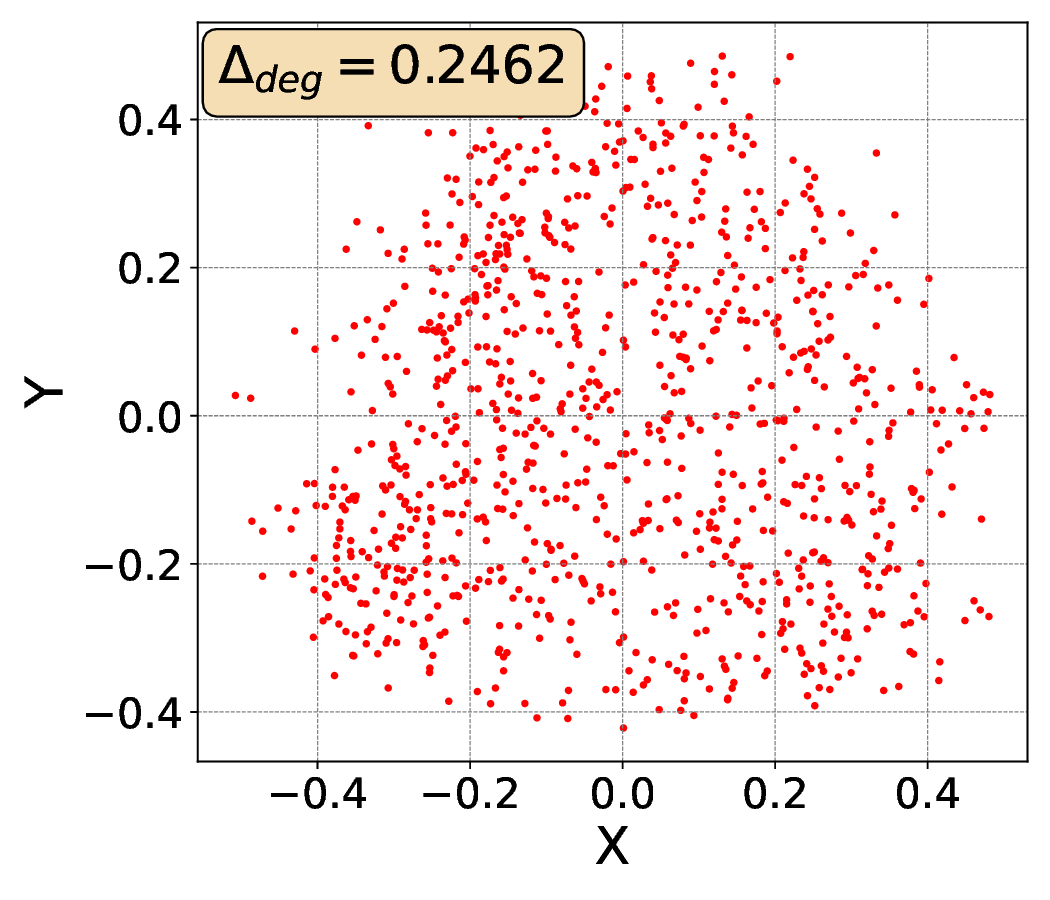}
    }
    \subfloat[MSCOCO (CLIP, T2I).]{
        \centering
        \includegraphics[width=0.24\textwidth]{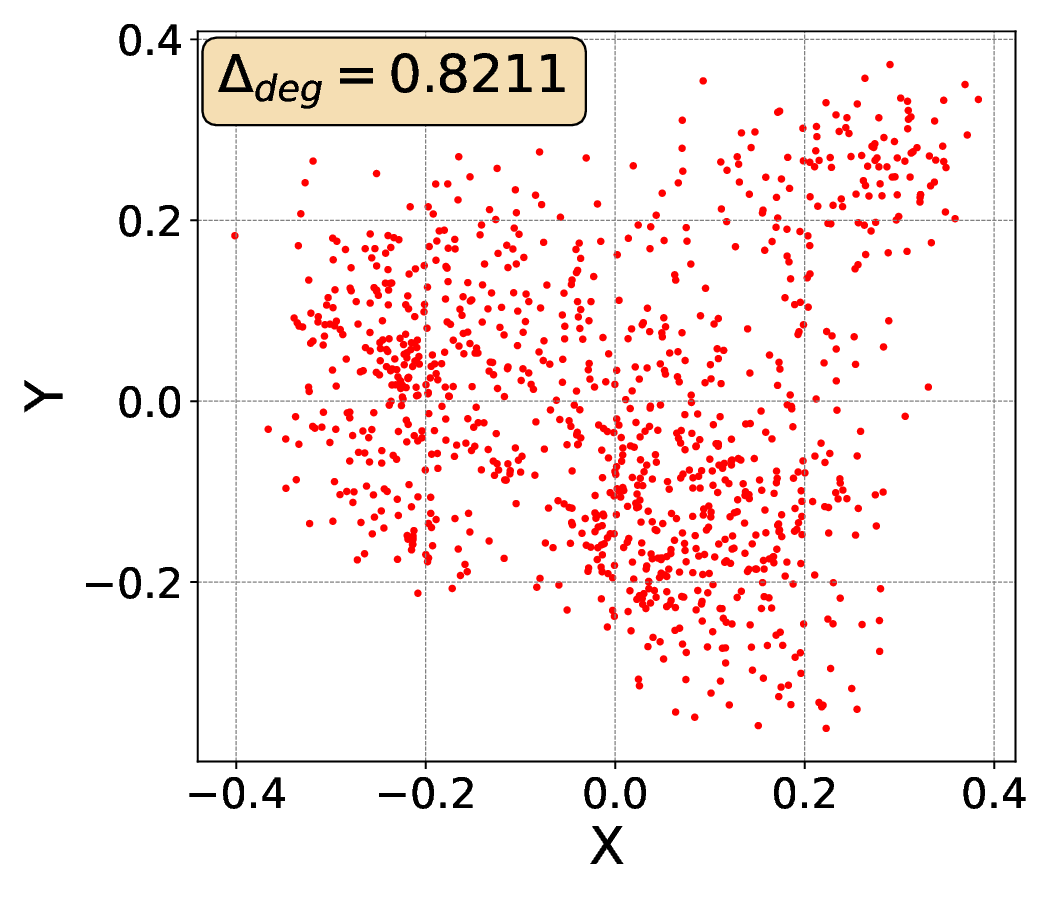}
    }
    \subfloat[MSCOCO (CLIP, I2T).]{
        \centering
        \includegraphics[width=0.24\textwidth]{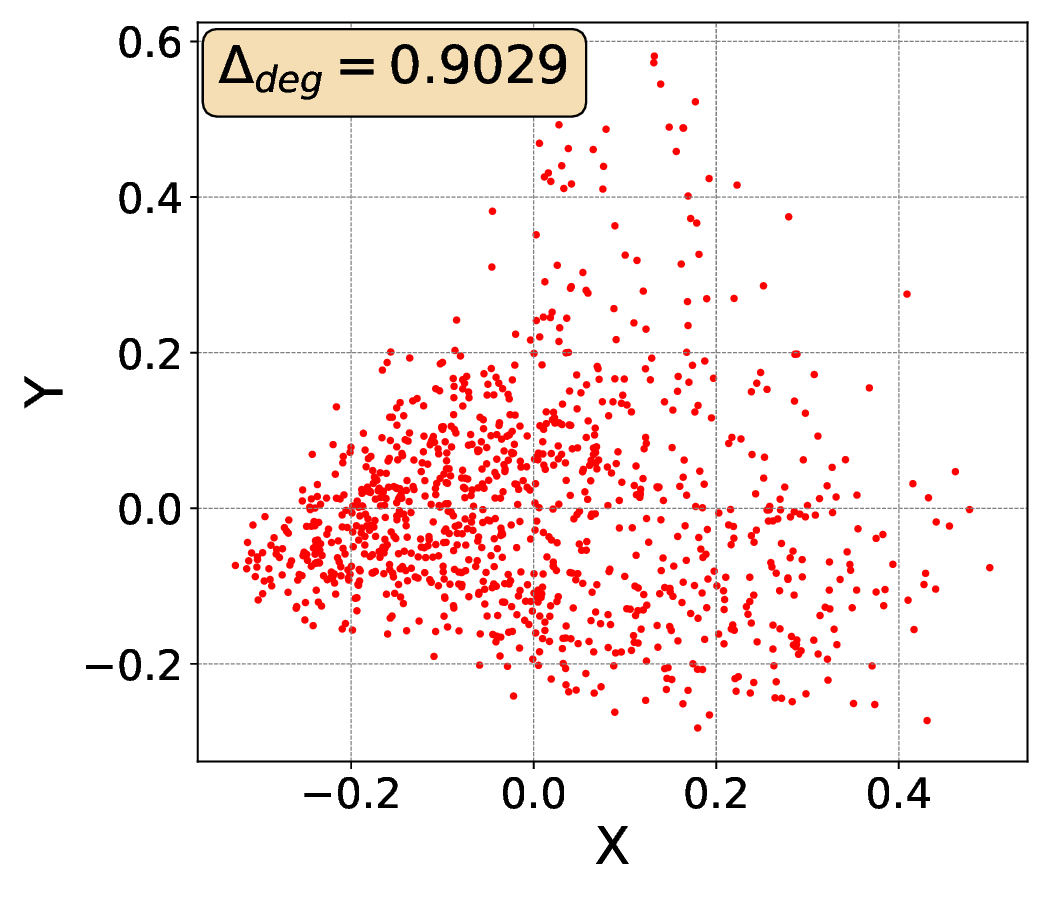}
    }
    
    \subfloat[MSCOCO (Oscar, I2T).]{
        \centering
        \includegraphics[width=0.24\textwidth]{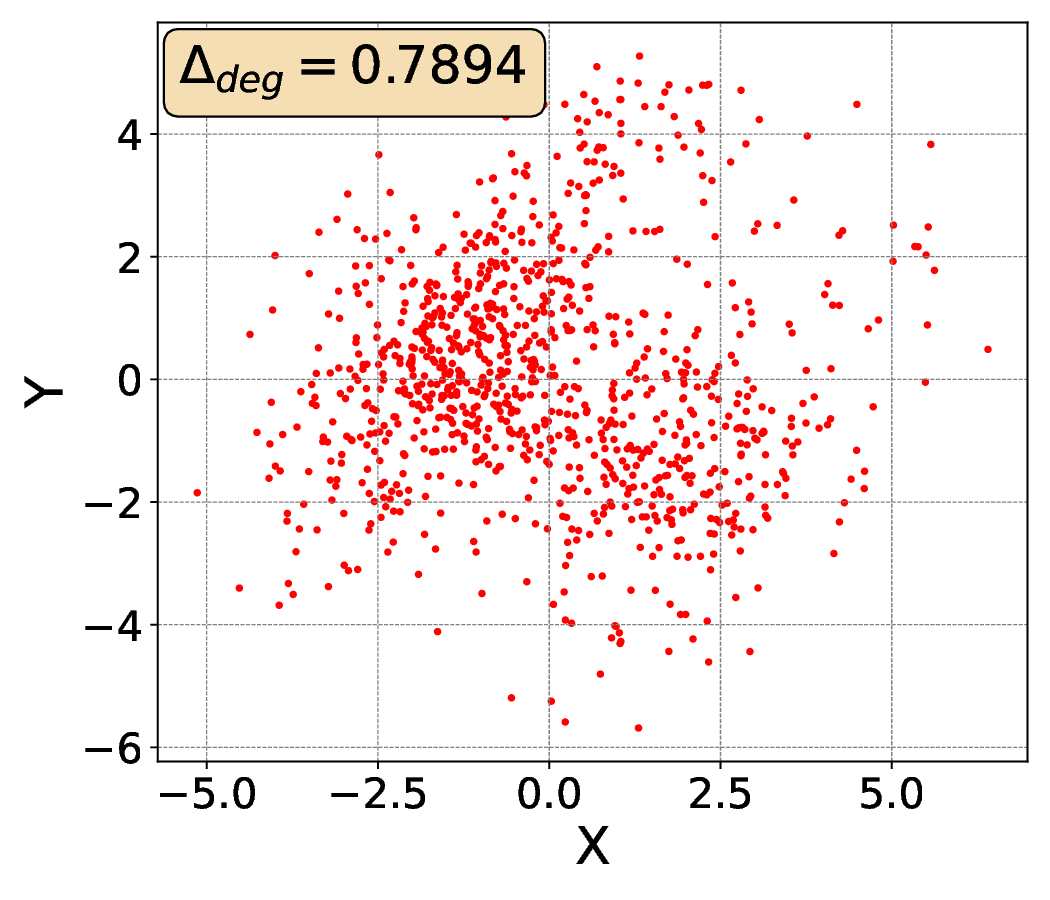}
    }
    \subfloat[Flickr30k (CLIP, I2T).]{
        \centering
        \includegraphics[width=0.24\textwidth]{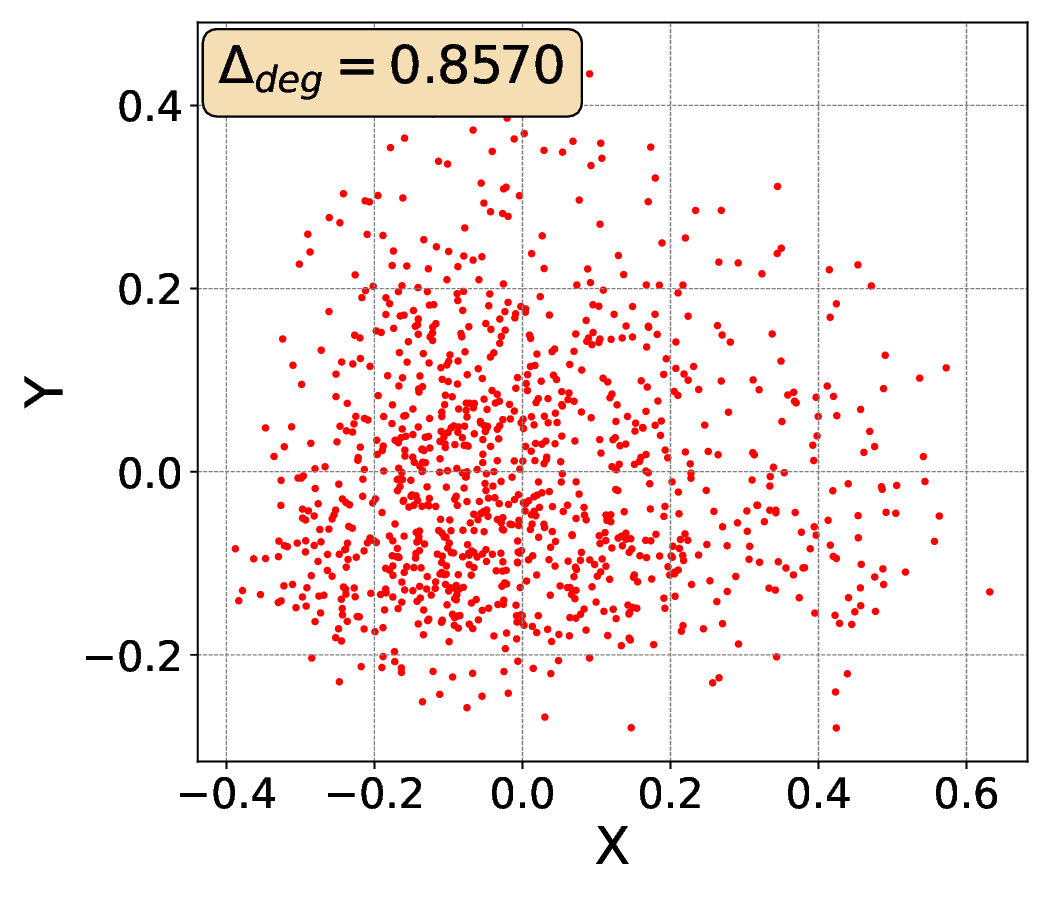}
    }
    \subfloat[Flickr30k (Oscar, T2I).]{
        \centering
        \includegraphics[width=0.24\textwidth]{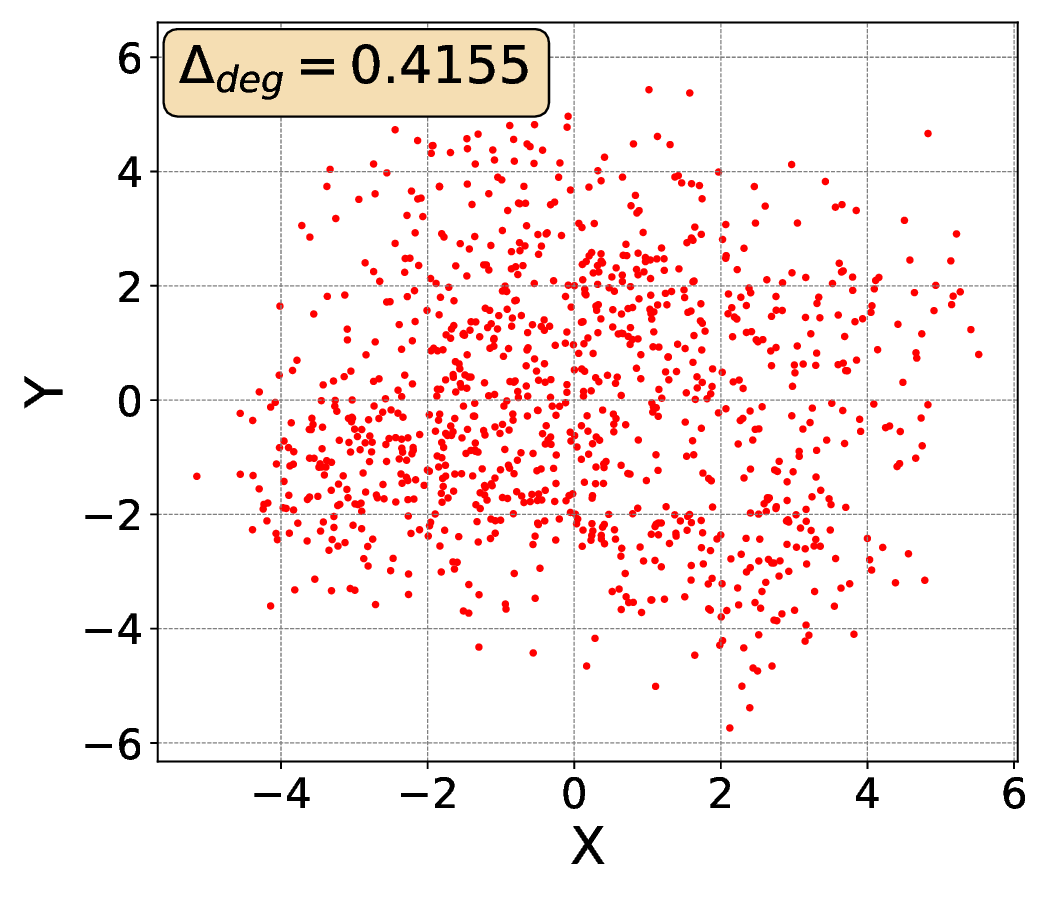}
    }
    \subfloat[MSVD (CLIP2Video, T2V).]{
        \centering
        \includegraphics[width=0.24\textwidth]{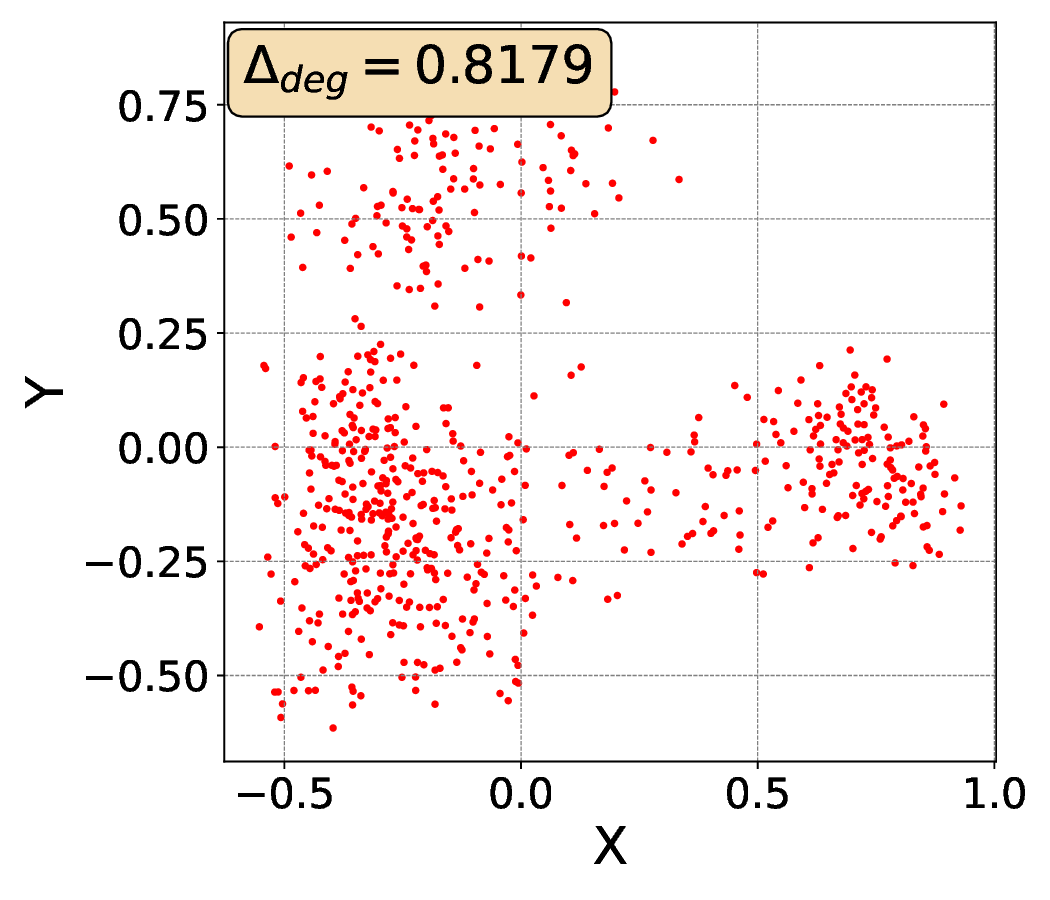}
    }

    \subfloat[MSVD (CLIP2Video, V2T).]{
        \centering
        \includegraphics[width=0.24\textwidth]{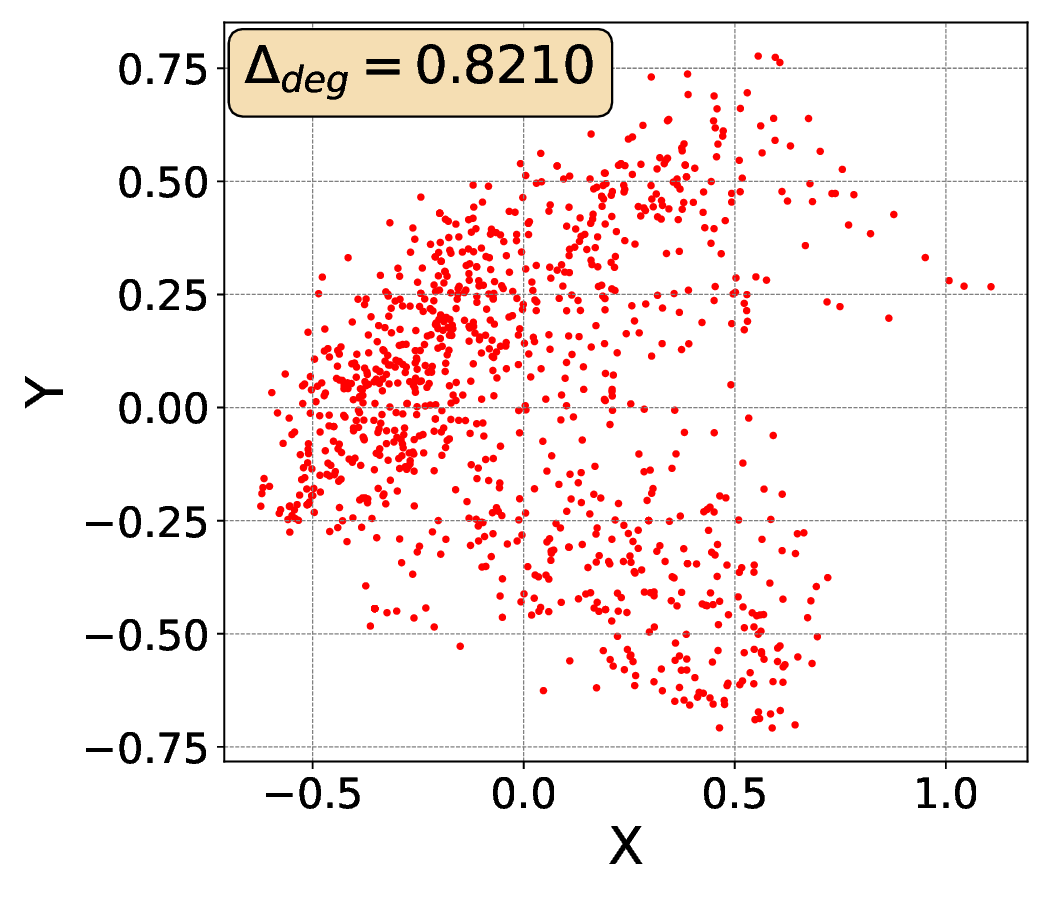}
    }
    \subfloat[MSR-VTT (X-CLIP, T2V).]{
        \centering
        \includegraphics[width=0.24\textwidth]{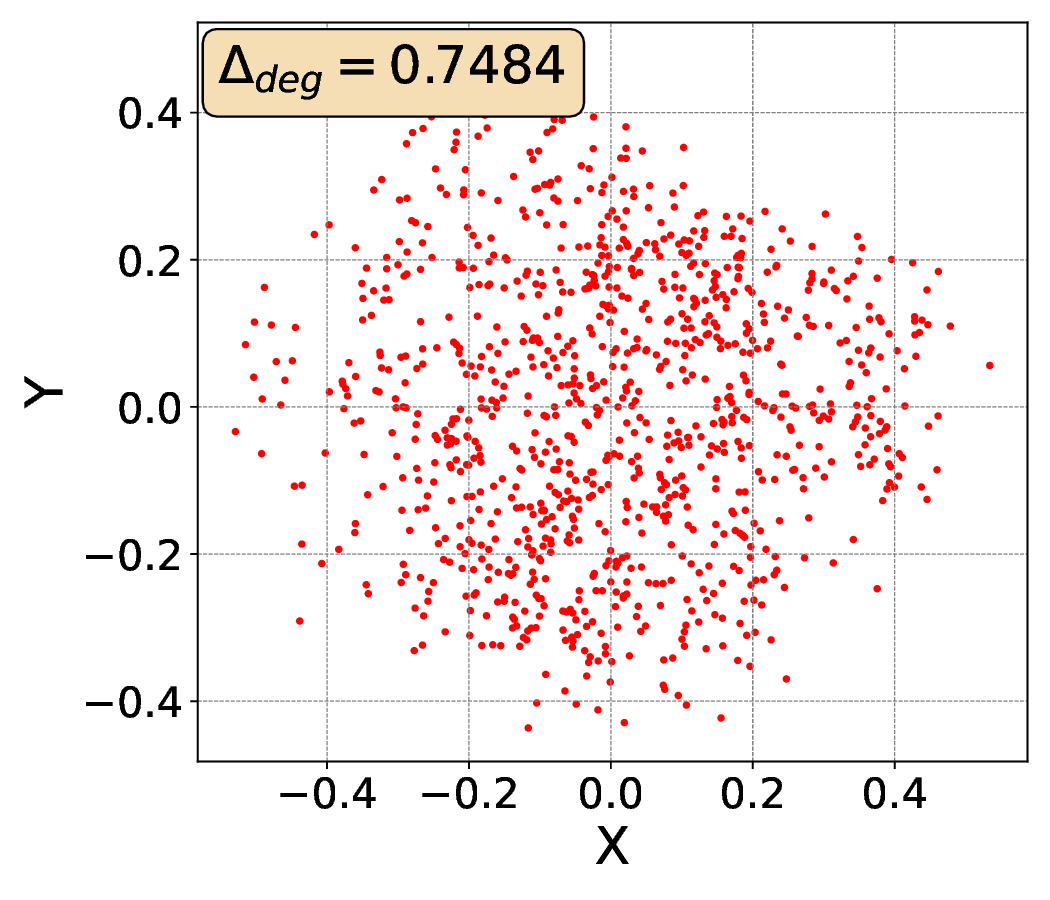}
    }
    \subfloat[ActivityNet (CE+, V2T).]{
        \centering
        \includegraphics[width=0.24\textwidth]{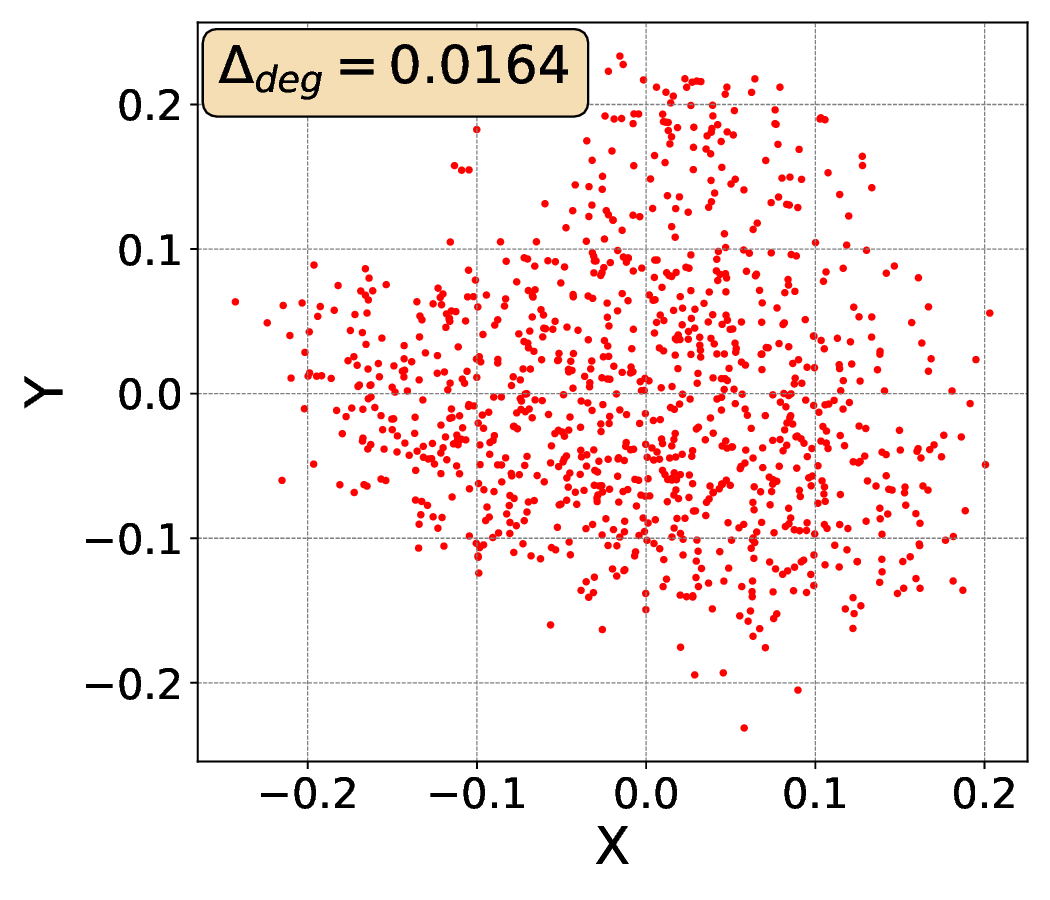}
    }
    \subfloat[MSVD (CE+, T2V)]{
        \centering
        \includegraphics[width=0.24\textwidth]{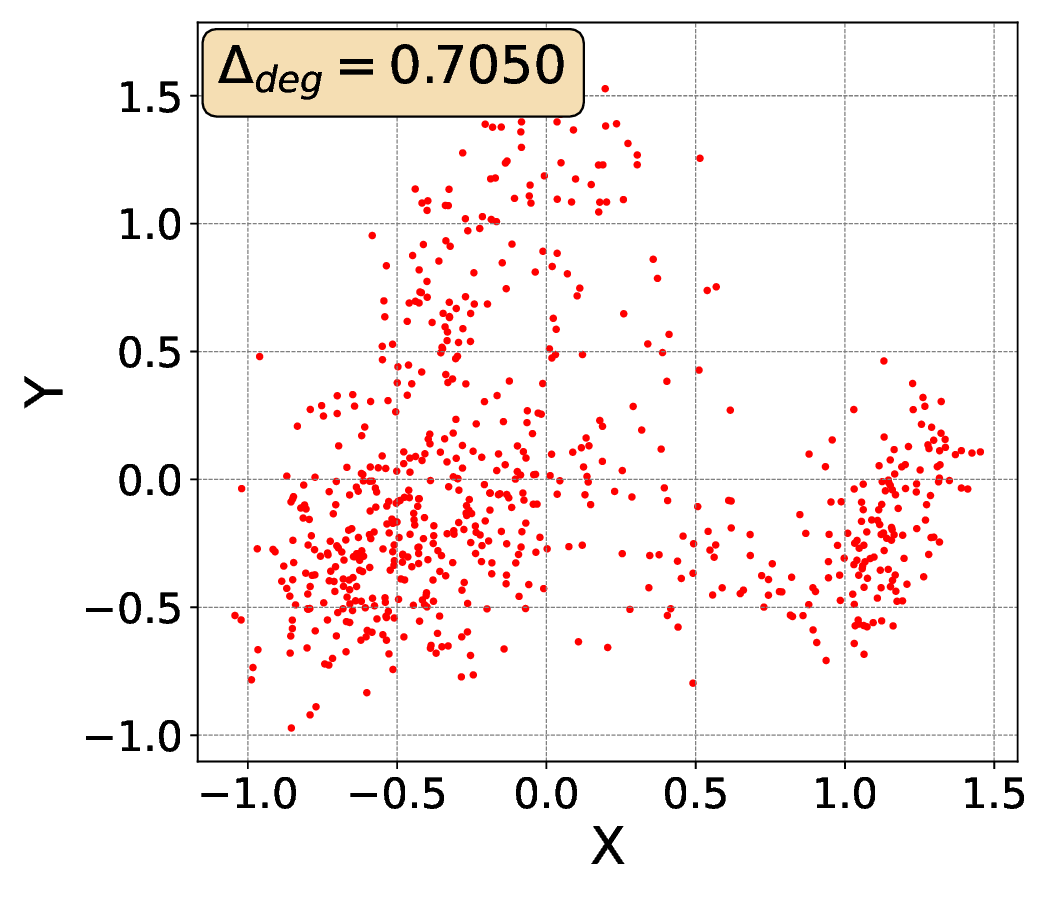}
    }

    \subfloat[MSVD (TT-CE+, T2V)]{
        \centering
        \includegraphics[width=0.24\textwidth]{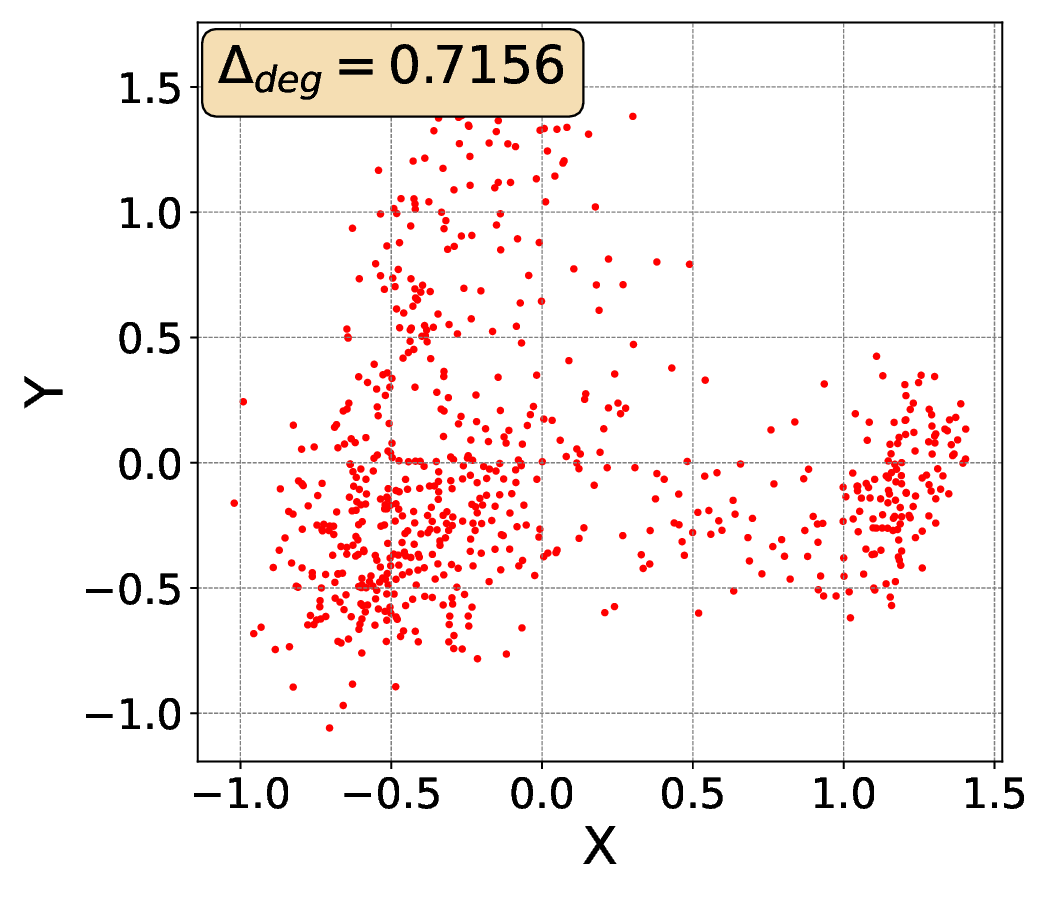}
    }
    \subfloat[MSVD (TT-CE+, V2T)]{
        \centering
        \includegraphics[width=0.24\textwidth]{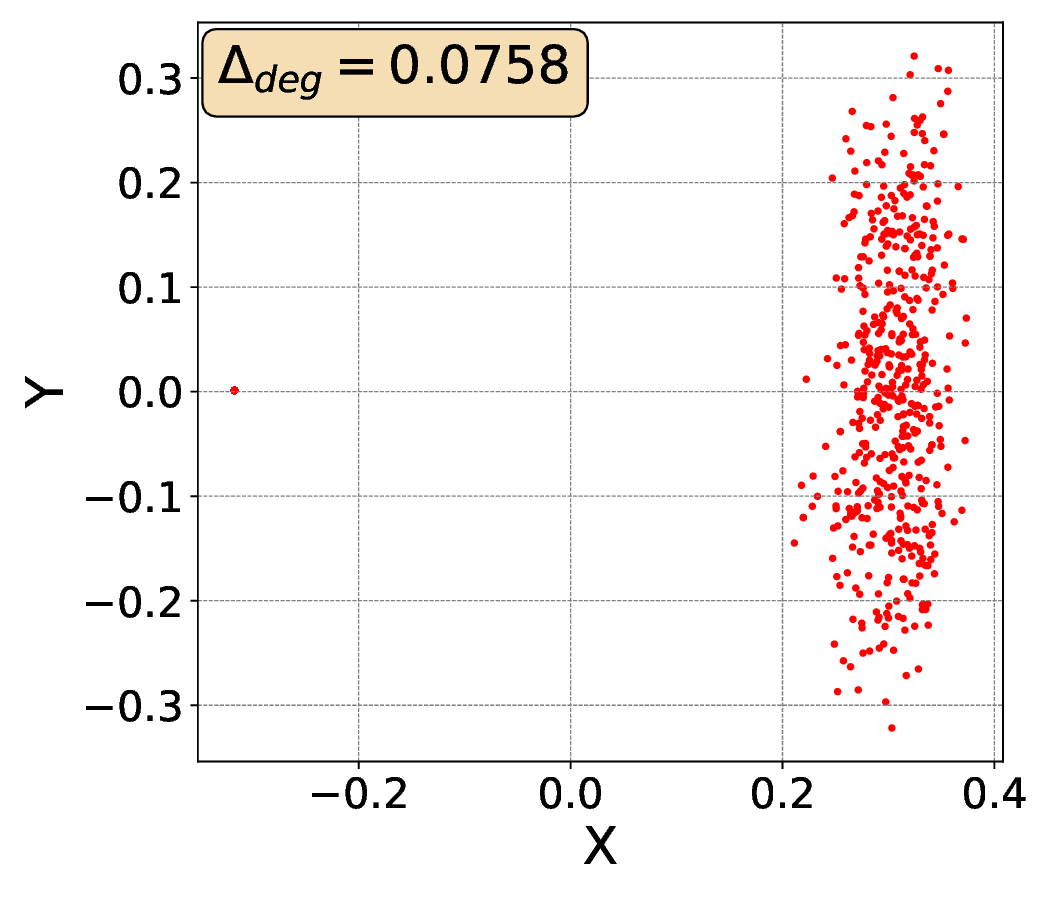}
    }
    \subfloat[MSVD (CLIP4Clip, T2V).]{
        \centering
        \includegraphics[width=0.24\textwidth]{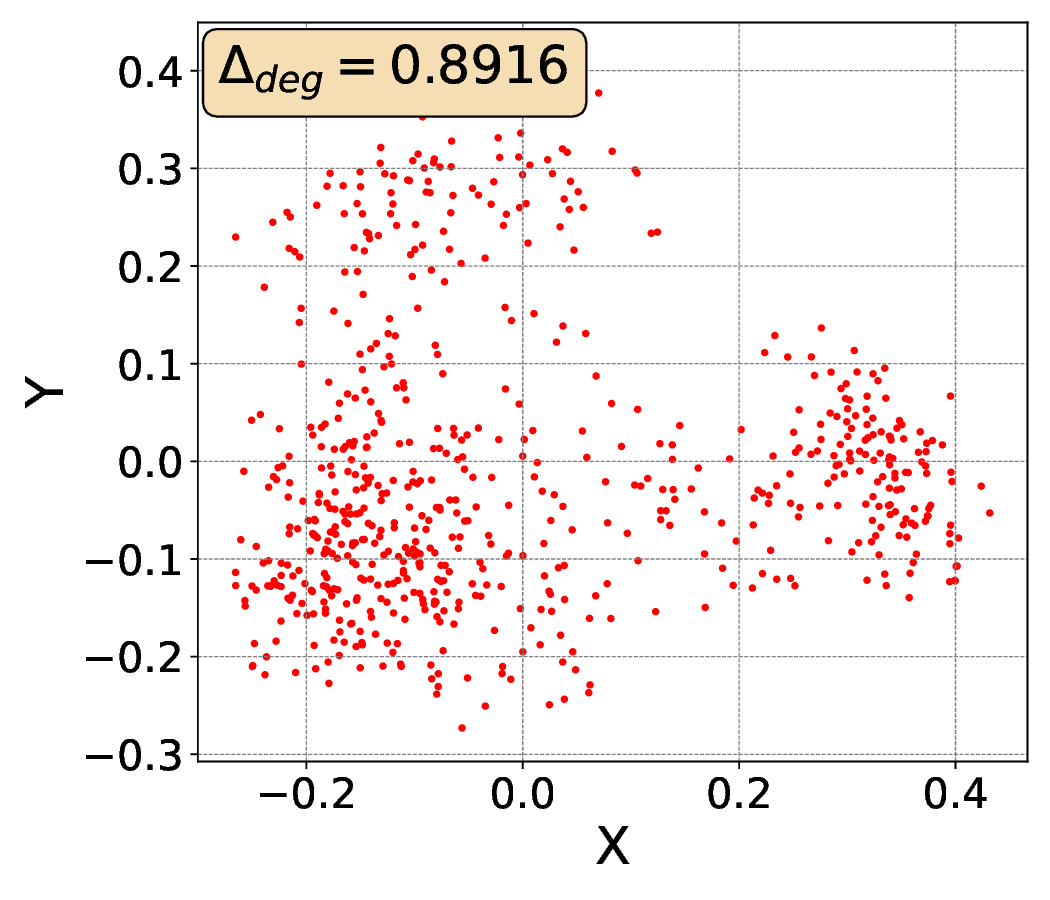}
    } 
    \subfloat[MSVD (CLIP4Clip, V2T).]{
        \centering
        \includegraphics[width=0.24\textwidth]{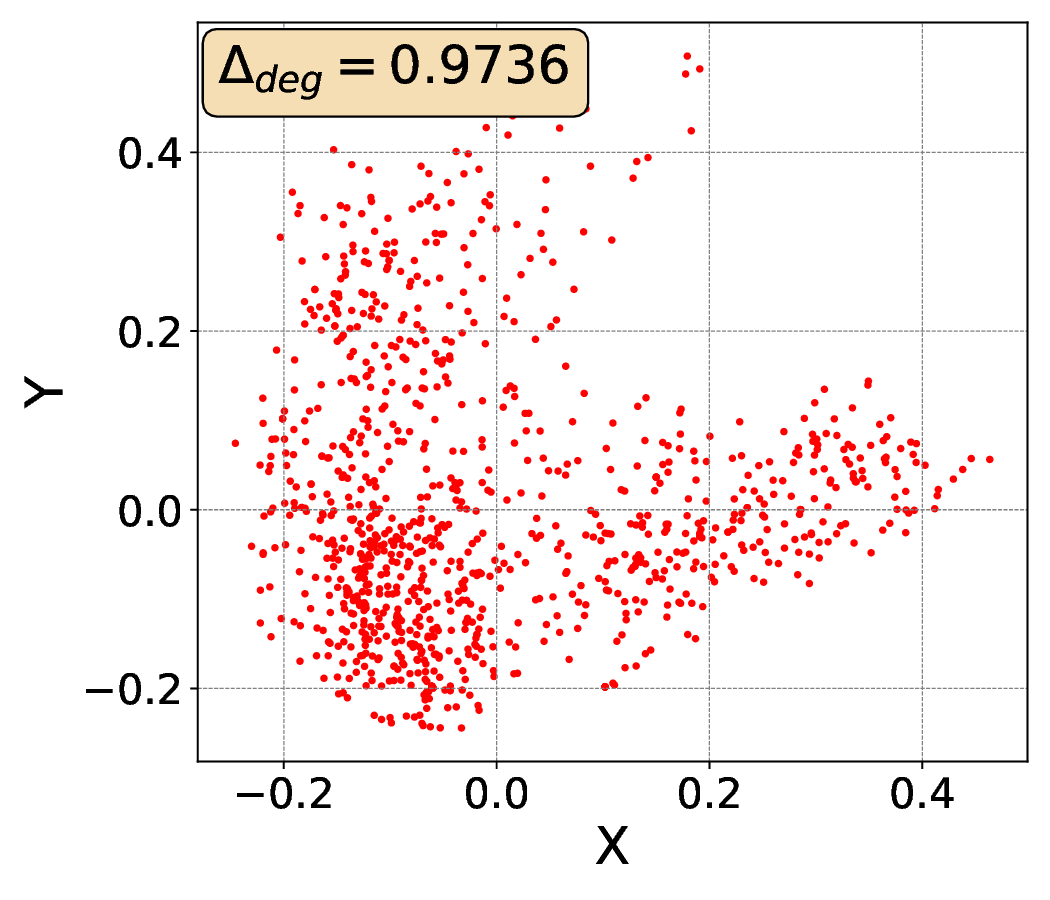}
    }
    
    \subfloat[MSVD (X-CLIP, T2V).]{
        \centering
        \includegraphics[width=0.24\textwidth]{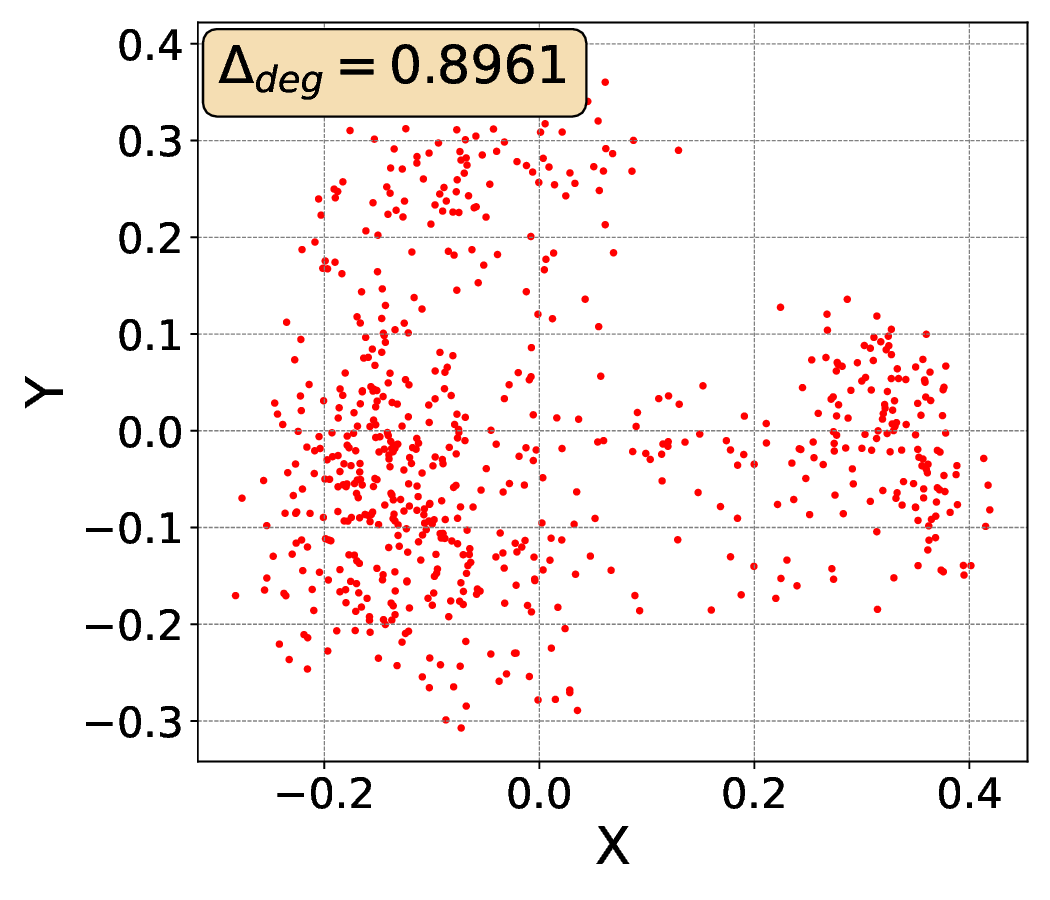}
    } 
    \subfloat[MSVD (X-CLIP, V2T).]{
        \centering
        \includegraphics[width=0.24\textwidth]{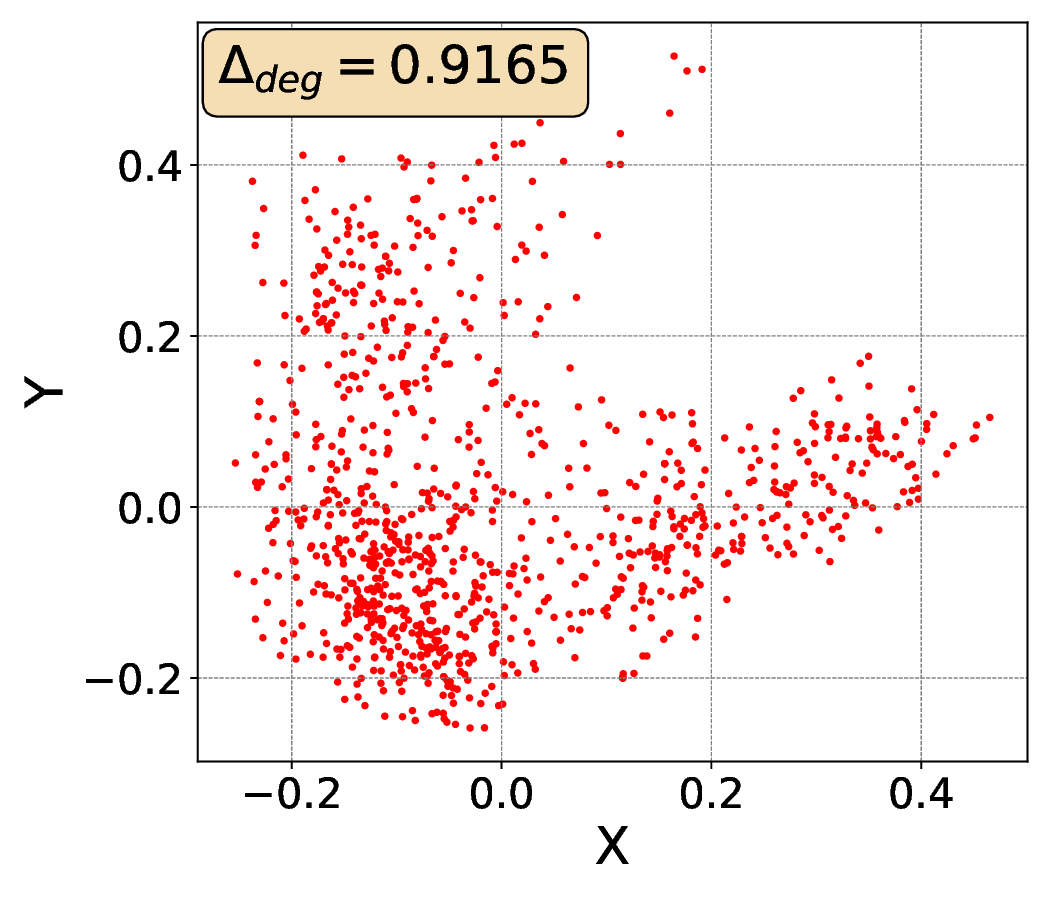}
    }
    \subfloat[DiDeMo (CE+, T2V).]{
        \centering
        \includegraphics[width=0.24\textwidth]{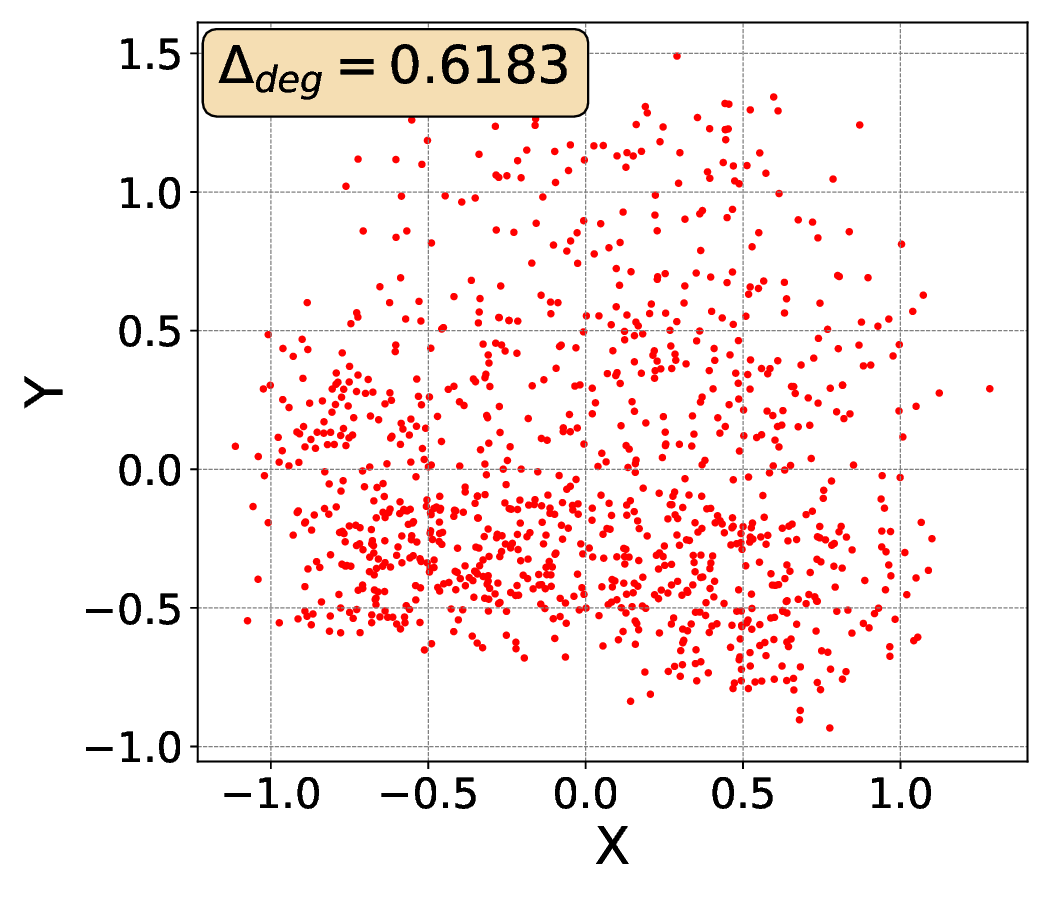}
    }
    \subfloat[DiDeMo (CE+, V2T).]{
        \centering
        \includegraphics[width=0.24\textwidth]{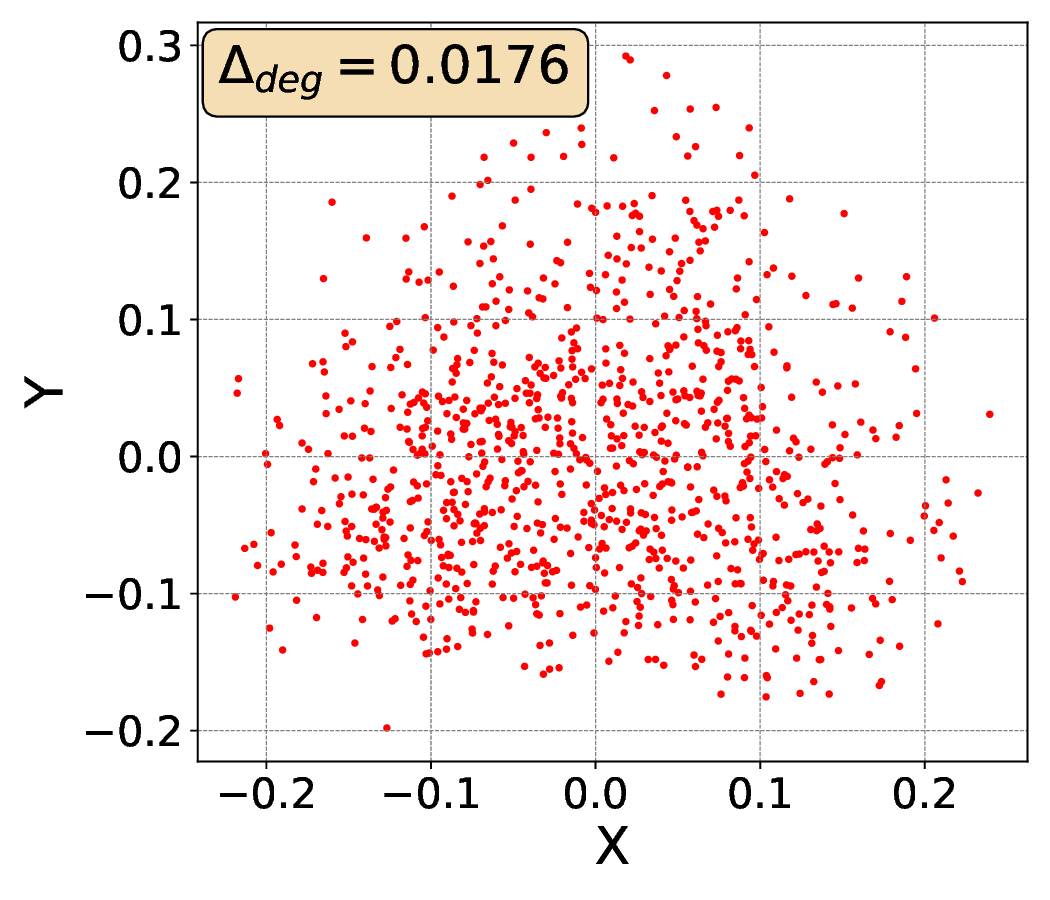}
    }
    
    \subfloat[DiDeMo (TT-CE+, T2V).]{
        \centering
        \includegraphics[width=0.24\textwidth]{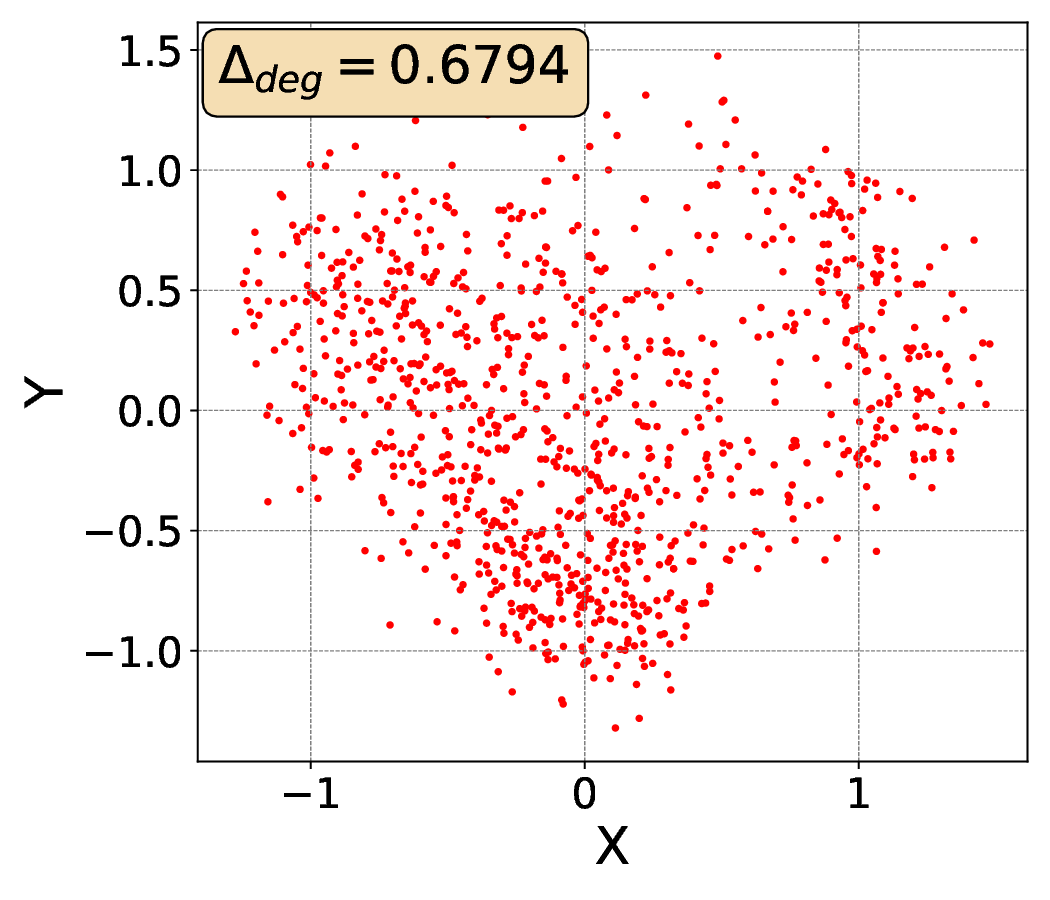}
    }
    \subfloat[DiDeMo (TT-CE+, V2T).]{
        \centering
        \includegraphics[width=0.24\textwidth]{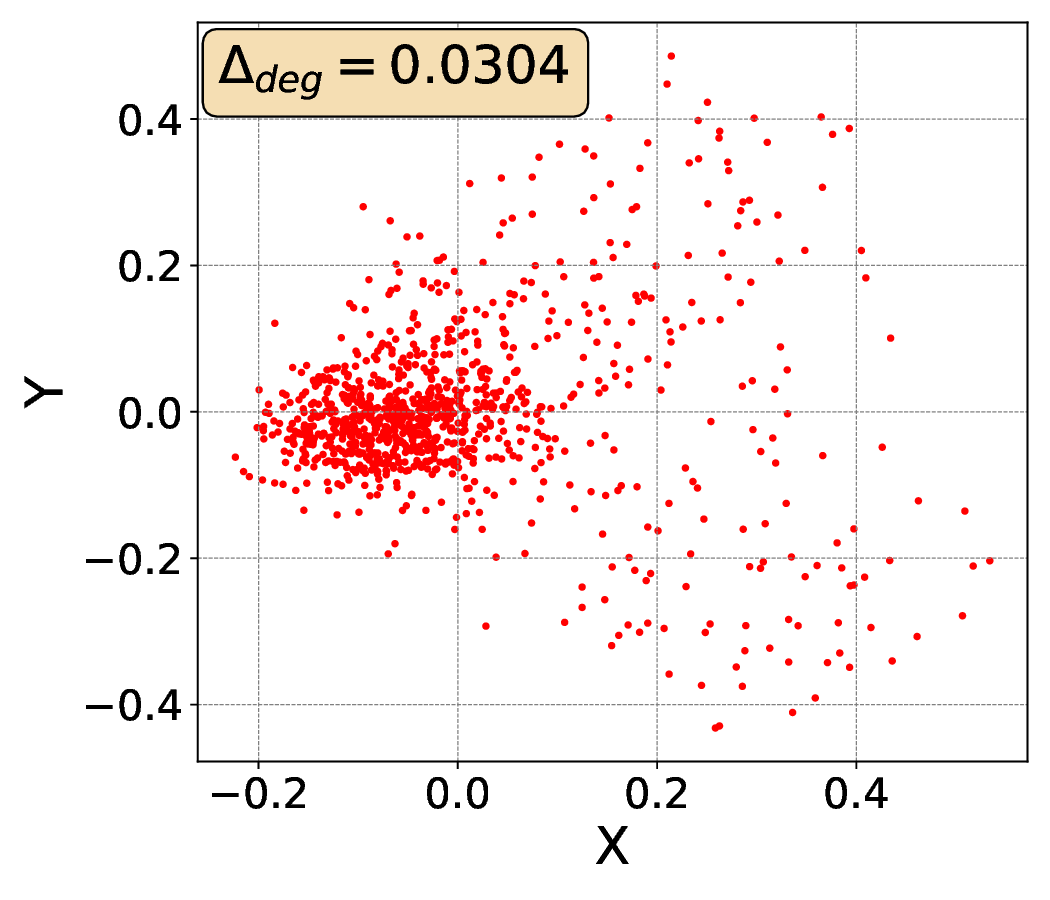}
    } 
      \caption{
      \textbf{The prevalence of \problem\ across various methods and datasets}. Each figure illustrates the distribution of the gallery set's representation for both retrieval tasks. Dimension reduction is performed using PCA, and the first two principal dimensions are chosen.
      }
      \label{fig: prevalence of problem}
   \end{center}
\end{figure*}

\begin{table*}[ht!]
\centering
\caption{Retrieval performance on Flickr30k.Best in \textbf{Bold} and the second best is \underline{underlined}. }
\label{tab: quan: f30k}
\resizebox{\textwidth}{!}{%
\begin{tabular}{ll|ccccc|ccccc}
\toprule
                       & \multirow{2}{*}{Normalization}      & \multicolumn{5}{c|}{Text-to-Image Retrieval}                            & \multicolumn{5}{c}{Image-to-Text Retrieval}                           \\
                        &                                & R@1 $\uparrow$ & R@5 $\uparrow$ & R@10 $\uparrow$ & MdR $\downarrow$ & MnR  $\downarrow$ & R@1 $\uparrow$ & R@5  $\uparrow$ & R@10 $\uparrow$ & MdR  $\downarrow$ & MnR $\downarrow$ \\ \midrule
CLIP   &                     & 58.98          & 83.48          & 90.14          & 1.0 & 6.04          & 78.10          & 94.90          & 98.10          & 1.0 & {\ul 1.98}    \\
 & +\pool (ratio=0.1)  & 59.10          & \textbf{83.56} & {\ul 90.18}    & 1.0 & 6.04          & 78.30          & 95.00          & {\ul 98.20}    & 1.0 & \textbf{1.97} \\
 & +\pool (ratio=0.25) & 59.10          & \textbf{83.56} & {\ul 90.18}    & 1.0 & 6.04          & 78.40          & 95.00          & {\ul 98.20}    & 1.0 & \textbf{1.97} \\
 & +\pool (ratio=0.5)  & 59.10          & {\ul 83.54}    & {\ul 90.18}    & 1.0 & 6.05          & 78.40          & 95.00          & {\ul 98.20}    & 1.0 & \textbf{1.97} \\
 & +\pool (ratio=1)    & 59.10          & {\ul 83.54}    & {\ul 90.18}    & 1.0 & 6.04          & 78.40          & 95.00          & {\ul 98.20}    & 1.0 & {\ul 1.98}    \\
 \rowcolor{red!10}& +\ours              & {\ul 60.18}    & 85.30          & \textbf{91.20} & 1.0 & \textbf{5.52} & {\ul 78.50}    & {\ul 95.10}    & {\ul 98.20}    & 1.0 & {\ul 1.98}    \\
\bottomrule
\end{tabular}%
}
\end{table*}

\begin{table*}[ht!]
\centering
\caption{Retrieval performance on MSR-VTT (full split and 1k split). Best in \textbf{Bold} and the second best is \underline{underlined}. 
}{}
\label{tab: quan: msrvtt}
\resizebox{\textwidth}{!}{%
\begin{tabular}{ll|ccccc|ccccc}
\toprule
          &   \multirow{2}{*}{Normalization}          & \multicolumn{5}{c|}{Text-to-Video Retrieval}             & \multicolumn{5}{c}{Video-to-Text Retrieval} \\
            &          & R@1   & R@5   & R@10  & MdR & MnR   & R@1     & R@5     & R@10    & MdR  & MnR    \\ \midrule
\multicolumn{12}{c}{MSR-VTT (full split)}\\\midrule
CE+                &                     & 13.51          & 36.01          & 48.75          & 11.0 & \textbf{70.28} & 21.61          & 50.57          & 63.48          & 5.0 & {\ul 22.62}    \\
 & +\pool (ratio=0.1)  & 13.52          & {\ul 36.02}    & {\ul 48.78}    & 11.0 & 70.36          & 21.57          & 50.54          & 63.44          & 5.0 & 22.63          \\
 & +\pool (ratio=0.25) & 13.50          & 36.00          & 48.76          & 11.0 & 70.33          & 21.57          & 50.57          & 63.48          & 5.0 & {\ul 22.62}    \\
 & +\pool (ratio=0.5)  & 13.51          & 35.99          & 48.76          & 11.0 & {\ul 70.31}    & 21.57          & 50.57          & 63.48          & 5.0 & {\ul 22.62}    \\
 & +\pool (ratio=1)    & 13.51          & 36.01          & 48.75          & 11.0 & \textbf{70.28} & 21.61          & 50.57          & 63.48          & 5.0 & {\ul 22.62}    \\
\end{tabular}%
 }
\end{table*}

\begin{table*}[ht!]
\centering
\caption{Retrieval performance on ActivityNet. Best in \textbf{Bold} and the second best is \underline{underlined}.}
\label{tab: quan: actnet}
\resizebox{\textwidth}{!}{%
\begin{tabular}{ll|ccccc|ccccc}
\toprule
                           & \multirow{2}{*}{Normalization} & \multicolumn{5}{c|}{Text-to-Video Retrieval}                            & \multicolumn{5}{c}{Video-to-Text Retrieval}                             \\
                           &                                & R@1            & R@5            & R@10           & MdR & MnR            & R@1            & R@5            & R@10           & MdR & MnR            \\ \midrule
CE+      &                     & 19.16          & {\ul 49.79}    & {\ul 65.79}    & {\ul 6.0}    & 21.99          & 18.51          & 47.85          & {\ul 63.94}    & 6.0 & 23.06          \\
 & +\pool (ratio=0.1)  & 19.22          & 49.77          & {\ul 65.79}    & {\ul 6.0}    & {\ul 21.98}    & 18.63          & 47.87          & 63.92          & 6.0 & 23.03          \\
 & +\pool (ratio=0.25) & 19.20          & {\ul 49.79}    & 65.77          & {\ul 6.0}    & {\ul 21.98}    & 18.59          & 47.94          & 63.92          & 6.0 & {\ul 23.01}    \\
 & +\pool (ratio=0.5)  & 19.20          & {\ul 49.79}    & 65.77          & {\ul 6.0}    & {\ul 21.98}    & 18.59          & 47.90          & {\ul 63.94}    & 6.0 & {\ul 23.01}    \\
 & +\pool (ratio=1)    & 19.16          & {\ul 49.79}    & {\ul 65.79}    & {\ul 6.0}    & 21.99          & 18.51          & 47.85          & {\ul 63.94}    & 6.0 & 23.06          \\
 \rowcolor{red!10}& +\ours              & {\ul 19.34}    & 49.58          & \textbf{66.02} & {\ul 6.0}    & \textbf{21.33} & {\ul 18.71}    & {\ul 48.12}    & \textbf{64.10} & 6.0 & \textbf{22.75} \\
\bottomrule
\end{tabular}%
}
\end{table*}

\begin{table*}[h!]
\centering
\caption{Retrieval performance on MSVD. Best in \textbf{Bold} and the second best is \underline{underlined}.}
\label{tab: quan: mvd}
\resizebox{\textwidth}{!}{%
\begin{tabular}{ll|ccccc|ccccc}
\toprule
                            & \multirow{2}{*}{Normalization} & \multicolumn{5}{c}{Text-to-Video Retrieval}                             & \multicolumn{5}{c}{Video-to-Text Retrieval}                              \\
                            &                                & R@1            & R@5            & R@10           & MdR & MnR            & R@1            & R@5            & R@10           & MdR  & MnR            \\\midrule
CE+        &                     & 23.94          & 54.98          & 69.00          & 4.0 & {\ul 18.46}    & 22.84          & 49.85          & {\ul 61.49}    & 6.0          & {\ul 33.96}    \\
 & +\pool (ratio=0.1)  & 24.11          & 55.26          & 69.25          & 4.0 & 18.62          & 23.88          & 50.15          & 60.15          & 5.5          & 35.54          \\
 & +\pool (ratio=0.25) & 24.01          & 55.10          & 69.09          & 4.0 & 18.60          & 23.13          & 50.15          & 60.60          & 5.5          & 34.86          \\
 & +\pool (ratio=0.5)  & 23.95          & 55.02          & 69.05          & 4.0 & 18.52          & 23.58          & 49.85          & 60.45          & 6.0          & 34.51          \\
 & +\pool (ratio=1)    & 23.94          & 54.98          & 69.00          & 4.0 & {\ul 18.46}    & 22.84          & 49.85          & {\ul 61.49}    & 6.0          & {\ul 33.96}    \\
\rowcolor{red!10} & +\ours              & {\ul 24.47}    & {\ul 55.43}    & {\ul 69.54}    & 4.0 & 25.31          & \textbf{28.51} & {\ul 51.34}    & {\ul 61.49}    & {\ul 5.0}    & 70.09          \\
\bottomrule
\end{tabular}%
}
\end{table*}

\begin{table*}[h!]
\centering
\caption{Retrieval performance on DiDeMo. Best in \textbf{Bold} and the second best is \underline{underlined}.}
\label{tab: quan: didemo}
\resizebox{\textwidth}{!}{%
\begin{tabular}{ll|ccccc|ccccc}
\toprule
\multirow{2}{*}{Method} & \multirow{2}{*}{Normalization}     & \multicolumn{5}{c|}{Text-to-Video Retrieval} & \multicolumn{5}{c}{Video-to-Text Retrieval} \\
                        &                                    & R@1 $\uparrow$    & R@5 $\uparrow$    & R@10 $\uparrow$   & MdR $\downarrow$   & MnR  $\downarrow$  & R@1 $\uparrow$   & R@5  $\uparrow$   & R@10 $\uparrow$  & MdR  $\downarrow$  & MnR $\downarrow$   \\\midrule
CE+    &                     & 18.23          & 42.63          & 56.08          & 8.0 & 42.05          & 18.82          & {\ul 42.73}    & 55.78          & 8.0 & {\ul 37.27}    \\
 & +\pool (ratio=0.1)  & 18.23          & 42.73          & {\ul 56.27}    & 8.0 & 42.19          & 18.82          & 42.03          & 55.88          & 8.0 & 37.73          \\
 & +\pool (ratio=0.25) & 18.23          & 42.73          & {\ul 56.27}    & 8.0 & 42.11          & 18.92          & 42.43          & 55.88          & 8.0 & 37.36          \\
 & +\pool (ratio=0.5)  & 18.13          & 42.83          & 56.18          & 8.0 & 42.07          & 18.92          & 42.33          & 55.98          & 8.0 & 37.52          \\
 & +\pool (ratio=1)    & 18.23          & 42.63          & 56.08          & 8.0 & 42.05          & 18.82          & {\ul 42.73}    & 55.78          & 8.0 & {\ul 37.27}    \\
 \rowcolor{red!10}& +\ours              & {\ul 18.82}    & \textbf{43.13} & \textbf{56.67} & 8.0 & \textbf{41.69} & {\ul 19.02}    & 42.43          & {\ul 56.18}    & 8.0 & \textbf{37.08} \\
\bottomrule
\end{tabular}%
}
\end{table*}

\begin{table*}[ht!]
\centering
\caption{Audio-text retrieval performance on CLOTHO. Best in \textbf{Bold} and the second best is \underline{underlined}. 
'*' refers to the results direct copied from \citep{koepke_audio_2022}. 
MoEE~\cite{miech_learning_2020}, MMT~\cite{DBLP:conf/eccv/Gabeur0AS20} are two methods used for audio-text retrieval.
}
\label{tab: quan: clotho appendix (full)}
\resizebox{\textwidth}{!}{%
\begin{tabular}{ll|ccccc|ccccc}
\toprule
\multirow{2}{*}{Method} & \multirow{2}{*}{Normalization}     & \multicolumn{5}{c|}{Text-to-Audio Retrieval} & \multicolumn{5}{c}{Audio-to-Text Retrieval} \\
                        &                                    & R@1 $\uparrow$    & R@5 $\uparrow$    & R@10 $\uparrow$   & MdR $\downarrow$   & MnR  $\downarrow$  & R@1 $\uparrow$   & R@5  $\uparrow$   & R@10 $\uparrow$  & MdR  $\downarrow$  & MnR $\downarrow$   \\\midrule
                        \textcolor{gray}{MoEE}*                                                &                                                     & \textcolor{gray}{6.00  }         & \textcolor{gray}{20.80  }        & \textcolor{gray}{32.30  }         & \textcolor{gray}{23.0   }          & \textcolor{gray}{60.20   }                               & \textcolor{gray}{7.20 }          & \textcolor{gray}{22.10   }        & \textcolor{gray}{33.20   }        & \textcolor{gray}{22.7  }           & \textcolor{gray}{71.80 }           \\
\textcolor{gray}{MMT}*                                                 &                                                     & \textcolor{gray}{6.50     }      & \textcolor{gray}{21.60   }       & \textcolor{gray}{66.90 }          & \textcolor{gray}{23.0}             & \textcolor{gray}{67.70}                                  & \textcolor{gray}{6.30    }       & \textcolor{gray}{22.80 }          & \textcolor{gray}{33.30   }       & \textcolor{gray}{22.3 }            & \textcolor{gray}{67.30 }           \\ \midrule
AR-CE  &                     & 6.27          & 22.32          & 33.30          & 23.0          & 58.96          & 7.27          & 23.35          & 35.12          & \textbf{21.0} & {\ul 74.68}    \\
 & +\pool (ratio=0.1)  & 6.37          & {\ul 22.33}    & 33.59          & 23.0          & 65.57          & 7.66          & 22.01          & 33.59          & {\ul 22.0}    & 85.61          \\
\bottomrule
\end{tabular}%
}
\end{table*}

\begin{table*}[ht!]
\centering
\caption{
Text-audio retrieval performance on AudioCaps. Best in \textbf{Bold} and the second best is \underline{underlined}. 
'*' refers to the results direct copied from \citep{koepke_audio_2022}. 
MoEE~\cite{miech_learning_2020}, MMT~\cite{DBLP:conf/eccv/Gabeur0AS20} are two methods used for audio-text retrieval.
}
\label{tab: quan: audiocaps appendix (full)}
\resizebox{\textwidth}{!}{%
\begin{tabular}{ll|ccccc|ccccc}
\toprule
\multirow{2}{*}{Method} & \multirow{2}{*}{Normalization} & \multicolumn{5}{c|}{Text-to-Audio Retrieval}                                             & \multicolumn{5}{c}{Audio-to-Text Retrieval}                                               \\
                        &                                & R@1 $\uparrow$ & R@5 $\uparrow$ & R@10 $\uparrow$ & MdR $\downarrow$ & MnR  $\downarrow$ & R@1 $\uparrow$ & R@5  $\uparrow$ & R@10 $\uparrow$ & MdR  $\downarrow$ & MnR $\downarrow$ \\ \midrule
\textcolor{gray}{MoEE}*                                                &                                                     & \textcolor{gray}{23.00}          & \textcolor{gray}{55.70}          & \textcolor{gray}{71.00}           & \textcolor{gray}{4.0}              & \textcolor{gray}{16.30}                                  & \textcolor{gray}{26.60}          & \textcolor{gray}{59.30}           & \textcolor{gray}{73.50}           & \textcolor{gray}{4.0}               & \textcolor{gray}{15.60}            \\
\textcolor{gray}{MMT}*                                                 &                                                     & \textcolor{gray}{36.10}          & \textcolor{gray}{72.00}          & \textcolor{gray}{84.50}           & \textcolor{gray}{2.3}              & \textcolor{gray}{7.50}                                   & \textcolor{gray}{39.60}          & \textcolor{gray}{76.80}           & \textcolor{gray}{86.70}           & \textcolor{gray}{2.0}               & \textcolor{gray}{6.50}             \\ \midrule
AR-CE    &                     & {\ul 22.33}    & {\ul 54.49}    & {\ul 70.54}    & {\ul 5.0}    & \textbf{15.89} & 24.02          & 56.00          & {\ul 71.81}    & 4.0 & {\ul 16.91}    \\
 & +\pool (ratio=0.1)  & {\ul 22.33}    & {\ul 54.49}    & {\ul 70.54}    & {\ul 5.0}    & \textbf{15.89} & 24.14          & 55.88          & 71.69          & 4.0 & 16.93          \\
\bottomrule
\end{tabular}%
}
\end{table*}

\begin{table*}[ht!]
\centering
\caption{Similarity measures within gallery set (MeanSim@1 equivalent to $\Delta_{deg}(G)$) on CLOTHO.}
\label{tab: deg: clotho appendix (full)}
\resizebox{\textwidth}{!}{%
\begin{tabular}{ll|ccc|ccc}
\toprule
\multirow{2}{*}{Method} & \multirow{2}{*}{Normalization}     & \multicolumn{3}{c|}{Text-to-Audio Retrieval} & \multicolumn{3}{c}{Audio-to-Text Retrieval} \\
                        & & MeanSim & MeanSim@1 & MeanSim@10 & MeanSim & MeanSim@1 & MeanSim@10  \\\midrule

AR-CE  &   &  0.0436 & 0.7518 & 0.6652 & 0.0143 & 0.2462 & 0.2144  \\

 \rowcolor{red!10}   & +\ours           &   0.0391 & 0.5246 & 0.4567 & 0.0115 & 0.1261 & 0.1017  \\
\rowcolor{red!10}   & +\fcut     &  0.0279 & 0.4739 & 0.4094 & 0.0114 & 0.1713 & 0.1420\\ 
                        \bottomrule
\end{tabular}%
}
\end{table*}

\begin{table*}[ht!]
\centering
\caption{
Similarity measures within gallery set (MeanSim@1 equivalent to $\Delta_{deg}(G)$) on AudioCaps.}
\label{tab: deg: audiocaps appendix (full)}
\resizebox{\textwidth}{!}{%
\begin{tabular}{ll|ccc|ccc}
\toprule
\multirow{2}{*}{Method} & \multirow{2}{*}{Normalization}     & \multicolumn{3}{c|}{Text-to-Audio Retrieval} & \multicolumn{3}{c}{Audio-to-Text Retrieval} \\
                        & & MeanSim & MeanSim@1 & MeanSim@10 & MeanSim & MeanSim@1 & MeanSim@10  \\\midrule

AR-CE  &   &  0.0358 & 0.6781 & 0.5622 & 0.0075 & 0.2656 & 0.2330  \\

 \rowcolor{red!10}  & +\ours            & 0.0358 & 0.6778 & 0.5620  & 0.006  & 0.1837 & 0.1582     \\
 \rowcolor{red!10} & +\fcut       & 0.0161 & 0.3264 & 0.2602 & 0.0054 & 0.1399 & 0.1142 \\ 
                        \bottomrule
\end{tabular}%
}
\end{table*}

\begin{table*}[ht!]
\centering
\caption{Similarity measures within gallery set (MeanSim@1 equivalent to $\Delta_{deg}(G)$) on MSCOCO (5k split).}
\label{tab: deg: coco}
\resizebox{\textwidth}{!}{%
\begin{tabular}{ll|ccc|ccc}
\toprule
\multirow{2}{*}{Method} & \multirow{2}{*}{Normalization}     & \multicolumn{3}{c|}{Text-to-Image Retrieval} & \multicolumn{3}{c}{Image-to-Text Retrieval} \\
                        & & MeanSim & MeanSim@1 & MeanSim@10 & MeanSim & MeanSim@1 & MeanSim@10  \\\midrule

CLIP  &    & 0.4717 & 0.8211 & 0.7803 & 0.5161 & 0.9029 & 0.8596  \\

 \rowcolor{red!10}& +\ours              &    0.4693 & 0.8142 & 0.7738 & 0.5137 & 0.8972 & 0.8541 \\
\rowcolor{red!10}  & +\fcut       &  0.4646 & 0.8059 & 0.7647 & 0.5105 & 0.8924 & 0.8477\\ 
                        \midrule
Oscar  &   &   0.0509 & 0.6014 & 0.4867 & 0.0166 & 0.7894 & 0.6601 \\

\rowcolor{red!10} & +\ours              & 0.0536 & 0.5960  & 0.4810  & 0.0172 & 0.7836 & 0.6523    \\
 \rowcolor{red!10} & +\fcut       & 0.0703 & 0.5610  & 0.4459 & 0.0374 & 0.7658 & 0.6285 \\ 
                        \bottomrule
\end{tabular}%
}
\end{table*}

\begin{table*}[ht!]
\centering
\caption{Similarity measures within gallery set (MeanSim@1 equivalent to $\Delta_{deg}(G)$) on Flickr30k.}
\label{tab: deg: f30k}
\resizebox{\textwidth}{!}{%
\begin{tabular}{ll|ccc|ccc}
\toprule
\multirow{2}{*}{Method} & \multirow{2}{*}{Normalization}     & \multicolumn{3}{c|}{Text-to-Image Retrieval} & \multicolumn{3}{c}{Image-to-Text Retrieval} \\
                        & & MeanSim & MeanSim@1 & MeanSim@10 & MeanSim & MeanSim@1 & MeanSim@10  \\\midrule

CLIP  &    &  0.4908 & 0.7639 & 0.7115 & 0.5119 & 0.8570  & 0.7882 \\

\rowcolor{red!10} & +\ours              &   0.4868 & 0.7575 & 0.7056 & 0.5086 & 0.8515 & 0.7829  \\
\rowcolor{red!10} & +\fcut        & 0.4671 & 0.7481 & 0.6942 & 0.4682 & 0.8358 & 0.7604 \\ 
                        \midrule
Oscar  &    & 0.0143 & 0.4155 & 0.3007 & 0.0102 & 0.7059 & 0.5050  \\

\rowcolor{red!10} & +\ours              &  0.0147 & 0.4015 & 0.2854 & 0.0116 & 0.6819 & 0.4778    \\
\rowcolor{red!10} & +\fcut        &  0.0165 & 0.4077 & 0.2948 & 0.0139 & 0.6915 & 0.4880\\ 
                        \bottomrule
\end{tabular}%
}
\end{table*}

\begin{table*}[ht!]
\centering
\caption{Similarity measures within gallery set (MeanSim@1 equivalent to $\Delta_{deg}(G)$) on MSR-VTT (full split and 1k split).}{}
\label{tab: deg: msrvtt}
\resizebox{\textwidth}{!}{%
\begin{tabular}{ll|ccc|ccc}
\toprule
\multirow{2}{*}{Method} & \multirow{2}{*}{Normalization}     & \multicolumn{3}{c|}{Text-to-Video Retrieval} & \multicolumn{3}{c}{Video-to-Text Retrieval} \\
                        & & MeanSim & MeanSim@1 & MeanSim@10 & MeanSim & MeanSim@1 & MeanSim@10  \\\midrule

\multicolumn{8}{c}{MSR-VTT (full split)}\\\midrule

CE+   &   & 0.1525 & 0.6750  & 0.5516 & 0.0006 & 0.0250  & 0.0223  \\

\rowcolor{red!10} & +\ours              &   0.1435 & 0.5974 & 0.4861 & 0.0005 & 0.0224 & 0.0199 \\
\rowcolor{red!10} & +\fcut        & 0.0845 & 0.4722 & 0.3697 & 0.0007 & 0.0196 & 0.0171 \\ 
                        \midrule

TT-CE+  &    &  0.1751 & 0.6878 & 0.5696 & 0.0008 & 0.0295 & 0.0267 \\

\rowcolor{red!10} & +\ours              &  0.1627 & 0.6080  & 0.5017 & 0.0007 & 0.0241 & 0.0216   \\
\rowcolor{red!10} & +\fcut        &  0.1196 & 0.5176 & 0.4161 & 0.0007 & 0.0237 & 0.0212\\ 
                        \midrule

\multicolumn{8}{c}{MSR-VTT (1k split)}\\\midrule

CLIP4Clip  &    &  0.5416 & 0.8201 & 0.7628 & 0.6073 & 0.8389 & 0.7922 \\

 \rowcolor{red!10}  & +\ours            &   0.5460  & 0.8195 & 0.7618 & 0.6149 & 0.8403 & 0.7919  \\
\rowcolor{red!10}& +\fcut        &  0.5280  & 0.8098 & 0.7502 & 0.6090  & 0.8385 & 0.7921\\ 
                        \midrule

CLIP2Video  &     & 0.5376 & 0.8179 & 0.7589 & 0.5769 & 0.8210  & 0.7702 \\

\rowcolor{red!10} & +\ours              &   0.5379 & 0.8179 & 0.7588 & 0.5798 & 0.8216 & 0.7703  \\
 \rowcolor{red!10} & +\fcut       &  0.5314 & 0.8129 & 0.7528 & 0.5627 & 0.8077 & 0.7559\\ 
                        \midrule

X-CLIP  &    & 0.3555 & 0.7484 & 0.6675 & 0.2626 & 0.5262 & 0.4608  \\

\rowcolor{red!10} & +\ours              &  0.3671 & 0.7367 & 0.6520  & 0.2653 & 0.5251 & 0.4600    \\
\rowcolor{red!10} & +\fcut        &  0.3054 & 0.6974 & 0.6065 & 0.2860  & 0.5344 & 0.4710 \\ 
                        \bottomrule
\end{tabular}%
}
\end{table*}

\begin{table*}[ht!]
\centering
\caption{Similarity measures within gallery set (MeanSim@1 equivalent to $\Delta_{deg}(G)$) on ActivityNet. }
\label{tab: deg: actnet}

\resizebox{\textwidth}{!}{%
\begin{tabular}{ll|ccc|ccc}
\toprule
\multirow{2}{*}{Method} & \multirow{2}{*}{Normalization}     & \multicolumn{3}{c|}{Text-to-Video Retrieval} & \multicolumn{3}{c}{Video-to-Text Retrieval} \\
                        & & MeanSim & MeanSim@1 & MeanSim@10 & MeanSim & MeanSim@1 & MeanSim@10  \\\midrule

CE+  &    &  0.1245 & 0.6950  & 0.6095 & 0.0012 & 0.0164 & 0.0141 \\

\rowcolor{red!10} & +\ours              &   0.1177 & 0.6223 & 0.5453 & 0.0011 & 0.0161 & 0.0139 \\
\rowcolor{red!10} & +\fcut        &  0.0601 & 0.4059 & 0.3489 & 0.0010  & 0.0129 & 0.0110\\ 
                        \midrule

TT-CE+   &   &  0.1347 & 0.7502 & 0.6673 & 0.0021 & 0.0209 & 0.0187 \\

\rowcolor{red!10} & +\ours              &   0.1329 & 0.7213 & 0.6412 & 0.0021 & 0.0206 & 0.0185   \\
 \rowcolor{red!10} & +\fcut       &  0.0570  & 0.4082 & 0.3563 & 0.0018 & 0.0145 & 0.0128\\ 
                        \midrule

CLIP4Clip  &    & 0.3411 & 0.8122 & 0.7566 & 0.2782 & 0.7638 & 0.6932   \\

\rowcolor{red!10} & +\ours              &  0.3412 & 0.8122 & 0.7566 & 0.2795 & 0.7635 & 0.6929   \\
 \rowcolor{red!10} & +\fcut       & 0.3554 & 0.8029 & 0.7469 & 0.2800   & 0.7483 & 0.6754 \\ 
                        \midrule

X-CLIP  &    & 0.3229 & 0.8097 & 0.7460  & 0.1842 & 0.6635 & 0.5816 \\

\rowcolor{red!10} & +\ours              &  0.3369 & 0.8052 & 0.7410  & 0.1867 & 0.6553 & 0.5743   \\
\rowcolor{red!10} & +\fcut        & 0.3033 & 0.7638 & 0.6926 & 0.1968 & 0.6464 & 0.5646 \\ 
                        \bottomrule
\end{tabular}%
}
\end{table*}

\begin{table*}[h!]
\centering
\caption{Similarity measures within gallery set (MeanSim@1 equivalent to $\Delta_{deg}(G)$) on MSVD.  }
\label{tab: deg: mvd}

\resizebox{\textwidth}{!}{%
\begin{tabular}{ll|ccc|ccc}
\toprule
\multirow{2}{*}{Method} & \multirow{2}{*}{Normalization}     & \multicolumn{3}{c|}{Text-to-Video Retrieval} & \multicolumn{3}{c}{Video-to-Text Retrieval} \\
                        & & MeanSim & MeanSim@1 & MeanSim@10 & MeanSim & MeanSim@1 & MeanSim@10  \\\midrule

CE+   &    &  0.1057 & 0.7050  & 0.5843 & 0.0155 & 0.0659 & 0.0641  \\

 \rowcolor{red!10} & +\ours             &   0.0464 & 0.4200   & 0.3249 & 0.0075 & 0.0548 & 0.0534  \\
\rowcolor{red!10} & +\fcut        &  0.0491 & 0.3884 & 0.3046 & 0.0064 & 0.0327 & 0.0310\\ 
                        \midrule

TT-CE+   &   & 0.1176 & 0.7156 & 0.5968 & 0.0193 & 0.0758 & 0.0739   \\

\rowcolor{red!10} & +\ours              &  0.0493 & 0.4194 & 0.3228 & 0.0152 & 0.0565 & 0.0547   \\
 \rowcolor{red!10} & +\fcut       &  0.0478 & 0.3661 & 0.2864 & 0.0093 & 0.0415 & 0.0398\\ 
                        \midrule

CLIP4Clip  &     &0.6790  & 0.8916 & 0.8518 & 0.6485 & 0.9736 & 0.9527   \\

\rowcolor{red!10} & +\ours              &  0.6918 & 0.8896 & 0.8479 & 0.6505 & 0.9737 & 0.9528   \\
\rowcolor{red!10} & +\fcut        & 0.6731 & 0.8838 & 0.8429 & 0.6481 & 0.9734 & 0.9522 \\ 
                        \midrule

CLIP2Video  &   & 0.5376 & 0.8179 & 0.7589 & 0.5769 & 0.8210  & 0.7702  \\

\rowcolor{red!10} & +\ours              &   0.5379 & 0.8179 & 0.7588 & 0.5798 & 0.8216 & 0.7703  \\
\rowcolor{red!10} & +\fcut        & 0.5314 & 0.8129 & 0.7528 & 0.5627 & 0.8077 & 0.7559 \\ 
                        \midrule

X-CLIP  &    & 0.6931 & 0.8961 & 0.8587 & 0.5796 & 0.9165 & 0.8945  \\

\rowcolor{red!10} & +\ours              &  0.7035 & 0.8904 & 0.8491 & 0.5879 & 0.9159 & 0.8942   \\
\rowcolor{red!10} & +\fcut        &0.6819 & 0.8885 & 0.8492 & 0.5600   & 0.9163 & 0.8913 \\ 
                        \bottomrule
\end{tabular}%
}
\end{table*}

\begin{table*}[h!]
\centering
\caption{Similarity measures within gallery set (MeanSim@1 equivalent to $\Delta_{deg}(G)$) on DiDeMo.  }
\label{tab: deg: didemo}
\resizebox{\textwidth}{!}{%
\begin{tabular}{ll|ccc|ccc}
\toprule
\multirow{2}{*}{Method} & \multirow{2}{*}{Normalization}     & \multicolumn{3}{c|}{Text-to-Video Retrieval} & \multicolumn{3}{c}{Video-to-Text Retrieval} \\
                        & & MeanSim & MeanSim@1 & MeanSim@10 & MeanSim & MeanSim@1 & MeanSim@10  \\\midrule

CE+   &    &  0.1375 & 0.6183 & 0.4835 & 0.0021 & 0.0176 & 0.0144 \\

 \rowcolor{red!10} & +\ours             &    0.1277 & 0.5686 & 0.4409 & 0.0019 & 0.0167 & 0.0136 \\
 \rowcolor{red!10} & +\fcut       & 0.1374 & 0.6001 & 0.4690  & 0.0019 & 0.0150  & 0.0123 \\ 
                        \midrule

TT-CE+   &  &  0.1609 & 0.6794 & 0.5580  & 0.0018 & 0.0304 & 0.0248   \\

\rowcolor{red!10} & +\ours              &  0.1432 & 0.5299 & 0.4292 & 0.0015 & 0.0244 & 0.0197   \\
\rowcolor{red!10} & +\fcut        &  0.0598 & 0.3680  & 0.2806 & 0.0018 & 0.0199 & 0.0161\\ 
                        \bottomrule
\end{tabular}%
}
\end{table*}


\begin{table*}[ht!]
\centering
\caption{Similarity measures between test gallery and query set on CLOTHO. The nearest neighbor is considered with respect to the gallery set.
}
\label{tab: cross: clotho appendix (full)}
\resizebox{\textwidth}{!}{%
\begin{tabular}{ll|ccc|ccc}
\toprule
\multirow{2}{*}{Method} & \multirow{2}{*}{Normalization}     & \multicolumn{3}{c|}{Text-to-Audio Retrieval} & \multicolumn{3}{c}{Audio-to-Text Retrieval} \\
                        & & MeanSim & MeanSim@1 & MeanSim@10 & MeanSim & MeanSim@1 & MeanSim@10  \\\midrule

AR-CE  &    & 0.0210  & 0.6028 & 0.5683 & 0.0210  & 0.5788 & 0.5212   \\

\rowcolor{red!10} & +\ours              & 0.0195 & 0.3931 & 0.3657 & 0.0185 & 0.3235 & 0.2830    \\
\rowcolor{red!10} & +\fcut        & 0.0146 & 0.4520  & 0.4231 & 0.0104 & 0.4044 & 0.3600  \\ 
                        \bottomrule
\end{tabular}%
}
\end{table*}

\begin{table*}[ht!]
\centering
\caption{
Similarity measures between test gallery and query set on AudioCaps. The nearest neighbor is considered with respect to the gallery set.
}
\label{tab: cross: audiocaps appendix (full)}
\resizebox{\textwidth}{!}{%
\begin{tabular}{ll|ccc|ccc}
\toprule
\multirow{2}{*}{Method} & \multirow{2}{*}{Normalization}     & \multicolumn{3}{c|}{Text-to-Audio Retrieval} & \multicolumn{3}{c}{Audio-to-Text Retrieval} \\
                        & & MeanSim & MeanSim@1 & MeanSim@10 & MeanSim & MeanSim@1 & MeanSim@10  \\\midrule

AR-CE  &   & 0.0036 & 0.6229 & 0.5729 & 0.0036  & 0.5927 & 0.4883    \\

\rowcolor{red!10} & +\ours              &  0.0036 & 0.6228 & 0.5728 & 0.0037  & 0.4454 & 0.3595   \\
 \rowcolor{red!10} & +\fcut       &  0.0005 & 0.4058 & 0.3672 & -0.0068 & 0.3568 & 0.2882\\ 
                        \bottomrule
\end{tabular}%
}
\end{table*}

\begin{table*}[ht!]
\centering
\caption{Similarity measures between test gallery and query set on MSCOCO (5k split). The nearest neighbor is considered with respect to the gallery set. }
\label{tab: cross: coco}
\resizebox{\textwidth}{!}{%
\begin{tabular}{ll|ccc|ccc}
\toprule
\multirow{2}{*}{Method} & \multirow{2}{*}{Normalization}     & \multicolumn{3}{c|}{Text-to-Image Retrieval} & \multicolumn{3}{c}{Image-to-Text Retrieval} \\
                        & & MeanSim & MeanSim@1 & MeanSim@10 & MeanSim & MeanSim@1 & MeanSim@10  \\\midrule

CLIP  &     & 0.1516 & 0.3282 & 0.3138 & 0.1516 & 0.3213 & 0.3009  \\

\rowcolor{red!10} & +\ours              &   0.1500   & 0.3245 & 0.3101 & 0.1511 & 0.3198 & 0.2994  \\
\rowcolor{red!10} & +\fcut        & 0.0635 & 0.2214 & 0.2035 & 0.0921 & 0.2414 & 0.2208 \\ 
                        \midrule
Oscar  &    &  -0.0188 & 0.5423 & 0.4790  & -0.0188 & 0.4893 & 0.3940  \\

\rowcolor{red!10} & +\ours              &  -0.0192 & 0.5379 & 0.4744 & -0.0193 & 0.4827 & 0.3873  \\
\rowcolor{red!10} & +\fcut        & -0.0257 & 0.4815 & 0.4172 & -0.0368 & 0.4290  & 0.3361 \\ 
                        \bottomrule
\end{tabular}%
}
\end{table*}

\begin{table*}[ht!]
\centering
\caption{Similarity measures between test gallery and query set on Flickr30k. The nearest neighbor is considered with respect to the gallery set.}
\label{tab: cross: f30k}
\resizebox{\textwidth}{!}{%
\begin{tabular}{ll|ccc|ccc}
\toprule
\multirow{2}{*}{Method} & \multirow{2}{*}{Normalization}     & \multicolumn{3}{c|}{Text-to-Image Retrieval} & \multicolumn{3}{c}{Image-to-Text Retrieval} \\
                        & & MeanSim & MeanSim@1 & MeanSim@10 & MeanSim & MeanSim@1 & MeanSim@10  \\\midrule

CLIP  &   &  0.1688 & 0.3167 & 0.2980  & 0.1688 & 0.3030 & 0.2777  \\

\rowcolor{red!10} & +\ours              & 0.1681 & 0.3156 & 0.2969 & 0.1683 & 0.3020 & 0.2769    \\
\rowcolor{red!10}  & +\fcut       &  0.0854 & 0.2539 & 0.2295 & 0.1150  & 0.2440 & 0.2190 \\ 
                        \midrule
Oscar  &   &  -0.0073 & 0.4349 & 0.3552 & -0.0073 & 0.3738 & 0.2695  \\

 \rowcolor{red!10} & +\ours             & -0.0074 & 0.4235 & 0.3440  & -0.0078 & 0.3441 & 0.2403    \\
\rowcolor{red!10} & +\fcut        & -0.0084 & 0.4259 & 0.3471 & -0.0102 & 0.3562 & 0.2554 \\ 
                        \bottomrule
\end{tabular}%
}
\end{table*}

\begin{table*}[ht!]
\centering
\caption{Similarity measures between test gallery and query set on MSR-VTT (full split and 1k split). The nearest neighbor is considered with respect to the gallery set.
}{}
\label{tab: cross: msrvtt}
\resizebox{\textwidth}{!}{%
\begin{tabular}{ll|ccc|ccc}
\toprule
\multirow{2}{*}{Method} & \multirow{2}{*}{Normalization}     & \multicolumn{3}{c|}{Text-to-Video Retrieval} & \multicolumn{3}{c}{Video-to-Text Retrieval} \\
                        & & MeanSim & MeanSim@1 & MeanSim@10 & MeanSim & MeanSim@1 & MeanSim@10  \\\midrule

\multicolumn{8}{c}{MSR-VTT (full split)}\\\midrule

CE+   &   &  -0.0026 & 0.4262 & 0.4021 & -0.0026 & 0.3830  & 0.3367  \\

 \rowcolor{red!10} & +\ours             &  -0.0024 & 0.3906 & 0.3680  & -0.0025 & 0.3474 & 0.3047   \\
\rowcolor{red!10} & +\fcut        & -0.0022 & 0.3610  & 0.3395 & -0.0142 & 0.2680  & 0.2316 \\ 
                        \midrule

TT-CE+   & &   0.0135 & 0.4632 & 0.4404 & 0.0135 & 0.4245 & 0.3774    \\

\rowcolor{red!10} & +\ours              & 0.0130  & 0.4234 & 0.4023 & 0.0123 & 0.3533 & 0.3128    \\
 \rowcolor{red!10} & +\fcut       &  0.0103 & 0.3779 & 0.3578 & 0.0028 & 0.3273 & 0.2880 \\ 
                        \midrule

\multicolumn{8}{c}{MSR-VTT (1k split)}\\\midrule

CLIP4Clip  & &    0.1074 & 0.2399 & 0.2160  & 0.1074 & 0.2388 & 0.2131   \\

\rowcolor{red!10} & +\ours              & 0.1078 & 0.2397 & 0.2159 & 0.1079 & 0.2394 & 0.2137     \\
\rowcolor{red!10} & +\fcut        &  0.1061 & 0.2394 & 0.2149 & 0.0906 & 0.2176 & 0.1923\\ 
                        \midrule

CLIP2Video  & &   0.1214 & 0.2689 & 0.2344 & 0.1214 & 0.2686 & 0.2325   \\

\rowcolor{red!10} & +\ours              & 0.1214 & 0.2688 & 0.2343 & 0.1217 & 0.2687 & 0.2327   \\
\rowcolor{red!10} & +\fcut        & 0.1153 & 0.2629 & 0.2278 & 0.0502 & 0.1825 & 0.1470 \\ 
                        \midrule

X-CLIP  &   & 0.1241 & 0.2484 & 0.2209 & 0.1241 & 0.2459 & 0.2158  \\

\rowcolor{red!10} & +\ours              &  0.1248 & 0.2479 & 0.2199 & 0.1241 & 0.2442 & 0.2142   \\
\rowcolor{red!10} & +\fcut        &  0.1197 & 0.2441 & 0.2154 & 0.0895 & 0.1955 & 0.1669\\ 
                        \bottomrule
\end{tabular}%
}
\end{table*}

\begin{table*}[ht!]
\centering
\caption{Similarity measures between test gallery and query set on ActivityNet. The nearest neighbor is considered with respect to the gallery set.}
\label{tab: cross: actnet}

\resizebox{\textwidth}{!}{%
\begin{tabular}{ll|ccc|ccc}
\toprule
\multirow{2}{*}{Method} & \multirow{2}{*}{Normalization}     & \multicolumn{3}{c|}{Text-to-Video Retrieval} & \multicolumn{3}{c}{Video-to-Text Retrieval} \\
                        & & MeanSim & MeanSim@1 & MeanSim@10 & MeanSim & MeanSim@1 & MeanSim@10  \\\midrule

CE+   &   &  0.0210  & 0.4209 & 0.3703 & 0.0210   & 0.4154 & 0.3679  \\

\rowcolor{red!10} & +\ours              &  0.0204 & 0.4066 & 0.3578 & 0.0208  & 0.4090  & 0.3623   \\
\rowcolor{red!10} & +\fcut        & 0.0149 & 0.3304 & 0.2893 & -0.0017 & 0.3049 & 0.2673 \\ 
                        \midrule

TT-CE+   & &  0.0391 & 0.5070  & 0.4611 & 0.0391 & 0.5019 & 0.4591    \\

\rowcolor{red!10} & +\ours              &  0.0383 & 0.4885 & 0.4440  & 0.0389 & 0.4973 & 0.4550   \\
\rowcolor{red!10} & +\fcut        &  0.0258 & 0.3796 & 0.3429 & 0.0148 & 0.3330  & 0.3017\\ 
                        \midrule

CLIP4Clip  & &  0.0533 & 0.3567 & 0.3223 & 0.0533 & 0.3567 & 0.3220     \\

\rowcolor{red!10} & +\ours              &  0.0533 & 0.3567 & 0.3222 & 0.0533 & 0.3553 & 0.3206   \\
\rowcolor{red!10} & +\fcut        & 0.0106 & 0.2821 & 0.2488 & 0.0107 & 0.2854 & 0.2512 \\ 
                        \midrule

X-CLIP  &  &  0.1233 & 0.3484 & 0.3135 & 0.1233 & 0.3464 & 0.3122   \\

\rowcolor{red!10} & +\ours              &  0.1233 & 0.3460  & 0.3111 & 0.1219 & 0.3397 & 0.3058   \\
\rowcolor{red!10} & +\fcut        & 0.1089 & 0.3018 & 0.2683 & 0.0968 & 0.3005 & 0.2672 \\ 
                        \bottomrule
\end{tabular}%
}
\end{table*}

\begin{table*}[h!]
\centering
\caption{Similarity measures between test gallery and query set on MSVD. The nearest neighbor is considered with respect to the gallery set.}
\label{tab: cross: mvd}

\resizebox{\textwidth}{!}{%
\begin{tabular}{ll|ccc|ccc}
\toprule
\multirow{2}{*}{Method} & \multirow{2}{*}{Normalization}     & \multicolumn{3}{c|}{Text-to-Video Retrieval} & \multicolumn{3}{c}{Video-to-Text Retrieval} \\
                        & & MeanSim & MeanSim@1 & MeanSim@10 & MeanSim & MeanSim@1 & MeanSim@10  \\\midrule

CE+   &   & 0.1029 & 0.6456 & 0.6188 & 0.1029 & 0.6037 & 0.5235    \\

\rowcolor{red!10} & +\ours              &  0.0159 & 0.4832 & 0.4624 & 0.0358 & 0.3979 & 0.3408   \\
\rowcolor{red!10} & +\fcut        & 0.0450  & 0.4501 & 0.4279 & 0.0384 & 0.3393 & 0.2870  \\ 
                        \midrule

TT-CE+   &  &   0.1170  & 0.6681 & 0.6417 & 0.1170  & 0.5997 & 0.531   \\

\rowcolor{red!10} & +\ours              & 0.0613 & 0.5144 & 0.4924 & 0.1019 & 0.5075 & 0.4456    \\
\rowcolor{red!10} & +\fcut        & 0.0421 & 0.4581 & 0.4369 & 0.0574 & 0.3807 & 0.3347 \\ 
                        \midrule

CLIP4Clip  & &   0.1273 & 0.2686 & 0.2606 & 0.1273 & 0.2483 & 0.2242    \\

\rowcolor{red!10} & +\ours              &  0.1283 & 0.2674 & 0.2593 & 0.1273 & 0.2473 & 0.2233    \\
\rowcolor{red!10} & +\fcut        & 0.0635 & 0.1982 & 0.1887 & 0.1167 & 0.2359 & 0.2119 \\ 
                        \midrule

CLIP2Video  &  &  0.1413 & 0.3032 & 0.2931 & 0.1413 & 0.2779 & 0.2500    \\
 
\rowcolor{red!10} & +\ours              & 0.1417 & 0.3031 & 0.2930  & 0.1416 & 0.2771 & 0.2492    \\
\rowcolor{red!10}  & +\fcut       & 0.1073 & 0.2635 & 0.2526 & 0.0770  & 0.1992 & 0.1723 \\ 
                        \midrule

X-CLIP  &  &0.1672 & 0.2943 & 0.2858 & 0.1672 & 0.2728 & 0.2487     \\

\rowcolor{red!10} & +\ours              & 0.1678 & 0.2940  & 0.2853 & 0.1675 & 0.2722 & 0.2482   \\
\rowcolor{red!10} & +\fcut        & 0.1621 & 0.2893 & 0.2805 & 0.1518 & 0.2564 & 0.2320  \\ 
                        \bottomrule
\end{tabular}%
}
\end{table*}

\begin{table*}[h!]
\centering
\caption{Similarity measures between test gallery and query set on DiDeMo. The nearest neighbor is considered with respect to the gallery set.}
\label{tab: cross: didemo}
\resizebox{\textwidth}{!}{%
\begin{tabular}{ll|ccc|ccc}
\toprule
\multirow{2}{*}{Method} & \multirow{2}{*}{Normalization}     & \multicolumn{3}{c|}{Text-to-Video Retrieval} & \multicolumn{3}{c}{Video-to-Text Retrieval} \\
                        & & MeanSim & MeanSim@1 & MeanSim@10 & MeanSim & MeanSim@1 & MeanSim@10  \\\midrule

CE+   &  &  -0.0068 & 0.2960  & 0.2379 & -0.0068 & 0.2888 & 0.2369   \\

\rowcolor{red!10} & +\ours              & -0.0066 & 0.2824 & 0.2265 & -0.0067 & 0.2825 & 0.2316    \\
\rowcolor{red!10} & +\fcut        & -0.0135 & 0.2715 & 0.2165 & -0.0148 & 0.2338 & 0.1897 \\ 
                        \midrule

TT-CE+   &  &   -0.002  & 0.4218 & 0.3485 & -0.002  & 0.4067 & 0.3439    \\

\rowcolor{red!10} & +\ours              &  -0.0007 & 0.3374 & 0.2755 & -0.0012 & 0.3747 & 0.3162   \\
\rowcolor{red!10} & +\fcut        & -0.0082 & 0.2832 & 0.2302 & -0.019  & 0.2396 & 0.1983 \\ 
                        \bottomrule
\end{tabular}%
}
\end{table*}

\subsection{Detail Setting of \pool} \label{sec:pool}
Following \Cref{eq:formaldef}, the adjacency matrices $\mathcal{S}^{g}_{\text{pool}}$ and $\mathcal{S}^{q}_{\text{pool}}$ of \pool\ can be presented as, 
\begin{align*} \label{eq:simpool1}
&\mathcal{S}^{g}_{\text{pool}}(i,j)= \begin{cases}
1\,\text{,} \operatorname{sim}(\mathbf{g_i}, \mathbf{\hat{g}_j}) \geq P_i(\hat{G},p)\\
0\,\text{,\,else}
\end{cases}\\
&\text{and also,}\\
&\mathcal{S}^{q}_{\text{pool}}(i,j)= \begin{cases}
1\,\text{,} \operatorname{sim}(\mathbf{g_i}, \mathbf{\hat{q}_j}) \geq P_i(\hat{Q},p)\\
0\,\text{,\,else}\,,
\end{cases}
\end{align*}
where $P_i(\hat{G},p)$ and $P_i(\hat{Q},p)$ are the value of $p$-percentage largest similarity between node $i$ and all the nodes in set $\hat{G}$ and $\hat{Q}$, respectively. We include four values of $p$, namely $[10, 25, 50, 100]$, to obtain a solid benchmark. All four benchmarks for the same task share identical hyperparameters $r_g$ and $r_q$. Note that the worst scenario of tuning for \pool\ is when there is no hyperparameter that can enable \pool\ to beat the baseline performance of the original retrieval model. But it doesn't happen during the experiment since we can always locate a pair of hyperparameters with a little effort with which \pool\ with at least one of four selected $p$ values can outperform the baseline on at least one of the evaluation metrics mentioned in \Cref{sec:data}. This indicates the effectiveness of the idea of inverse convolution.

\subsection{Computational Complexity} \label{sec:complexity}

In the section, we discuss the computation complexity of the proposed \ours\ and \fcut. As noted \Cref{sec:def}, $N_g$ is the size of gallery data, and $N_{\hat{Q}}$ and $N_{\hat{G}}$ are the size of the training (or validation) query and gallery set, respectively. We also assume $N_q$ is the size of the query.

Both methods need to precompute two adjacency matrices between the gallery set and both sets of training (or validation) data before performing the convolution, which costs $\mathcal{O}(N_g (N_{\hat{G}}+N_{\hat{Q}}))$ as matrix multiplication. Then, the matrix multiplication of inverse convolution operation also incurs a cost of  $\mathcal{O}(N_g (N_{\hat{G}}+N_{\hat{Q}}))$ for \ours\ and $\mathcal{O}(k^{-1} N_g (N_{\hat{G}}+N_{\hat{Q}}))$ for \fcut, for the pruned matrix \tail. Note that if we choose a small $k$ for \tail\, like $1\%$ in the empirical study, the convolution step can be much faster in practice (though the complexity doesn't change due to the pre-computation step). 

After performing the proposed methods, we need another $\mathcal{O}(N_q N_g)$ time to calculate the similarity between the query and gallery set when performing the inference, given that we don't consider any advanced trick to perform $\operatorname{argmax}$.

In sum, both \ours\ and \fcut\ incur computational cost of $\mathcal{O}(N_g (N_{\hat{G}}+N_{\hat{Q}}))$ for inverse graph convolution (before inference) and $\mathcal{O}(N_q N_g)$ for inference.

\subsection{More Quantitative Results} \label{sec:more exp}

The results on retrieval performance for all the datasets with various methods are presented in \Cref{tab: quan: clotho appendix (full),tab: quan: audiocaps appendix (full),tab: quan: coco,tab: quan: f30k,tab: quan: actnet,tab: quan: msrvtt,tab: quan: mvd,tab: quan: didemo}. The results indicate that both proposed methods achieve significantly better retrieval performance on all the tasks compared to the original baseline and \pool\ benchmarks.

\textbf{Text-Image and Image-Text Retrieval.} 
Results are presented in \Cref{tab: quan: coco,tab: quan: f30k}. 

On the MSCOCO (5k split) dataset, the \fcut\ method outperforms other methods in both text-to-image retrieval and image-to-text Retrieval for both CLIP and Oscar models. Specifically, for the CLIP model, \fcut\ achieves the best R@1 and R@5 in text-to-image retrieval and the best R@1, R@5, R@10, and MnR scores in image-to-text retrieval. Similarly, for the Oscar model, \fcut\ also delivers the best results, with the best R@1 and R@5 in text-to-image retrieval.

In the Flickr30k dataset, the \fcut\ method generally shows the highest performance in both text-to-image retrieval and image-to-text retrieval across both CLIP and Oscar models.

\textbf{Text-Video and Video-Text Retrieval.} 
Results are presented in {\Cref{tab: quan: actnet,tab: quan: msrvtt,tab: quan: mvd,tab: quan: didemo}}. 

From the table for MSR-VTT (full split), we can observe that \fcut\ demonstrated the best performance in both text-to-video and video-to-text retrieval on R@1, R@5, and R@10 with both methods. The results for MSR-VTT (1k split) show that \fcut\ generally achieves the best performance in text-to-video retrieval while \ours\ presents superior results in video-to-text retrieval.

For the ActivityNet dataset, \fcut exhibits superior performance in text-to-video and video-to-text retrieval on R@1, R@5, and R@10 with all four methods.

For the MSVD dataset, the best-performing method for text-to-video retrieval is \fcut\ since it has the best performance on R@1, R@5, and R@10 with all methods except X-CLIP. Meanwhile, \ours\ method has the best results for video-to-text retrieval with CLIP4CLIP, CLIP2Video, and X-CLIP.

On the DiDeMo dataset, \fcut\ with both methods achieve the best performance in the text-to-video and video-to-text retrieval on R@1 and R@5.

\textbf{Text-Audio and Audio-Text Retrieval.} 
Results are presented in \Cref{tab: quan: clotho appendix (full),tab: quan: audiocaps appendix (full)}. On the CLOTHO dataset, the AR-CE method enhanced with the \ours\ and \fcut\ techniques outperforms the earlier techniques, MoEE and MMT, and all other benchmarks. Specifically, either \ours\ or \fcut\ achieves the best results in R@1, R@5, and R@10 in both text-to-audio and audio-to-text retrieval tasks. 

In the case of the AudioCaps dataset, \fcut\ method performs best across multiple metrics in both text-to-audio and audio-to-text retrieval tasks.

\subsection{More Ablation Study} \label{sec:more_ablt}

\textbf{Continuation on RQ1: Is the data degeneration problem alleviated?}

To strengthen the conclusion that both \ours\ and \fcut have strong capability to alleviate \problem, we conduct a comprehensive experiment on all the datasets and methods as an extension to \Cref{tab: sim measure intra} and \Cref{tab: sim measure inter}. We have \Cref{tab: deg: clotho appendix (full),tab: deg: audiocaps appendix (full),tab: deg: coco,tab: deg: f30k,tab: deg: msrvtt,tab: deg: actnet,tab: deg: mvd,tab: deg: didemo} that presents the mean similarity within the test gallery set of both tasks for three scenarios, the overall mean (MeanSim), the mean between the nearest neighbor(MeanSim@1), and the mean between nearest 10 neighbors (MeanSim@10). And we have \Cref{tab: cross: clotho appendix (full),tab: cross: audiocaps appendix (full),tab: cross: coco,tab: cross: f30k,tab: cross: msrvtt,tab: cross: actnet,tab: cross: mvd,tab: cross: didemo} include the similarity score from the test gallery set to the test query set with the same evaluation metrics.

It is quite obvious that both \ours\ and \fcut\ help decrease the similarity score on all metrics, especially on MeanSim@1(\ie, $\Delta_{deg}(G)$). Also, both methods have better performance on the retrieval task from text to other modalities, which is the more important task in practice compared to the other direction.

Based on the results, we can safely draw the conclusion that both proposed methods are able to significantly alleviate the \problem\ by significantly reducing $\Delta_{deg}(G)$ score.

\textbf{Continuation on RQ2: Is \fcut\ sensitive to the hyperparameter $r_g$ and $r_q$?}

To address the question, like analysis on \ours, we assess the R@1 metrics of our proposed method under a wide range of hyperparameters compared to the optimal choice adopted. For each task, one of $r_g$ or $r_q$ is fixed and the other is tuned to observe its impact on the R@1 metrics. 
The results obtained from the MSCOCO dataset are depicted in \Cref{fig: hyper sen cut}. Despite some observed variations, our method consistently outperforms the baseline, represented by the red dashed line. This suggests that the proposed method can continuously enhance performance even when parameters are varied over a broad range.

\begin{figure*}[t!]
    \centering
    \subcaptionbox{Image-to-Text R@1 w.r.t $r_g$}{\includegraphics[width=0.3\textwidth]{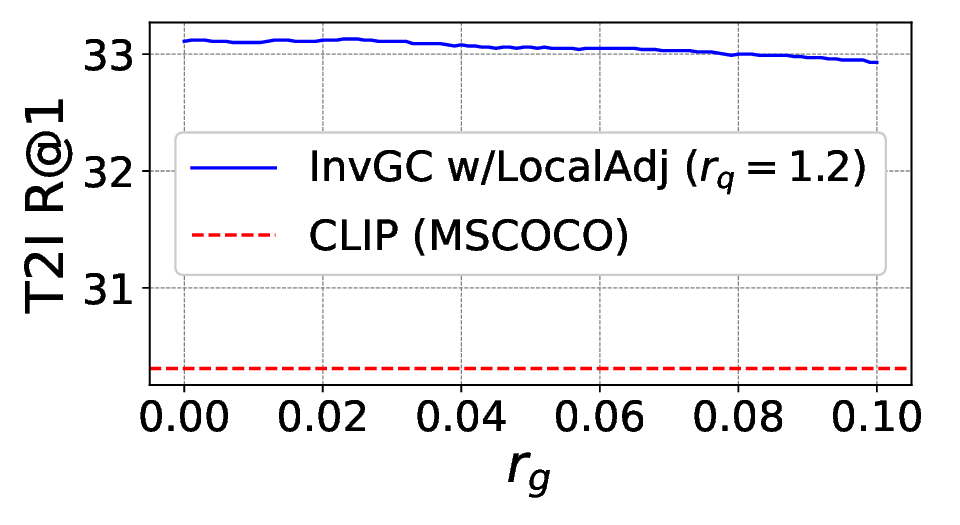}}
    \subcaptionbox{Image-to-Text R@1 w.r.t $r_q$}{\includegraphics[width=0.3\textwidth]{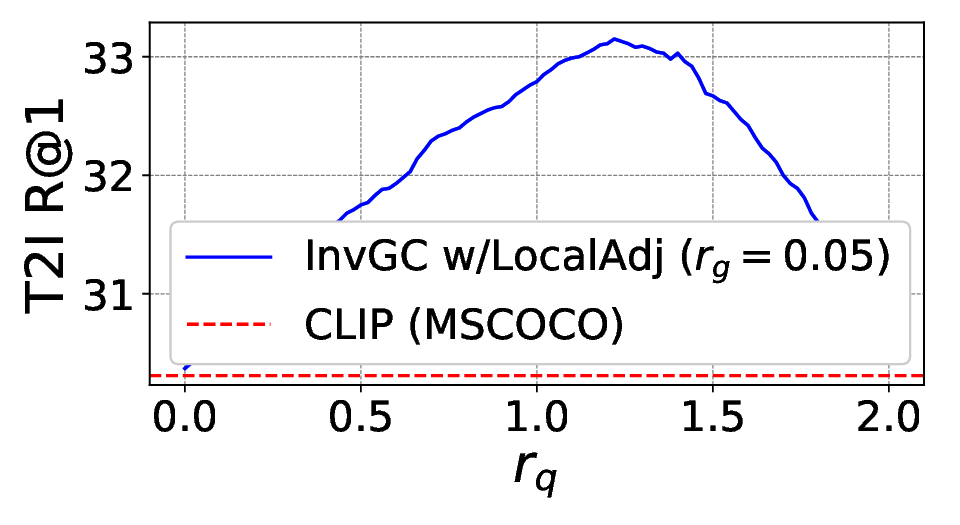}}
    
    \subcaptionbox{Text-to-Image R@1 w.r.t $r_g$}{\includegraphics[width=0.3\textwidth]{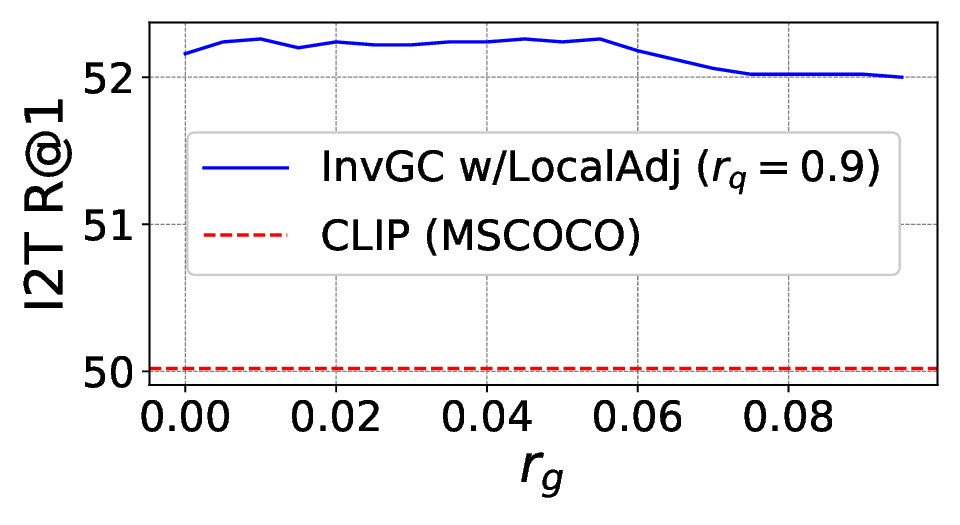}}
    \subcaptionbox{Text-to-Image R@1 w.r.t $r_q$}{\includegraphics[width=0.3\textwidth]{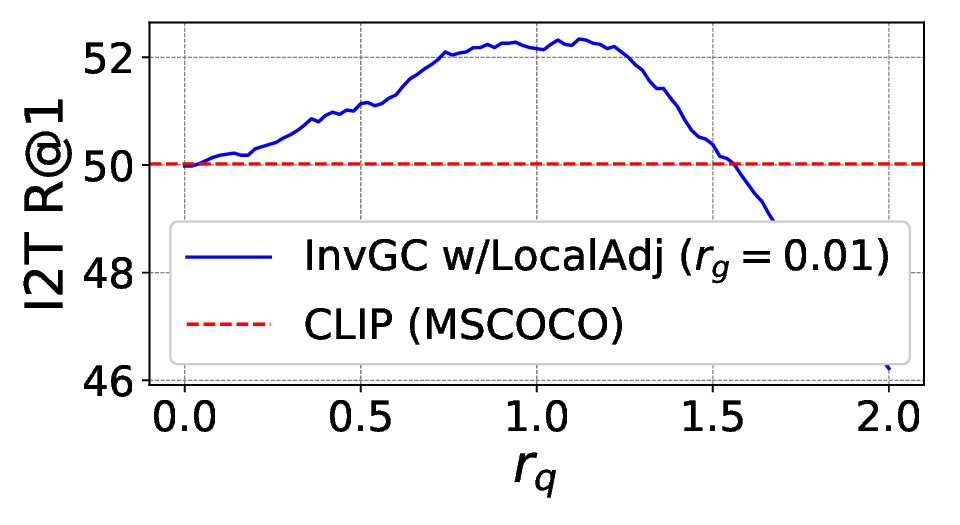}}
    \caption{Hyperpapameters sensitivity of \fcut\ on MSCOCO.
    }
    \label{fig: hyper sen cut}
\end{figure*}

\section{Proofs}
\subsection{Proof of Theorem \ref{the: 1}}\label{A:1}

 To prove the theorem, we start with some preliminary definitions. We denote the $n$-dimensional unit sphere as,
$$
 \begin{aligned}
      \mathcal{S}_{n-1}:=\left\{\mathbf{x} \in \mathbb{R}^n:\|\mathbf{x}\|_2=1\right\}
 \end{aligned}
$$

and $n$-dimensional ball, 

$$
 \begin{aligned}
    \mathcal{B}_{n, b}:=\left\{\mathbf{x} \in \mathbb{R}^n:||x||_2 \leq a \right\}
 \end{aligned}
$$

Then, $n$-dimensional ball representing the neighborhood of $\mathbf{x}$ with radius $a$ can be denoted as $\mathcal{B}_{n, \mathbf{x}, a}:=\left\{\mathbf{y} \in \mathbb{R}^n:||y -x||_2 \leq a \right\}$. We go on to denote spherical caps $\mathbf{C}_{n, \mathbf{x}, a}$ as
$$
\begin{aligned}
\mathbf{C}_{n, \mathbf{x}, b} & :=\mathcal{S}_{n-1} \cap \mathcal{B}_{n, \mathbf{x}, b}
\end{aligned}
$$
Note that the radius $a$ here is not the neighborhood coverage $b$, where they have a relation as $a^2 = 2\sqrt{a^2 -b^2}$ by simple triangle calculation. Since the basis of $\mathbf{C}_{R,\mathbf{x},a}$ is an $n-1$ dimensional hypersphere of radius $b$ and $b$ is the key factor we are interested in, we denote the spherical caps using $b$ as,

$$
\begin{aligned}
    \operatorname{Cap}_{n, \mathbf{x}, b} & := \mathbf{C}_{n, \mathbf{x}, a}
\end{aligned}
$$
Since the cosine similarity applied is independent of the norm of each point in $G$ and $Q$, we can assume without loss of generality that all the points in $G$ and $Q$ are on the unit sphere, i.e. $\forall \mathbf{x} \in G \cup Q, x \in \mathcal{S}_{n-1}$. Since query data follows an independent and identical uniform distribution in $\mathbb{R}^n$, therefore it still follows an independent and identical uniform distribution on $\mathcal{S}_{n-1}$.

Note that the probability of $\mathbf{x_1}$ is correctly retrieved is lower bounded by the probability that the corresponding query point $\mathbf{q} \in \operatorname{Cap}_{n, \mathbf{x_1}, b}$, for $q$ must be the nearest neighbor of $\mathbf{x_1}$.

Given uniform distribution, the probability is $\frac{\operatorname{A}(\operatorname{Cap}_{n, \mathbf{x_1}, b})}{\operatorname{A}(\mathcal{S}_{n-1})}$, where the $\operatorname{A}$ is the operator for area of hypersurface. 

Generally, bounding volume is easier than the surface area, so we have the following Lemma \Cref{the: 2} to help establish the relationship between the two.

\begin{lemma}\label{the: 2}
Therefore, the relationship between the surface area $\operatorname{A}(\operatorname{Cap}_{n,\mathbf{x},b})$ and volumn $\operatorname{V}(\operatorname{Cap}_{n,\mathbf{x},b})$\footnote{Every time we mention the hypervolume of a sphere $\operatorname{V}(\operatorname{Cap}_{n,\mathbf{x},b})$ where $\operatorname{Cap}_{n,\mathbf{x},b} = \mathbf{C}_{n,\mathbf{x},a} = \mathcal{S}_{n-1} \cap \mathcal{B}_{n, \mathbf{x}, a}$, we actually always refer to that of $\operatorname{V}(\mathcal{B}_{n, 1} \cap \mathcal{B}_{n, \mathbf{x}, a})$. We keep using expressions like $\operatorname{V}(\operatorname{Cap}_{n,\mathbf{x},b})$ to prevent possible confusion when dealing with the area and volume at the same time} of $n$-dimensional sphere cap is,

\begin{equation}
   \operatorname{A}(\operatorname{Cap}_{n,\mathbf{x},b})=n \operatorname{V}(\operatorname{Cap}_{n,\mathbf{x},b})
\end{equation}
\end{lemma}

\begin{proof}
    We have the relationship between the surface area $A(\mathcal{S}_{n-1})$ and volume $V(\mathcal{S}_{n-1})$ \footnote{When mentioning the hypervolume of a sphere $\mathcal{S}_{n-1}$, we actually always refer to that of $\mathcal{B}_{n,1}$.} of $n$-dimensional unit hypersphere as,

\begin{equation}
    \operatorname{A}(S_{n-1})=\frac{d}{d r} \operatorname{V}(S_{n-1})=n \operatorname{V}(S_{n-1})
\end{equation}

Therefore, the same relationship can be adapted to sphere cap as well.
\end{proof}

With \Cref{the: 2}, we only need to bound the volume $\operatorname{V}(\operatorname{Cap}_{n,\mathbf{x},b})$, which is given by the following \Cref{the: 3}.

\begin{theorem}\label{the: 3}
Given $n$-dimensional unit sphere $\mathcal{S}_{n-1}$ and spherical caps $\operatorname{Cap}_{n, \mathbf{x}, b}$ on it, we have,

$$
\frac{1}{2 }\cdot b^{n}  > \frac{\operatorname{V}\left(\operatorname{Cap}_{n, \mathbf{x}, b}\right)}{\operatorname{V}\left(\mathcal{S}_{n-1}\right)}> \frac{1}{4n }\cdot b^{n+1}  .
$$

\end{theorem}

\begin{proof}
\textbf{Lower bound}: Follow the proof of Lemma 4.1 in \cite{cap_vol}. The basis of $\operatorname{Cap}_{R,\mathbf{x},b}$ is an $n-1$ dimensional hypersphere of radius $r=b$, denoted as $\mathcal{B}_{n-1,b}$. Therefore $\mathrm{Cap}_{n, \mathbf{x}, b}$ includes a cone $C_1$ with basis as $\mathcal{B}_{n-1,b}$ and height $h=1-\sqrt{1-b^2}$. Then, a cylinder $C_2$ with the same basis but height $2 \cdot b$ includes $\mathcal{B}_{n,b}$. Note that we have $h>b^2/2$ by Taylor expansion. Based on all the facts, we have:
$$
\begin{aligned}
\operatorname{V}\left(\operatorname{Cap}_{n, \mathbf{x}, b}\right)&>\operatorname{V}\left(C_1\right)=\operatorname{V}\left(\mathcal{B}_{n-1,b}\right) \frac{h}{n}\\
&= \operatorname{V}\left(C_2\right) \frac{h}{2b n}\\
& >\operatorname{V}\left(\mathcal{B}_{n,b}\right) \frac{h}{2 b n}\\
& >\operatorname{V}\left(\mathcal{B}_{n,b}\right) \frac{b}{4 n}
\end{aligned}
$$
Therefore,

$$
\begin{aligned}
     \frac{b}{4 n} \cdot \frac{\operatorname{V}\left(\operatorname{Cap}_{n, \mathbf{x}, b}\right)}{\operatorname{V}\left(\mathcal{S}_{n-1}\right)} &>  \frac{b}{4 n} \cdot \frac{\operatorname{V}\left(\mathcal{B}_{n,b}\right)}{\operatorname{V}\left(\mathcal{B}_{n,1}\right)}\\
    &=  \frac{b}{4 n} \cdot \left(\frac{b}{1}\right)^n \\
    &=  \frac{1}{4 n} \cdot b^{n+1} 
\end{aligned}
$$
\textbf{Upper bound}: Based on proof of Lower bound, we notice that half of $\mathcal{B}_{n,b}$ includes $\operatorname{Cap}_{n, \mathbf{x}, b}$, we have

$$
\begin{aligned}
\operatorname{V}\left(\operatorname{Cap}_{n, \mathbf{x}, b}\right)& <\frac{\operatorname{V}\left(\mathcal{B}_{n,b}\right)}{2}
\end{aligned}
$$
Therefore,

$$
\begin{aligned}
    \frac{\operatorname{V}\left(\operatorname{Cap}_{n, \mathbf{x}, b}\right)}{\operatorname{V}\left(\mathcal{S}_{n-1}\right)} &< \frac{1}{2} \cdot \frac{\operatorname{V}\left(\mathcal{B}_{n,b}\right)}{\operatorname{V}\left(\mathcal{B}_{n,1}\right)}\\
    &= \frac{1}{2} \cdot b^{n} 
\end{aligned}
$$
\end{proof}

Finally, we finish the proof using \Cref{the: 2}.
\subsection{Proof of Corollary \ref{the: 4}}

\begin{proof}

Using the result from \Cref{the: 1}, we have the following inequalities for the probability of successful retrieval of $b_1$ and $b_2$,

$$
\begin{aligned}
 \operatorname{P}(\mathbf{x},b_1) < \frac{n}{2 }\cdot b_1^{n},
\end{aligned}
$$
And,

$$
\begin{aligned}
\operatorname{P}(\mathbf{x},b_2) >\frac{1}{4 }\cdot b_2^{n+1}\,,
\end{aligned}
$$

Therefore, we have,

$$
\begin{aligned}
\frac{\operatorname{P}(\mathbf{x_1},b_1)}{\operatorname{P}(\mathbf{x_1},b_2)} < \frac{2n}{b_2} \cdot \left(\frac{b_1}{b_2}\right)^{n} \,.
\end{aligned}
$$

Note that $2/b_2$ is constant with respect to $n$, so we finish the proof.
\end{proof}

\end{document}